\documentclass[twoside]{article}

\usepackage[accepted]{aistats2020}
\usepackage[round]{natbib}
\renewcommand{\bibname}{References}
\renewcommand{\bibsection}{\subsubsection*{\bibname}}

\usepackage[utf8]{inputenc} 
\usepackage[T1]{fontenc}    
\usepackage{color}
\definecolor{mydarkblue}{rgb}{0,0.08,0.45}
\usepackage[colorlinks=true,
    linkcolor=mydarkblue,
    citecolor=mydarkblue,
    filecolor=mydarkblue,
    urlcolor=mydarkblue]{hyperref}

\usepackage{amsmath, amsfonts, amsthm, amssymb}
\usepackage{bbm}
\usepackage{graphicx}
\usepackage{enumitem}
\usepackage[title,page]{appendix}
\usepackage{hyperref}
\usepackage{caption}
\usepackage{subcaption}
\usepackage{float}
\usepackage{comment}
\usepackage[capitalize,nameinlink]{cleveref}
\usepackage[utf8]{inputenc}
\usepackage[T1]{fontenc}

\newcommand{\Prob}{\mathbb{P}}
\newcommand{\E}{\mathbb{E}}
\newcommand{\reals}{\mathbb{R}}
\newcommand{\D}{\mathcal{D}}    
\newcommand{\A}{\mathcal{A}}    
\newcommand{\cR}{\mathcal{R}}
\newcommand{\R}{\mathbb{R}}

\newcommand{\Var}{\text{Var}}

\newcommand{\Tr}{\text{Tr}}

\newcommand{\ebias}{\mathcal{E}_\text{bias}}
\newcommand{\evar}{\mathcal{E}_\text{variance}}
\newcommand{\enoise}{\mathcal{E}_\text{noise}}

\DeclareMathOperator{\rank}{rank}
\DeclareMathOperator{\rowspace}{rowspace}
\DeclareMathOperator{\nullspace}{nullspace}

\newcommand{\var}{\mathrm{Var}}

\newtheorem{theorem}{Theorem}
\newtheorem{assumption}{Assumption}
\newtheorem{lemma}{Lemma}
\newtheorem{coro}{Corollary}
\newtheorem{prop}{Proposition}
\newtheorem{definition}{Definition}

\newcommand{\beq}{\begin{equation}}
\newcommand{\eeq}{\end{equation}}
\newcommand{\be}{\begin{equation}}
\newcommand{\ee}{\end{equation}}
\newcommand{\beqa}{\begin{eqnarray}}
\newcommand{\eeqa}{\end{eqnarray}}
\newcommand{\bean}{\begin{eqnarray*}}
\newcommand{\eean}{\end{eqnarray*}}
\crefname{appsec}{Appendix}{Appendices}

\begin{document}

%
\runningtitle{A Modern Take on the Bias-Variance Tradeoff}

%
\runningauthor{Neal, Mittal, Baratin, Tantia, Scicluna, Lacoste-Julien, Mitliagkas}

\twocolumn[

\aistatstitle{A Modern Take on the Bias-Variance Tradeoff \\ in Neural Networks}


\aistatsauthor{ Brady Neal \quad Sarthak Mittal \quad  Aristide Baratin \quad Vinayak Tantia \quad Matthew Scicluna\\ \textbf{Simon Lacoste-Julien$^{\dag,\ddagger}$ \quad Ioannis Mitliagkas$^{\dag}$} }

\aistatsaddress{ Mila, Université de Montréal\\ 
   $^\dag$Canada CIFAR AI Chair \qquad $^\ddag$CIFAR Fellow} ]

\begin{abstract}
The bias-variance tradeoff tells us that as model complexity increases, bias falls and variances increases, leading to a U-shaped test error curve. However, recent empirical results with over-parameterized neural networks are marked by a striking absence of the classic U-shaped test error curve: test error keeps decreasing in wider networks. This suggests that there might not be a bias-variance tradeoff in neural networks with respect to network width, unlike was originally claimed by, e.g., \citet{geman}. Motivated by the shaky evidence used to support this claim in neural networks, we measure bias and variance in the modern setting. We find that \emph{both} bias \emph{and} variance can decrease as the number of parameters grows. To better understand this, we introduce a new decomposition of the variance to disentangle the effects of optimization and data sampling. We also provide theoretical analysis in a simplified setting that is consistent with our empirical findings.


\end{abstract}

\section{INTRODUCTION}

There is a dominant dogma in machine learning:
\begin{quote}
    ``The price to pay for achieving low bias is high variance'' \citep{geman}.
\end{quote}
The quantities of interest here are the bias and variance of a learned model's {\em prediction} on a new input, where the randomness comes from the sampling of the training data.
This idea that bias decreases while variance increases with model capacity, leading to a U-shaped test error curve is commonly known as the \emph{bias-variance tradeoff} (\cref{fig:main_common_intuition_wrong} (left)).

There exist experimental evidence and theory that support the idea of a tradeoff. In their landmark paper, \citet{geman} measure bias and variance in various models.
They show convincing experimental evidence for the bias-variance tradeoff in nonparametric methods such as kNN (k-nearest neighbor) and kernel regression. They also show experiments on neural networks and claim that bias decreases and variance increases with network width. Statistical learning theory \citep{vapnik1998statistical} successfully predicts these U-shaped test error curves implied by a tradeoff for a number of classic machine learning models.
A key element is identifying a notion of model capacity, understood as the main parameter controlling this tradeoff.



Surprisingly, there is a growing amount of empirical evidence that \textit{wider} networks generalize \textit{better} than their smaller counterparts \citep{DBLP:journals/corr/NeyshaburTS14,wide_resnet,novak2018sensitivity, lee2018deep,belkin2018, jamming,fisher-rao_metric,DBLP:journals/corr/CanzianiPC16}.
 In those cases the classic U-shaped test error curve is not observed. 

A number of different research directions have spawned in response to these findings. \citet{DBLP:journals/corr/NeyshaburTS14} hypothesize the existence of an implicit regularization mechanism.
Some study the role that optimization plays \citep{implicit_bias__gd_linear_sep, implicit_bias_opt_geo}. Others suggest new measures of capacity \citep{fisher-rao_metric,neyshabur2018the}.
All approaches focus on test error, rather than studying bias and variance directly \citep{neyshabur2018the,scaling,fisher-rao_metric,belkin2018}.

\begin{figure*}[t]
    \centering
    \begin{subfigure}[t]{0.48\textwidth}
        \centering
        \includegraphics[width=\textwidth]{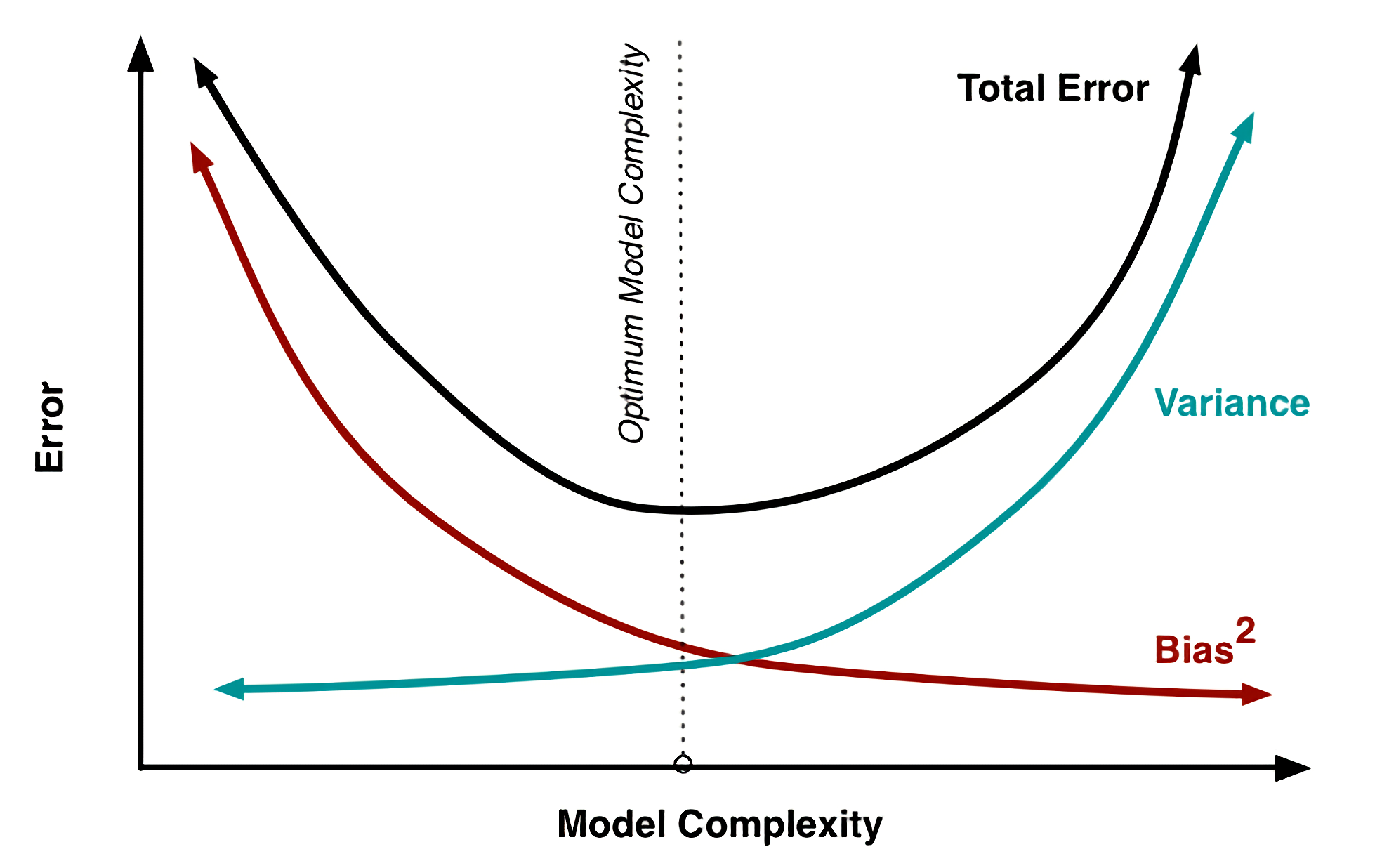}
    \end{subfigure}
    \hfill
    \begin{subfigure}[t]{0.48\textwidth}
        \centering
         \includegraphics[width=\textwidth]{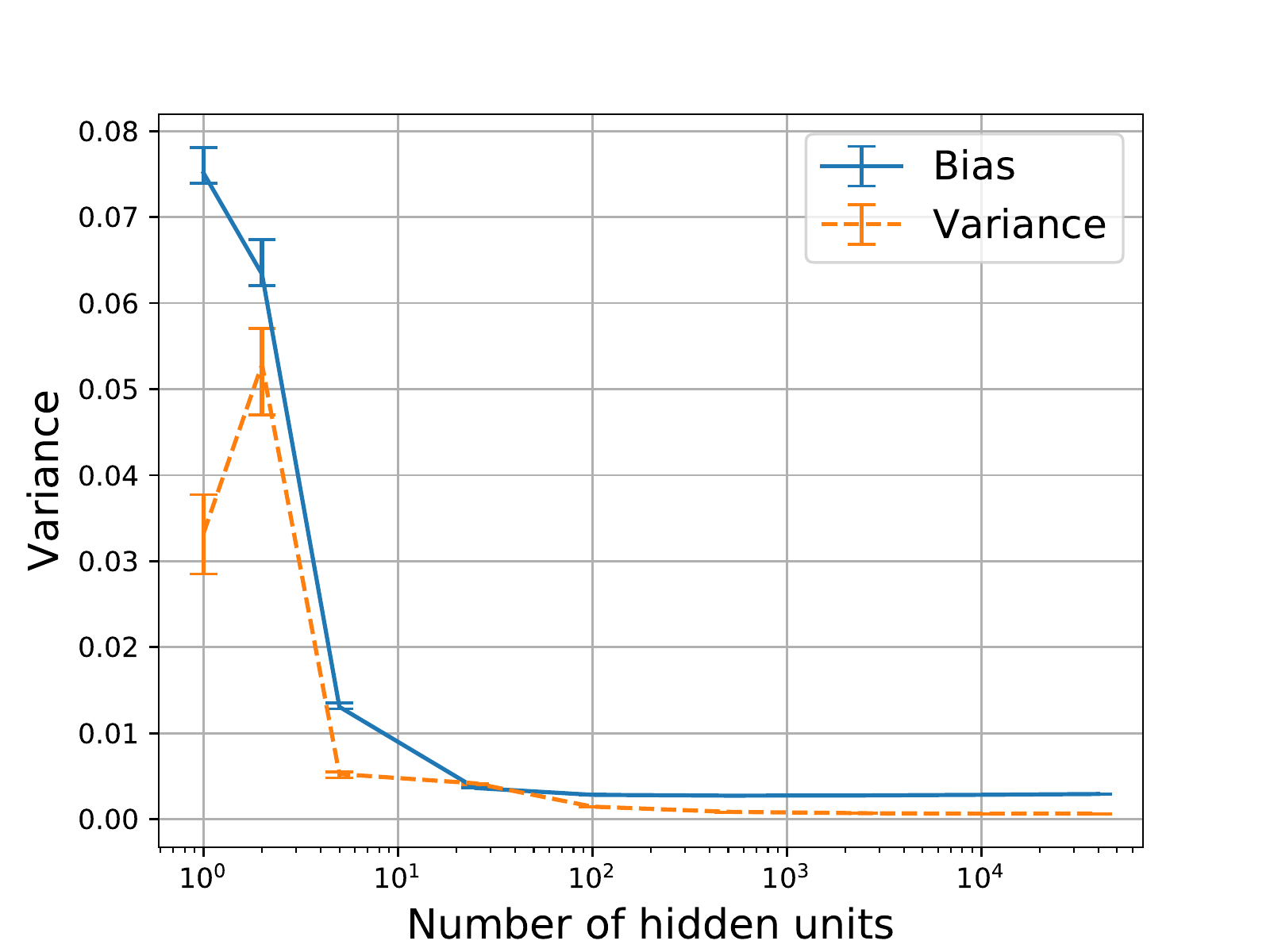}
    \end{subfigure}
    \caption{On the left is an illustration of the common intuition for the bias-variance tradeoff \citep{fortmann-roe_2012}. We find that {\em both} bias and variance decrease when we increase network width on MNIST (right) and other datasets (\cref{sec:width}). These results seem to contradict the traditional intuition of a strict tradeoff.}
    \label{fig:main_common_intuition_wrong}
\end{figure*}

Test error analysis does not give a definitive answer on the lack of a bias-variance tradeoff. 
Consider boosting:
it is known that its test error often decreases with the number of rounds \cite[Figures 8-10]{Schapire1999}.
In spite of this monotonicity in test error, \citet{buhlmann2003boosting}
show that variance grows at an exponentially decaying rate, calling this an ``exponential bias-variance tradeoff.''
To study the bias-variance tradeoff, one has to isolate and measure bias and variance individually.
To the best of our knowledge, there has not been published work reporting such measurements on neural networks since 
\citet{geman}.




We go back to basics and study bias and variance.
We start by taking a closer look at \citet[Figure 16 and Figure 8 (top)]{geman}'s experiments with neural networks.
We notice that their experiments do not support their claim that ``bias falls and variance increases with the number of hidden units.''
The authors attribute this inconsistency to convergence issues and maintain their claim that the bias-variance tradeoff is universal.
Motivated by this inconsistency, we perform a set of bias-variance experiments with modern neural networks.

We measure prediction bias and variance of fully connected neural networks.
These measurements allow us to reason directly about whether there exists a tradeoff with respect to network width.
We find evidence that \emph{both} bias \emph{and} variance can decrease at the same time as network width increases in common classification and regression settings (\cref{fig:main_common_intuition_wrong,sec:width}).

We observe the qualitative lack of a bias-variance tradeoff in network width with a number of gradient-based optimizers. 
In order to take a closer look at the roles of optimization and data sampling, we propose a simple decomposition of total prediction variance (\cref{sec:decomposition}). 
We use the law of total variance to get a term that corresponds to average (over data samplings) variance due to optimization and a term that corresponds to variance due to training set sampling of an ensemble of differently initialized networks. 
Variance due to optimization is significant in the under-parameterized regime and monotonically decreases with width in the over-parameterized regime. 
There, total variance is much lower and dominated by variance due to sampling (\cref{fig:all_variances}).

We provide theoretical analysis, consistent with our empirical findings,
in simplified analysis settings:
i) prediction variance does not grow arbitrarily with number of parameters in fixed-design linear models;
ii) variance due to optimization diminishes with number of parameters in neural networks under strong assumptions.

\begin{figure*}[t]
    \centering
    \begin{subfigure}[t]{0.47\textwidth}
        \centering
         \includegraphics[width=\textwidth]{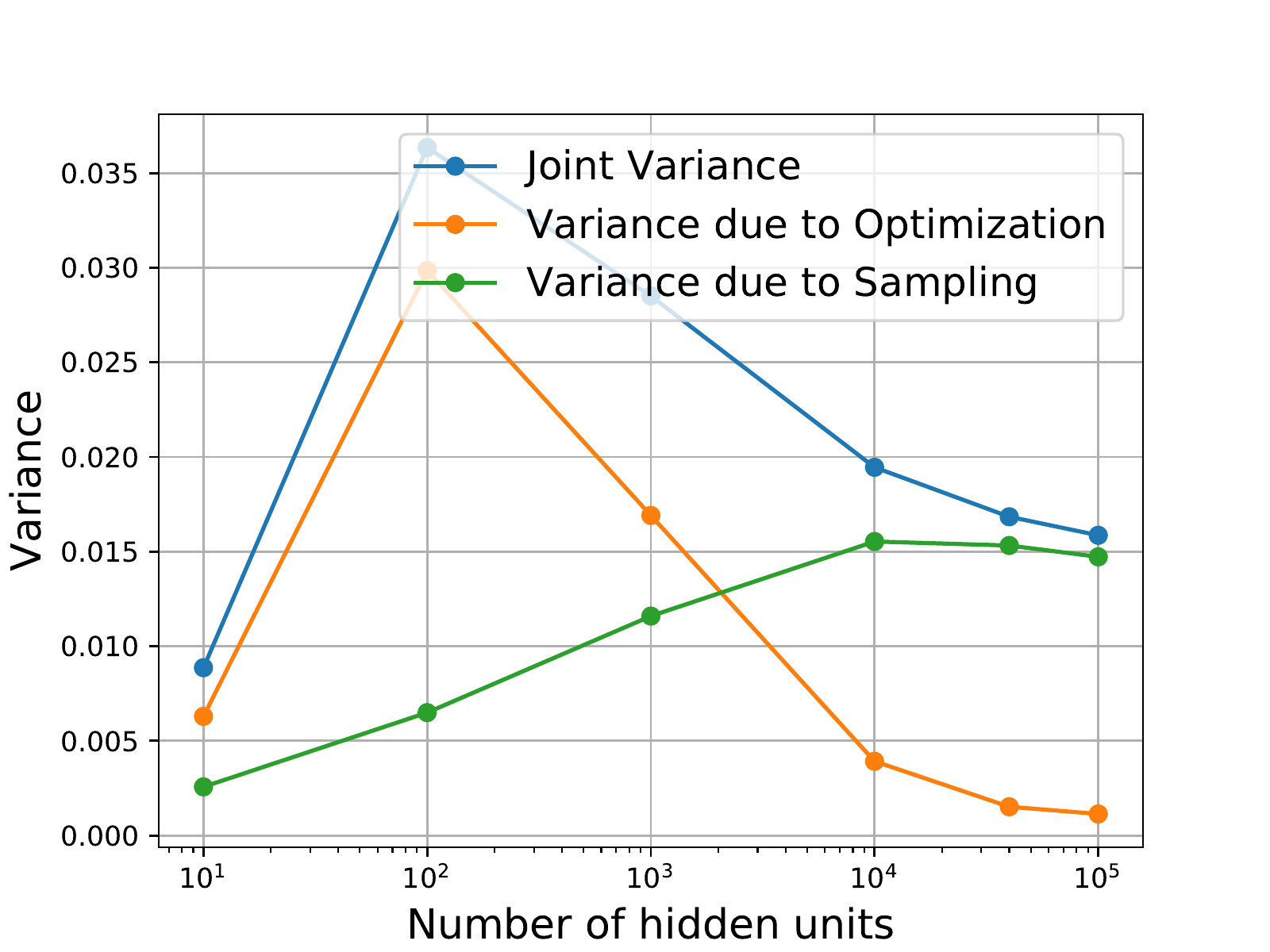}
    \end{subfigure}
    \hspace{.5cm}
    \begin{subfigure}[t]{0.47\textwidth}
        \centering
     \includegraphics[width=\textwidth]{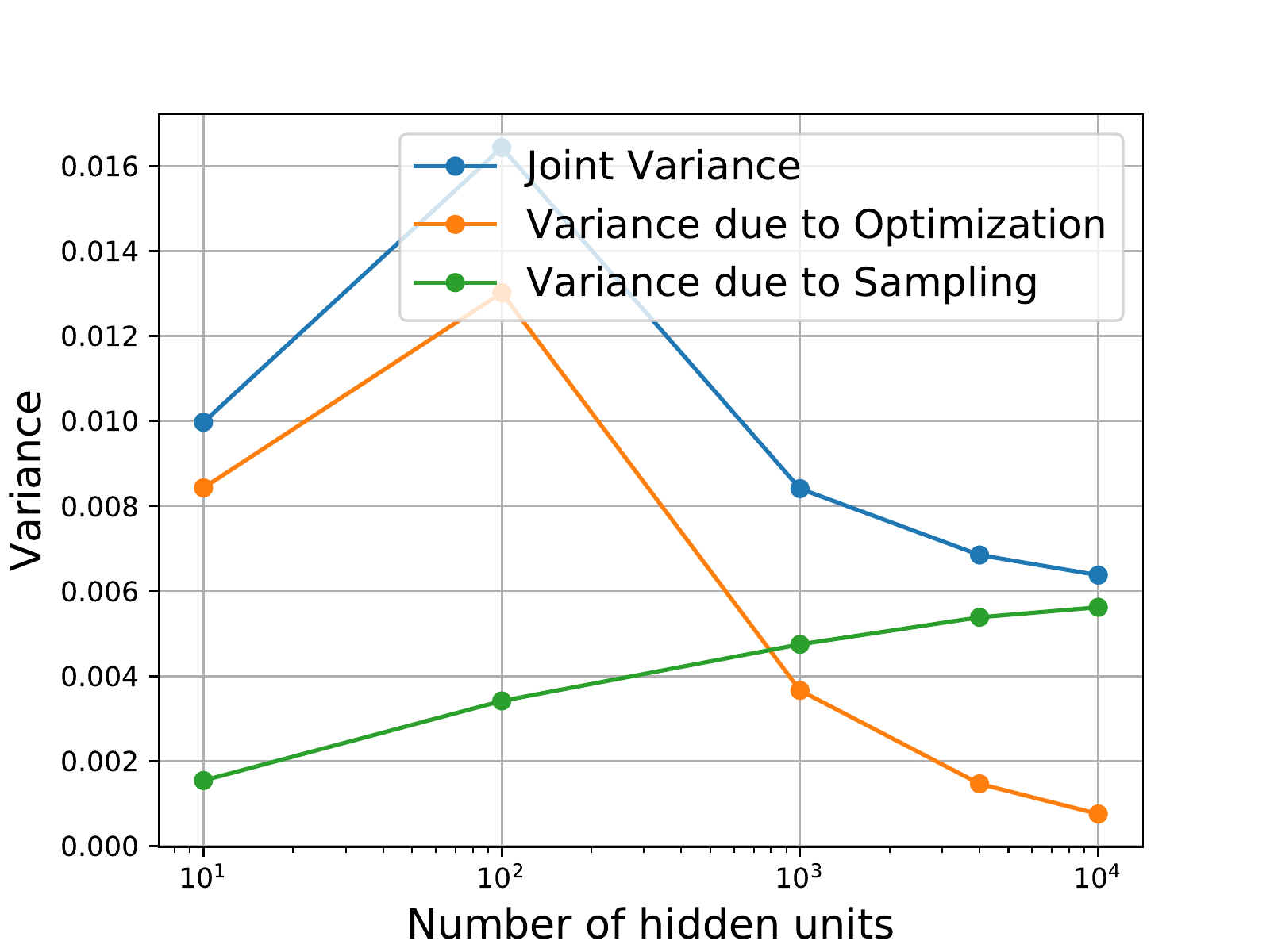}
    \end{subfigure}
    \caption{Trends of variance due to sampling and variance due to optimization with width on CIFAR10 (left) and on SVHN (right).
    Variance due to optimization decreases with width, once in the over-parameterized setting.
    Variance due to sampling plateaus and remains constant. This is in contrast with what the bias-variance tradeoff would suggest.}
    \label{fig:all_variances}
\end{figure*}





    


\paragraph{Organization}
The rest of this paper is organized as follows. We discuss relevant related work in \cref{sec:related-work}. \Cref{sec:preliminaries} establishes necessary preliminaries, including our variance decomposition. In \cref{sec:width}, we empirically study the impact of network width on variance. In \cref{sec:theory}, we present theoretical analysis in support of our findings.

\section{RELATED WORK}
\label{sec:related-work}

\citet{DBLP:journals/corr/NeyshaburTS14,neyshaburthesis} point out that because increasing network width does not lead to a U-shaped test error curve, there must be some form of implicit regularization controlling capacity. Our work is consistent with this finding, but by approaching the problem from the bias-variance perspective, we gain additional insights: 1) We specifically address the hypothesis that decreased bias must come at the expense of increased variance (see \citet{geman} and \cref{app:intuitions}) by measuring both quantities.
2) Our more fine-grain approach reveals that variance due to optimization vanishes with width, while variances due to sampling increases and levels off.
This insight about variance due to sampling is consistent with existing variance results for boosting \citep{buhlmann2003boosting}.
To ensure that we are studying networks of increasing capacity, one of the experimental controls we use throughout the paper is to verify that bias is decreasing.

In independent concurrent work, \citet{jamming, belkin2018} point out that generalization error acts according to conventional wisdom in the under-parameterized setting, that it decreases with capacity in the over-parameterized setting, and that there is a sharp transition between the two settings. Although the phrase ``bias-variance trade-off'' appears in \citet{belkin2018}'s title, their work really focuses on the shape of the test error curve: they argue it is not the simple U-shaped curve that conventional wisdom would suggest, and it is not the decreasing curve that \citet{DBLP:journals/corr/NeyshaburTS14} found; it is ``double descent curve,'' which is essentially a concatenation of the two curves. This is in contrast to our work, where we actually measure bias, variance, and components of variance in this over-parameterized regime. Interestingly, \citet{belkin2018}'s empirical study of test error provides some evidence that our bias-variance finding might not be unique to neural networks and might be found in other models such as decision trees.

In subsequent work,\footnote{By ``subsequent work,'' we mean work that appeared on arXiv five months after our paper appeared on arXiv.} \citet{belkin2019models,hastie2019surprises} perform a theoretical analysis of student-teacher linear models (with random features), showing the double descent curve theoretically. \citet{Advani2017HighdimensionalDO} also performed a similar analysis. \citet{hastie2019surprises} is the only one to theoretically analyze variance. Their work differs from ours in that we run experiments with neural networks on complex, real data, while they carry out a theoretical analysis of linear models in a simplified teacher (data generating distribution) setting.




\section{PRELIMINARIES}
\label{sec:preliminaries}

\subsection{Set-up}
\label{sec:setup}

We consider the typical supervised learning task of predicting  an output $y \in \mathcal{Y}$ from an input $x \in \mathcal{X}$, where the pairs $(x, y)$ are drawn from some unknown joint distribution, $\mathcal{D}$. The learning problem consists of learning a function $h_S:  \mathcal{X} \to \mathcal{Y}$ from a finite training dataset $S$ of $m$ i.i.d.\  samples from $\mathcal{D}$. The quality of a predictor $h$ can quantified by the expected error,
\beq \label{populrisk}
\mathcal{E}(h) = \E_{(x, y) \sim \mathcal{D}} \, \ell(h(x), y) \, ,
\eeq
for some loss function $\ell : \mathcal{Y} \times \mathcal{Y} \to \R$. 

In this paper, predictors $h_{\theta}$ are parameterized by the weights $\theta \in \R^N$ of neural networks.
We consider
the average performance over possible training sets (denoted by the random variable $S$) of size $m$. This is the same quantity \citet{geman} consider. While $S$ is the only random quantity studied in the traditional bias-variance decomposition, we also study randomness coming from optimization. We denote the random variable for optimization randomness (e.g.\ initialization) by $O$.

Formally, given a fixed training set $S$ and fixed optimization randomness $O$, the learning algorithm $\A$ produces $\theta$ = $\A(S, O)$. Randomness in optimization translates to randomness in $\A(S, \cdot)$. Given a fixed training set, we encode the randomness due to $O$ in a conditional distribution $p(\theta |S)$.
Marginalizing over the training set $S$ of size $m$ gives a marginal distribution $p(\theta)=\E_S p(\theta | S) $ on the weights learned by $\mathcal{A}$ from $m$ samples. In this context, the average performance of the learning algorithm using training sets of size $m$ can be expressed in the following ways:
\beq \label{fullrisk}
\mathcal{R}_m = \E_{\theta \sim p} \mathcal{E}(h_\theta) = \E_S \E_{\theta\sim p(\cdot |S)} \mathcal{E}(h_\theta) = \E_S \E_O \mathcal{E}(h_\theta)
\eeq

\subsection{Bias-variance decomposition}
\label{Sec:bv}

We briefly recall the standard bias-variance decomposition in the case of squared-loss. We work in the context of classification, where each  class $k \in \{1\cdots K\}$ is represented by a one-hot vector in $\mathbb{R}^K$. The predictor outputs a score  or probability vector in $\R^K$. In this context, the risk in \cref{fullrisk}   decomposes into three sources of error \citep{geman}:
\beq  \label{bv}
\mathcal{R}_m = \enoise + \ebias + \evar 
\eeq  
The first term is an intrinsic error term independent of the predictor;  the second is a bias term:
\begin{equation*}
\enoise = \E_{(x,y)} \left[\|y - \bar{y}(x)\|^2 \right],
\end{equation*}
\begin{equation*}
\ebias = \E_{x} \left[\|\E_\theta [h_\theta(x)] - \bar{y}(x)\|^2 \right],
\end{equation*}
where  $\bar{y}(x)$ denotes the expectation $\E[y | x]$ of $y$ given $x$. The third term is the expected variance of the output predictions: 
\begin{equation*}
    \evar = \E_{x}  \Var(h_\theta(x)),
\end{equation*}
\begin{equation*}
    \Var(h_\theta(x)) =
    \E_\theta \left[ \|h_\theta(x) - \E_\theta[h_\theta(x)] \|^2\right],
 \end{equation*}
where the expectation over $\theta$ can be done as in \cref{fullrisk}. Interpreting this bias-variance decomposition as a bias-variance tradeoff is quite pervasive (see, e.g., \citet[Chapter 2.9]{hastie01statisticallearning}, \citet[5.4.4]{Goodfellow-et-al-2016}, \citet[Chapter 3.2]{Bishop:2006}). It is generally invoked to emphasize that the model selected should be of the complexity that achieves the optimal balance between bias and variance.

Note that risks computed with classification losses (e.g cross-entropy or 0-1 loss)  do not have such a clean bias-variance decomposition \citep{Domingos00aunified, James03varianceand}. However, it is natural to expect that bias and variance are useful indicators of the performance of models that are not assessed with squared error. In fact, we show the classification risk can be bounded as 4 times the regression risk in \cref{app:classification_regression_relation}. To empirically examine this connection, in all of our graphs that have ``test error'' or ``training error'' on some classification task, we plot the 0-1 classification error (see, e.g., \cref{fig:small_data_tuned_error}).

\subsection{Further decomposing variance into its sources}
\label{sec:decomposition}

In the set-up of \cref{sec:setup} 
the prediction is a random variable that depends on two sources of randomness:
the randomly drawn training set, $S$,
and any optimization randomness, $O$, encoded into the conditional $p(\cdot |S)$.
In certain regimes, one gets significantly different predictions when using a different initialization.
Similarly, the output of a learned predictor changes when we use a different training set.
How do we start disentangling variance caused by sampling from variance caused by optimization?  
There are few different ways; here we describe one of them.

Our goal is to measure prediction variance due to sampling,
while controlling for the effect of optimization randomness.
\begin{definition}[(Ensemble) Variance due to sampling]
We consider the variance of an ensemble of infinitely many predictors with different optimization randomness (e.g.\ random initializations):
$$
    \var_{S}\left(
        \E_O\left[
            h_\theta(x) | S
        \right]\right).
        $$
\end{definition}

A common practice to estimate variance due to optimization effects is to run multiple seeds on a fixed training set. 
\begin{definition}[(Mean) Variance due to optimization]
We consider the average (over training sets) variance over optimization randomness for a fixed training set: $$ \E_{S}\left[
        \var_O\left(
            h_\theta(x) | S
        \right)
    \right].
$$
\end{definition}

The law of total variance naturally decomposes variance into these very terms:
\begin{align}
    \label{eqn:total-variance}
    \Var(h_\theta(x))
    = 
    &\E_{S}\left[
        \var_O\left(
            h_\theta(x) | S
        \right)
    \right]
        +
    \var_{S}\left(
        \E_O\left[
            h_\theta(x) | S
        \right]
    \right) 
\end{align}
We use this decomposition to get a finer understanding of our observations (\cref{fig:all_variances}).


\begin{figure*}[t]
    \centering
    \begin{subfigure}[t]{0.32\textwidth}
        \centering
         \includegraphics[width=\textwidth]{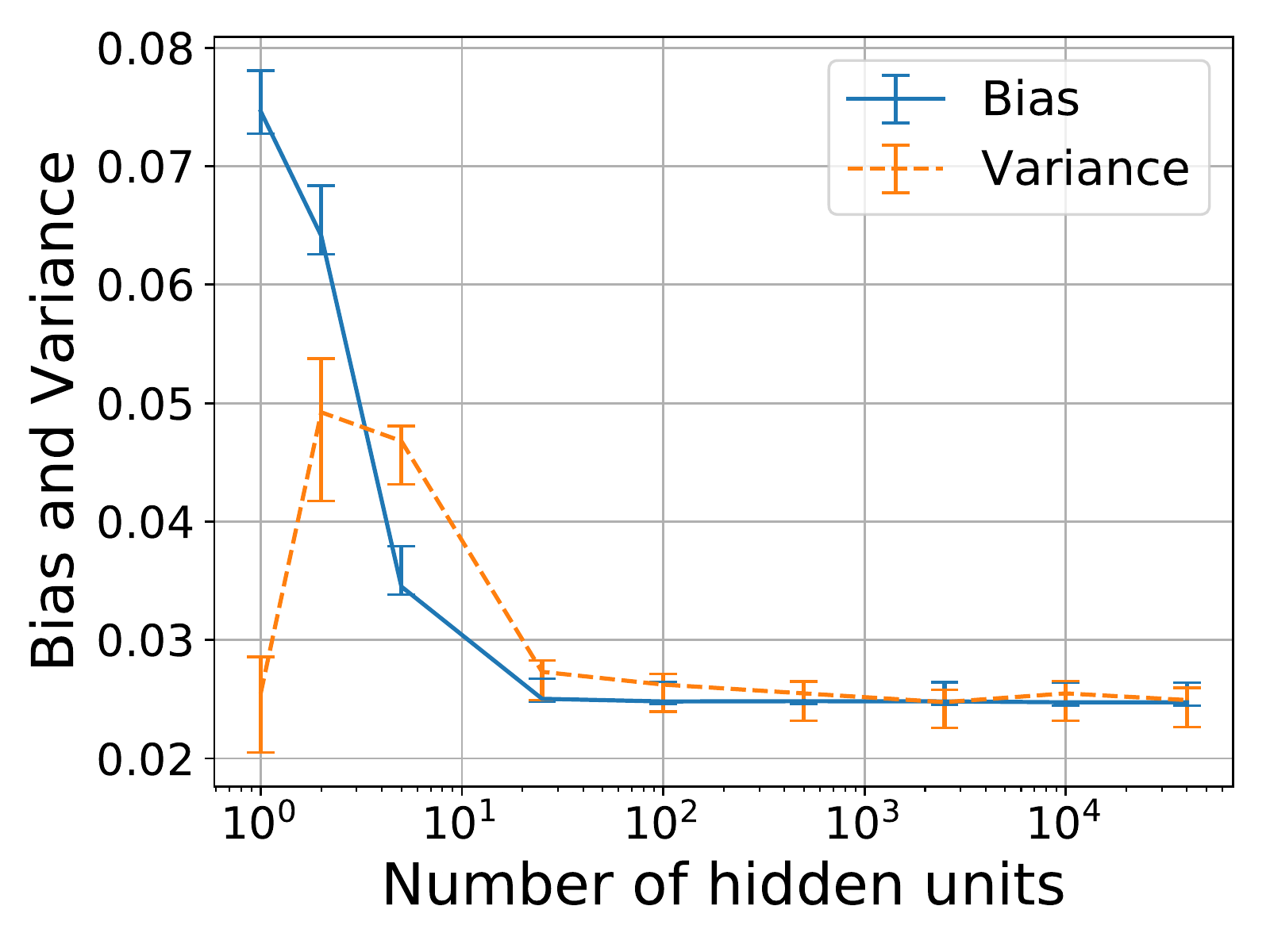}
         \caption{Variance decreases with width, even in the small MNIST setting.}
        \label{fig:small_data_tuned_bv}
    \end{subfigure}
    \hfill
    \begin{subfigure}[t]{0.32\textwidth}
        \centering
         \includegraphics[width=\textwidth]{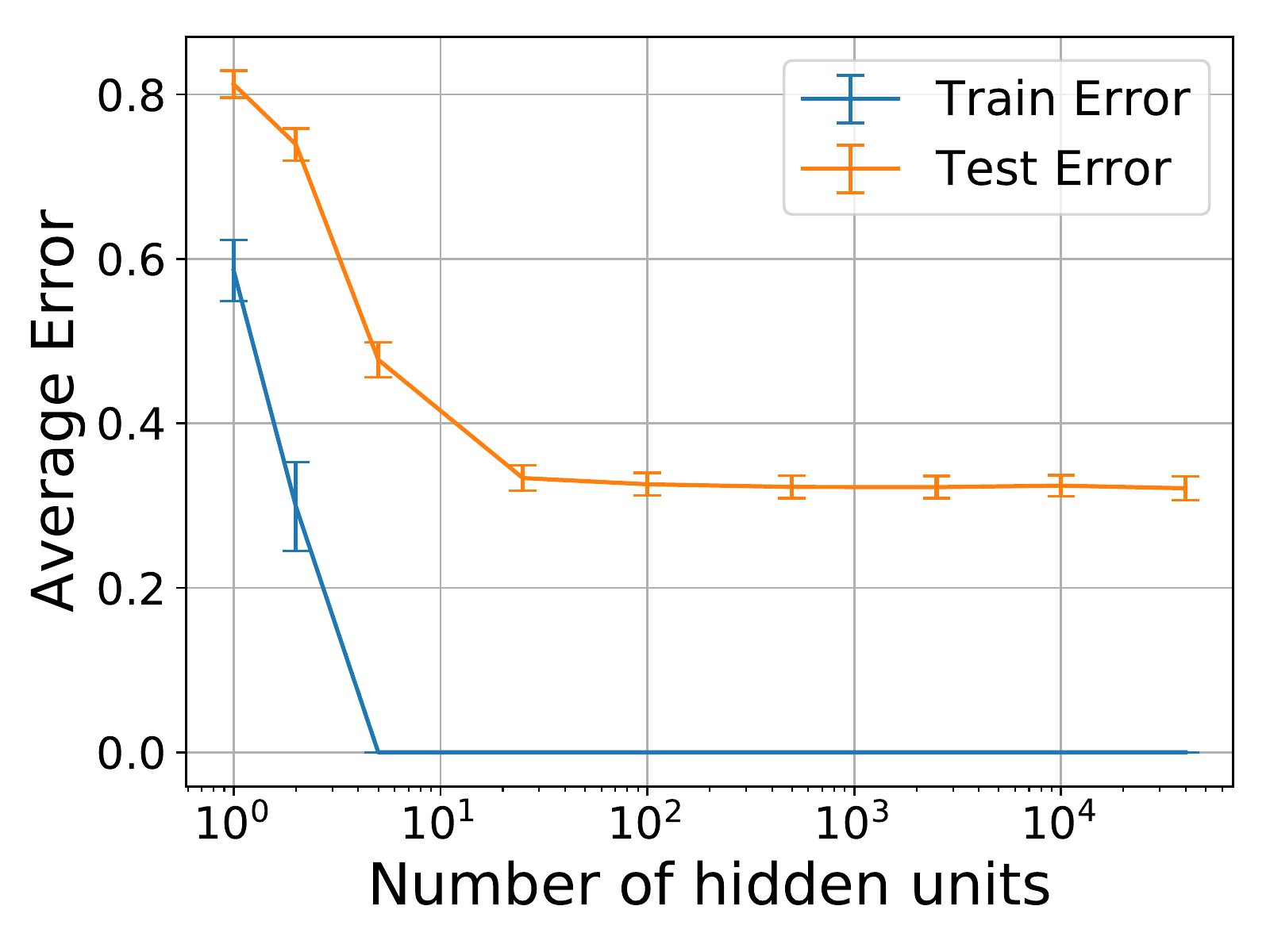}
     \caption{Test error trend is same as bias-variance trend (small MNIST).}
        \label{fig:small_data_tuned_error}
    \end{subfigure}
    \hfill
    \begin{subfigure}[t]{0.32\textwidth}
        \centering
        \includegraphics[width=\textwidth]{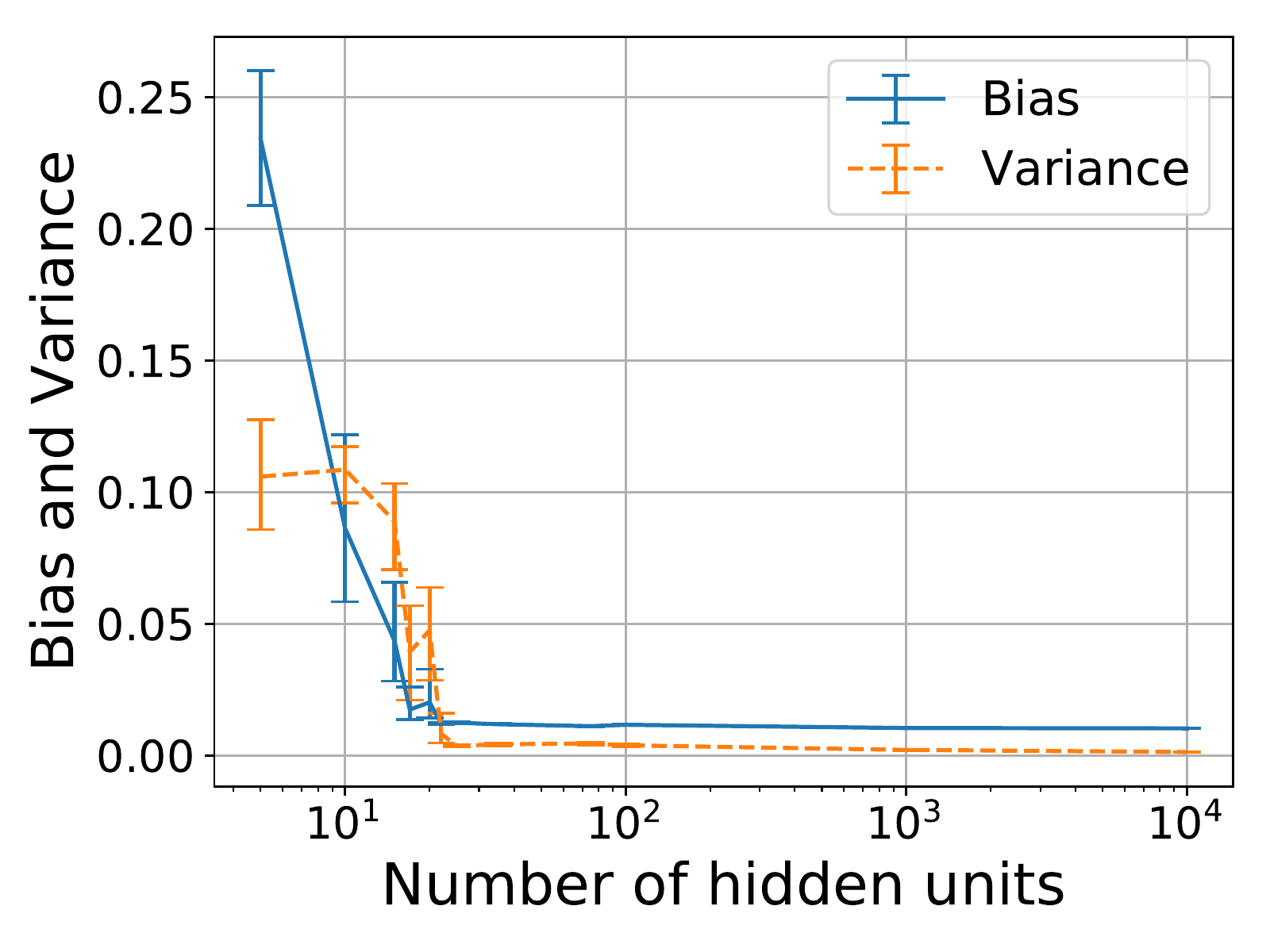}
         \caption{Similar bias-variance trends on sinusoid regression task.}
         \label{fig:sinusoid_bv}
    \end{subfigure}
    \caption{We see the same bias-variance trends in small data settings: small MNIST (left) and a regression setting (right).}
\end{figure*}

\section{EXPERIMENTS}
\label{sec:width}
In this section, we study how variance of fully connected single hidden layer networks varies with width. We provide evidence against \citet{geman}'s important claim about neural networks: \begin{quote}
    ``The basic trend is what we expect: bias falls and variance increases with the number of hidden units.''
\end{quote}
Our main finding is that, for all tasks that we study, bias and variance both decrease as we scale network width.
We also provide a meaningful decomposition of prediction variance into a variance due to sampling term and a variance due to optimization term.

\subsection{Common experimental details} 

We run experiments on different datasets: MNIST, SVHN, CIFAR10, small MNIST, and a sinusoid regression task. Averages over data samples are performed by taking the training set $S$ and creating 50 bootstrap replicate training sets $S'$ by sampling with replacement from $S$. We train 50 different neural networks for each hidden layer size using these different training sets. Then, we estimate $\ebias$\footnote{Because we do not have access to $\bar{y}$, we use the labels $y$ to estimate $\ebias$. This is equivalent to assuming noiseless labels and is standard procedure for estimating bias \citep{Kohavi:1996, Domingos00aunified}.} and $\evar$ as in Section \ref{Sec:bv}, where the population expectation  $\E_x$ is estimated with an average over the test set. To estimate the two terms from the law of total variance (\autoref{eqn:total-variance}), we use 10 random seeds for the outer expectation and 10 for the inner expectation, resulting in a total of 100 neural networks for each hidden layer size. Furthermore, we compute 99\% confidence intervals for our bias and variance estimates using the bootstrap \citep{efron1979}.

The networks are initialized using PyTorch's default initialization, which scales the variance of the weight initialization distribution inversely proportional to the width \citep{LeCun:1998,xavier2010}. The networks are trained using SGD with momentum and generally run for long after 100\% training set accuracy is reached (e.g.\ 500 epochs for full data MNIST and 10000 epochs for small data MNIST). The overall trends we find are robust to how long the networks are trained after the training error converges. The step size hyperparameter is specified in each of the sections, and the momentum hyperparameter is always set to 0.9. To make our study as general as possible, we consider networks without regularization bells and whistles such as weight decay, dropout, or data augmentation, which \citet{zhang} found to not be necessary for good generalization.

\subsection{Decreasing variance in full data setting}

We find a clear decreasing trend in variance with width of the network in the full data MNIST setting (\cref{fig:main_common_intuition_wrong}). We also see the same trend with CIFAR10 (\cref{app:CIFAR10_width}) and SVHN (\cref{app:SVHN_width}). In these experiments, the same step size is used for all networks for a given dataset (0.1 for MNIST and 0.005 for CIFAR10 and SVHN). The trend is the same with or without early stopping, so early stopping is not necessary to see decreasing variance, similar to how it was not necessary to see better test set performance with width in \citet{DBLP:journals/corr/NeyshaburTS14}. Wider ResNets are known to achieve lower test error \citep{wide_resnet}; this likely translates to decreasing variance with width in convolutional networks as well. Much of the over-parameterization literature focuses on over-parameterization in width; interestingly, the variance trend is not the same when varying depth (\cref{app:depth}).

\begin{figure*}[t]
    \centering
    \begin{subfigure}[t]{0.32\textwidth}
        \centering
        \includegraphics[width=\textwidth]{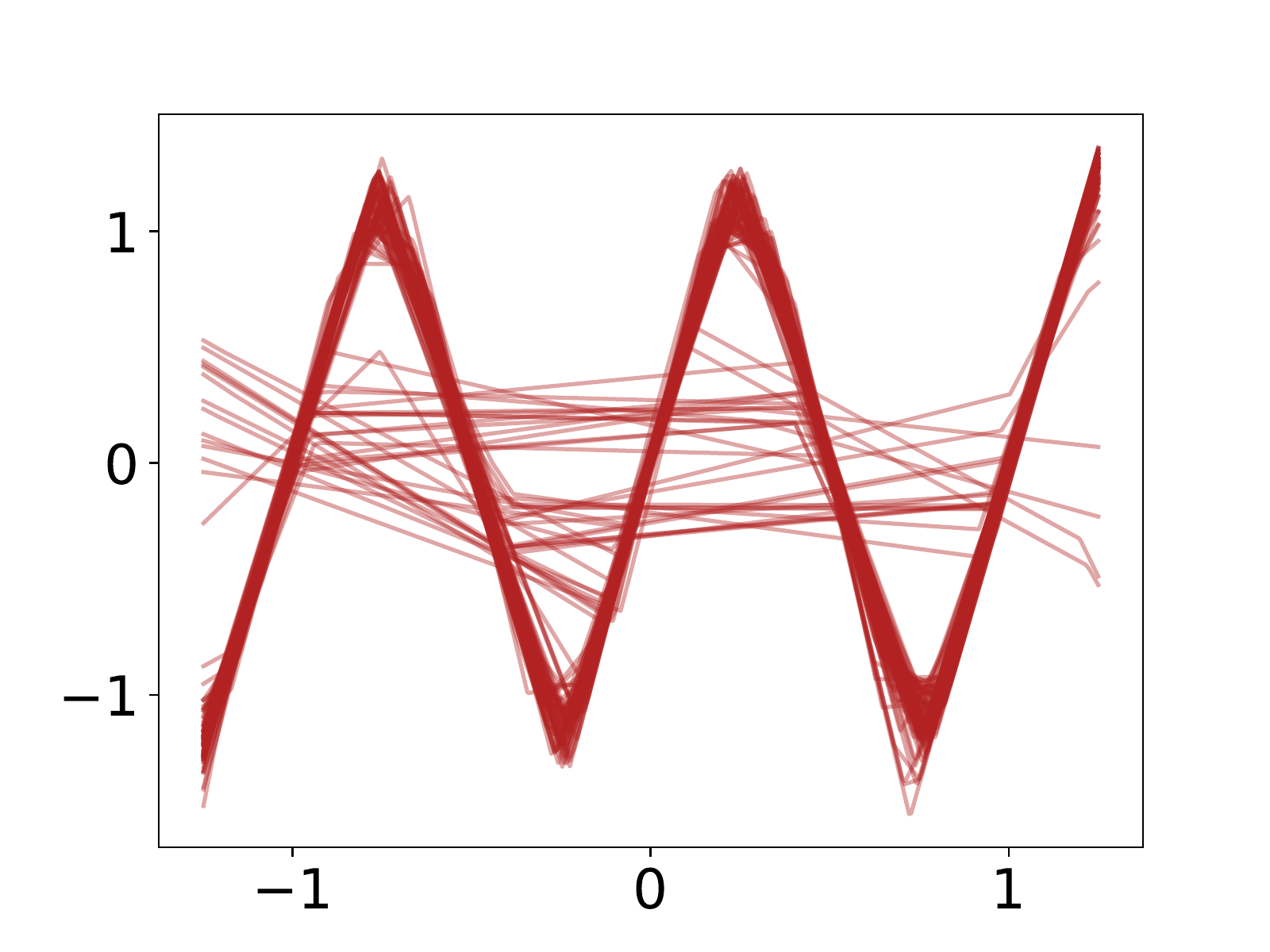}
    \end{subfigure}
    \hfill
    \begin{subfigure}[t]{0.32\textwidth}
        \centering
         \includegraphics[width=\textwidth]{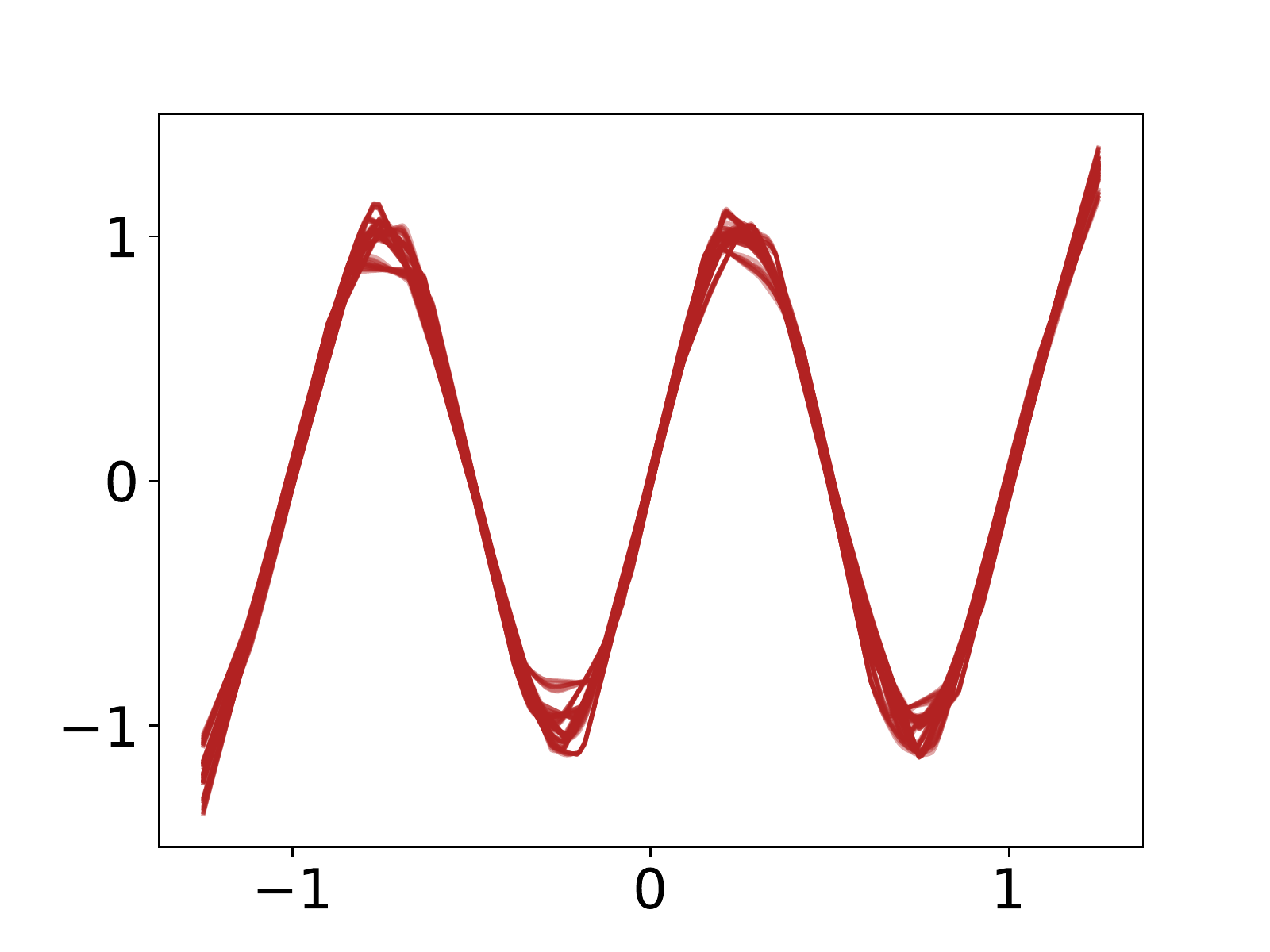}
    \end{subfigure}
    \hfill
    \begin{subfigure}[t]{0.32\textwidth}
        \centering
         \includegraphics[width=\textwidth]{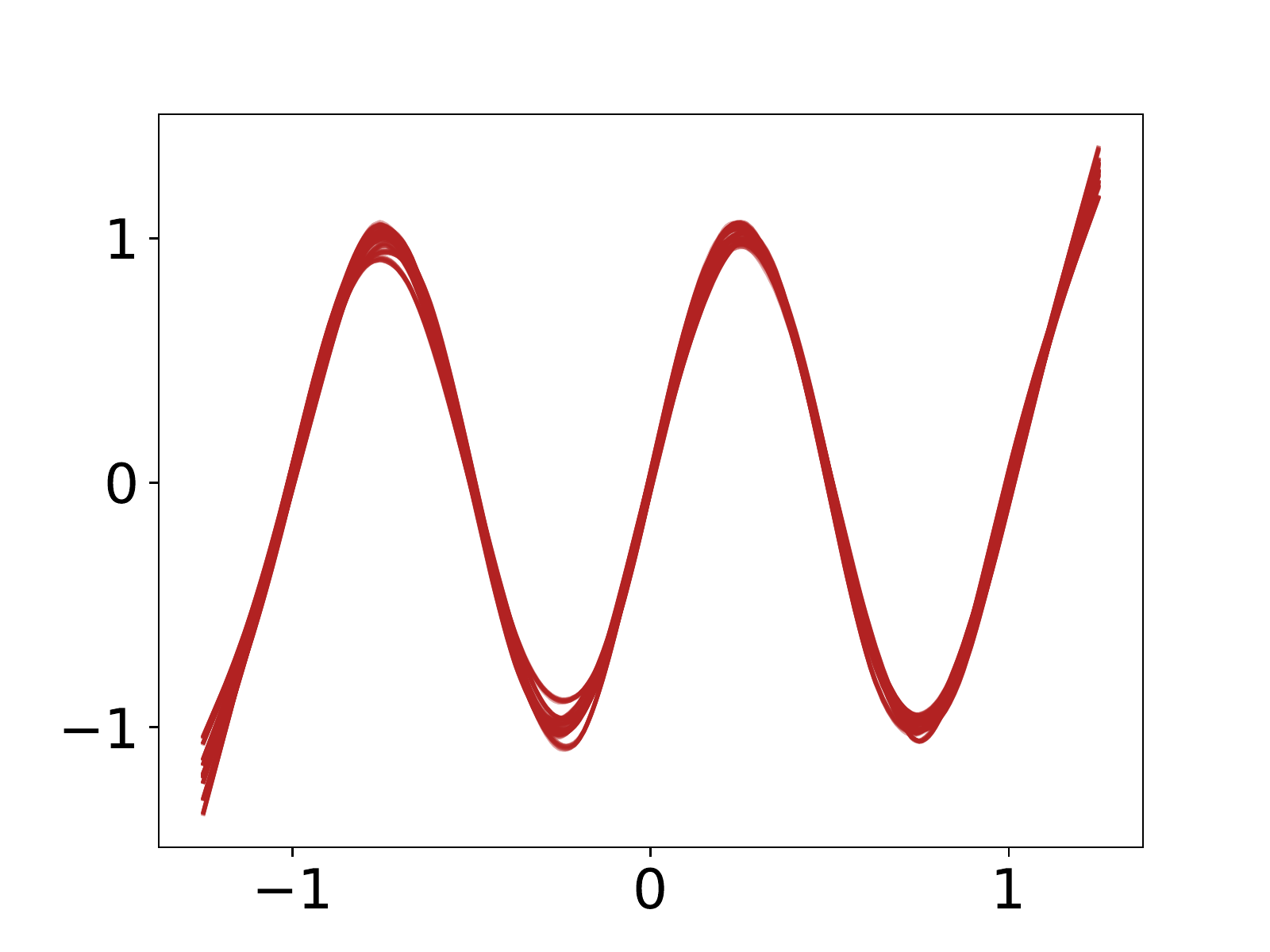}
    \end{subfigure}
    \caption{Visualization of the 100 different learned functions of single hidden layer neural networks of widths 15, 1000, and 10000 (from left to right) on the task of learning a sinusoid. The learned functions are increasingly similar with width, suggesting decreasing variance.
    More in \cref{app:sinusoid_regression}.}
    \label{fig:sinusoid}
\end{figure*}

\subsection{Testing the limits: decreasing variance in the small data setting}
\label{sec:small_data}

Decreasing the size of the dataset can only increase variance. To study the robustness of the above observation, we decrease the size of the training set to just 100 examples. In this small data setting, somewhat surprisingly, we still see that \emph{both} bias \emph{and} variance decrease with width (\cref{fig:small_data_tuned_bv}). The test error behaves similarly (\cref{fig:small_data_tuned_error}). Because performance is more sensitive to step size in the small data setting, the step size for each network size is tuned using a validation set (see \cref{app:tuned_lr} for step sizes). The training for tuning is stopped after 1000 epochs, whereas the training for the final models is stopped after 10000 epochs.  Note that because we see decreasing bias with width, effective capacity is, indeed, increasing while variance is decreasing.

One control that motivates the experimental design choice of optimal step size is that it leads to the conventional decreasing bias trend (\cref{fig:small_data_tuned_bv}) that indicates increasing effective capacity. In fact, in the corresponding experiment where step size is the same 0.01 for all network sizes, we do not see monotonically decreasing bias (\cref{app:fixed_lr}). 

This sensitivity to step size in the small data setting is evidence that we are testing the limits of our hypothesis. By looking at the small data setting, we are able to test our hypothesis when the ratio of size of network to dataset size is quite large, and we still find this decreasing trend in variance (\cref{fig:small_data_tuned_bv}). 

To see how dependent this phenomenon is on SGD, we also run these experiments using batch gradient descent and PyTorch's version of LBFGS. Interestingly, we find a decreasing variance trend with those optimizers as well. These experiments are included in \cref{app:other_optimizers}.

\subsection{Decoupling variance due to sampling from variance due to optimization}
\label{sec:width_var_decoupling}

In order to better understand this variance phenomenon in neural networks, we separate the variance due to sampling from the variance due to optimization, according to the law of total variance (\autoref{eqn:total-variance}). Contrary to what traditional bias-variance tradeoff intuition would suggest, we find variance due to sampling increases slowly and levels off, once sufficiently over-parameterized (\cref{fig:all_variances}). Furthermore, we find that variance due to optimization decreases with width, causing the total variance to decrease with width (\cref{fig:all_variances}).

A body of recent work has provided evidence that over-parameterization (in width) helps gradient descent optimize to global minima in neural networks \citep{globalminima_iclr2019, pmlr-v80-du18a, DBLP:journals/corr/SoltanolkotabiJ17, NIPS2014_5267, DBLP:journals/corr/abs-1801-02254}. Always reaching a global minimum implies low variance due to optimization on the \textit{training set}. Our observation of decreasing variance on the \textit{test set} shows that the over-parameterization (in width) effect on optimization seems to extend to generalization, on the data sets we consider.

\subsection{Visualization with regression on sinusoid}

We trained different width neural networks on a noisy sinusoidal distribution with 80 independent training examples. This sinusoid regression setting also exhibits the familiar bias-variance trends (\cref{fig:sinusoid_bv}) and trends of the two components of the variance and the test error (\cref{fig:sinusoid_curves_app} of \cref{app:sinusoid_regression}).

Because this setting is low-dimensional, we can visualize the learned functions. The classic caricature of high capacity models is that they fit the training data in a very erratic way (example in \cref{fig:high_var_caricature} of \cref{app:sinusoid_regression}). We find that wider networks learn sinusoidal functions that are much more similar than the functions learned by their narrower counterparts (\cref{fig:sinusoid}). We have analogous plots for all of the other widths and ones that visualize the variance similar to how it is commonly visualized for Gaussian processes in \cref{app:sinusoid_regression}.

\section{DISCUSSION AND THEORETICAL INSIGHTS}
\label{sec:theory}

Our empirical results demonstrate that in the practical setting, variance due to optimization decreases with network width while variance due to sampling increases slowly and levels off once sufficiently over-parameterized. In \cref{sec:linear_models}, we discuss the simple case of linear models and point out that non-increasing variance can already be seen in the over-parameterized setting. In \cref{sec:back-to-nn} we take inspiration from linear models to provide arguments for the behavior of variance in increasingly wide neural networks, and we discuss the assumptions we make.

\subsection{Insights from linear models}
\label{sec:linear_models}

In this section, we review the classic result that the variance of a linear model grows with the number of parameters \citep[Section 7.3]{hastie_09} and point out that variance behaves differently in the over-parameterized setting.

We consider least-squares linear regression in a standard setting which assumes a noisy linear mapping $y = \theta^T x + \epsilon$ between input feature vectors $x\in \R^N$ and real outputs, where $\epsilon$ denotes the noise random variable with $\E[\epsilon] = 0$ and $\Var(\epsilon) = \sigma_{\epsilon}^2$. In this context, the over-parameterized setting is when the dimension $N$ of the input space is larger than the number $m$ of examples.

Let $X$ denote the $m \times N$ design matrix whose $i$\textsuperscript{th} row is the training point $x_i^T$, let $Y$ denote the corresponding labels, and let $\Sigma = X^T X$ denote the empirical covariance matrix.  
We consider the fixed-design setting where $X$ is fixed, so all of the randomness due to data sampling comes solely from $\epsilon$. $\A$ learns weights $\hat{\theta}$ from $(X, Y)$, either by a closed-form solution or by gradient descent, using a standard initialization $\theta_0 \sim \mathcal{N}(0, \frac{1}{N} I)$. The predictor makes a prediction on $x \sim \D$: $h(x) = \hat{\theta}^T x$. Then, the quantity we care about is $\E_x \Var(h(x))$.

\subsubsection{Under-parameterized setting}

The case where $N \leq m$ is standard: if  $X$ has maximal rank, $\Sigma$ is invertible; the solution is independent of the initialization and given by $\hat{\theta} = \Sigma^{-1} X^T Y$. All of the variance is a result of randomness in the noise $\epsilon$.  For a fixed $x$,
\begin{equation}
    \Var(h(x)) = \sigma_\epsilon^2 \Tr(x x^T \Sigma^{-1}) \, .
\end{equation}
This grows with the number of parameters $N$. For example, taking the expected value over the empirical distribution, $\hat{p}$, of the sample, we recover that the variance grows with $N$:
\begin{equation}
    \E_{x \sim \hat{p}} [\Var(h(x))] = \frac{N}{m} \sigma_\epsilon^2 \, .
\end{equation}
We provide a reproduction of the proofs in \cref{app:linear_underparam}.

\subsubsection{Over-parameterized setting}

The over-parameterized case where $N > m$ is more interesting: even if $X$ has maximal rank,  $\Sigma$ is not invertible. This leads to a subspace of solutions, but gradient descent yields a unique solution from updates that belong to the span of the training points $x_i$ (row space of $X$) \citep{lecun1991}, which is of dimension $r = \rank(X) = \rank(\Sigma)$. Correspondingly, no learning occurs in the null space of $X$, which is of dimension $N - r$. 
Therefore, gradient descent yields the solution that is closest to initialization: $\hat{\theta} = P_\perp(\theta_0) + \Sigma^+ X^T Y$, where $P_{\perp}$ projects onto the null space of $X$ and $+$ denotes the Moore-Penrose inverse. 

The variance has two contributions: one due to initialization and one due to sampling (here, the noise $\epsilon$), as in \cref{eqn:total-variance}. These are made explicit in \cref{prop:linear_model}.  
\begin{prop}[Variance in over-parameterized linear models]
\label{prop:linear_model}
Consider the over-parameterized setting where $N > m$.  For a fixed $x$, the variance decomposition of \cref{eqn:total-variance} yields
\beq \label{eq:over_lin_variance}
\Var(h(x))=  \frac{1}{N}\|P_\perp(x)\|^2 + \sigma_\epsilon^2 \Tr(x x^T \Sigma^+) \, . \eeq 
\end{prop}
This does not grow with the number of parameters $N$. In fact,  because $\Sigma^{-1}$ is replaced with $\Sigma^+$, the variance \emph{scales as the dimension of 
the data} (i.e the rank of $X$), as opposed to the number of parameters. For example, taking the expected value over the empirical distribution, $\hat{p}$, of the sample, we obtain
\begin{equation}
    \E_{x \sim \hat{p}} [\Var(h(x))] = \frac{r}{m} \sigma_\epsilon^2 \, ,
\end{equation}
where $r = \rank(X)$.
We provide the proofs for over-parameterized linear models in \cref{app:linear_overparam}.

\subsection{A more general result} 

\label{sec:back-to-nn}

We will illustrate our arguments in the following simplified setting, where $\mathcal{M}$, $\mathcal{M}^\perp$, and $d(N)$ are the more general analogs of $\rowspace(X)$, $\nullspace(X)$, and $r$ (respectively):

\textbf{Setting. }
Let $N$ be the dimension of the parameter space. 
The prediction for a fixed example $x$, given by a trained network parameterized by $\theta$ depends on:

(i) a subspace of the parameter space, $\mathcal{M} \in \mathbb{R}^N$ with relatively small dimension, $d(N)$,  which depends only on the learning task. 

(ii) parameter components corresponding to directions orthogonal to $\mathcal{M}$. The orthogonal $\mathcal{M}^\perp$ of $\mathcal{M}$ has dimension, $N-d(N)$, and is essentially irrelevant to the learning task.

We can write the parameter vector as a sum of these two components $    \theta = \theta_\mathcal{M} + \theta_{\mathcal{M}^\perp}
$. We will further make the following assumptions.
\vspace{-0.1cm}
\begin{enumerate}[label=\textbf{Assumption \arabic*}, labelindent=0pt, wide, labelwidth=!]
\item
\label{assum:opt_invariant}
The optimization of the loss function is invariant with respect to $\theta_{\mathcal{M}\perp}$.
\item
Regardless of initialization, the optimization method consistently yields a solution with the same $\theta_\mathcal{M}$ component (i.e.\ the same vector when projected onto $\mathcal{M}$).
\end{enumerate}

\subsubsection{Variance due to initialization}

\label{sec:variance-from-optimization}

Given the above assumptions, the following result shows that the variance from initialization\footnote{Among the different sources of optimization randomness, we focus on randomness from initialization and do not focus on randomness from stochastic mini-batching because we found the phenomenon of decreasing variance with width persists when using \textit{batch} gradient descent (\cref{sec:small_data}, \cref{app:other_optimizers}).} vanishes as we increase $N$. The full proof, which builds on concentration results for Gaussians (based on Levy's lemma \citep{ledoux2001concentration}), is given in \cref{app:more_general_setting}.

\newcommand{\thmp}{\theta_{\mathcal{M}^\perp}}

\begin{theorem}[Decay of variance due to initialization] \label{thm:init-var-decay}
Consider the setting of Section~\ref{sec:back-to-nn} 
Let $\theta$ denote the parameters at the end of the learning process.
Then, for a fixed data set and parameters  initialized as $\theta_0 \sim 
    \mathcal{N}(0, \frac{1}{N} I)$, the variance of the prediction satisfies the inequality, 
 \beq 
  \mbox{Var}_{\theta_0}(h_\theta(x)) \leq C \frac{2 L^2}{N}
 \eeq
where $L$ is the Lipschitz constant of the prediction with respect to $\theta$, and for some universal constant $C >O$. 
\end{theorem} 
This result guarantees that the variance decreases to zero as $N$ increases, provided 
the Lipschitz constant $L$ grows more slowly than the square root of dimension,
$L=o(\sqrt{N})$.

\subsubsection{Variance due to sampling}

Under the above assumptions,  the parameters at the end of learning take the form $\theta = \theta_\mathcal{M}^* + \theta_{0 \mathcal{M}^\perp}$.  For fixed initialization, the only source of variance of the prediction is the randomness of   $\theta_\mathcal{M}^*$ on the learning manifold. The variance depends on the parameter dimensionality only through 
\mbox{$\dim \mathcal{M} = d(N)$}, and hence remains constant if $d(N)$ does (see \citet{Li18IntDim}'s ``intrinsic dimension'').

\paragraph{Discussion on assumptions}
\label{sec:discussion_deep_net_assumptions}

We made strong assumptions,  but there is some support for them in the literature. The existence of a subspace $\mathcal{M}_\perp$ in which no learning occurs was also conjectured by \citet{Advani2017HighdimensionalDO} and  shown to hold in linear neural networks under a simplifying assumption that decouples the dynamics of the weights in different layers. \citet{Li18IntDim} 
empirically showed the existence of a critical number $d(N) = d$ of relevant parameters for a given learning task, independent of the  size of the model. 
\citet{Sagun17} showed that the spectrum of the Hessian for over-parameterized networks splits into $(i)$ a bulk centered near zero and $(ii)$ a small number of large eigenvalues; and  \citet{Gur-Ari2018} recently gave evidence that the small subspace spanned by the Hessian's top eigenvectors is preserved over long periods of training. These results suggest that learning occurs mainly in a small number of directions. 


\section{CONCLUSION AND FUTURE WORK}


We provide evidence against \citet{geman}'s claim that ``the price to pay for achieving low bias is high variance,'' finding that \emph{both} bias \emph{and} variance decrease with network width. \citet{geman}'s claim is found throughout machine learning and is meant to generally apply to all of machine learning (\cref{app:intuitions}), and it is correct in many cases (e.g.\ kNN, kernel regression, splines). Is this lack of a tradeoff specific to neural networks or is it present in other models as well such as decision trees?

We propose a new decomposition of the variance, finding variance due to sampling (analog of regular variance in simple settings) does not appear to be dependent on width, once sufficiently over-parameterized, and that variance due to optimization decreases with width. 
By taking inspiration from linear models, we perform a theoretical analysis of the variance that is consistent with our empirical observations.

We view future work that uses the bias-variance lens as promising. For example, a probabilistic notion of effective capacity of a model is natural when studying generalization through this lens (\cref{app:prob_capacity}).  We did not study how bias and variance change over the course of training; that would make an interesting direction for future work. We also see further theoretical treatment of variance as a fruitful direction for better understanding complexity and generalization abilities of neural networks.

\bibliography{bibliography}
\bibliographystyle{custom_icml2019}

\onecolumn



\newpage

\begin{appendices}
\crefalias{section}{appsec}
\crefalias{subsection}{appsec}

\section{Probabilistic notion of effective capacity}
\label{app:prob_capacity}

The problem with classical complexity measures is that they do not take into account optimization and have no notion of what will actually be learned. \citet[Section 1]{arpit17} define a notion of an \emph{effective} hypothesis class to take into account what functions are possible to be learned by the learning algorithm.

However, this still has the problem of not taking into account what hypotheses are \emph{likely} to be learned. To take into account the probabilistic nature of learning, we define the $\epsilon$-\textit{hypothesis class} for a data distribution $\D$ and learning algorithm $\A$, that contains the hypotheses which are at least $\epsilon$-likely for some $\epsilon > 0$:
\begin{equation}
    \mathcal{H}_{\D}(\A) = \{h : p(h(\A, S)) \geq \epsilon\},
    \label{eqn:def-eps-hypothesis-class}
\end{equation}

where $S$ is a training set drawn from $\D^m$, $h(\A, S)$ is a random variable drawn from the distribution over learned functions induced by $\D$ and the randomness in $\A$; $p$ is the corresponding density. Thinking about a model's $\epsilon$-hypothesis class can lead to drastically different intuitions for the complexity of a model and its variance (\cref{fig:var_spectrum}). This is at the core of the intuition for why the traditional view of bias-variance as a tradeoff does not hold in all cases.

\begin{figure}[h]
 \centering
 \includegraphics[width=.9\textwidth]{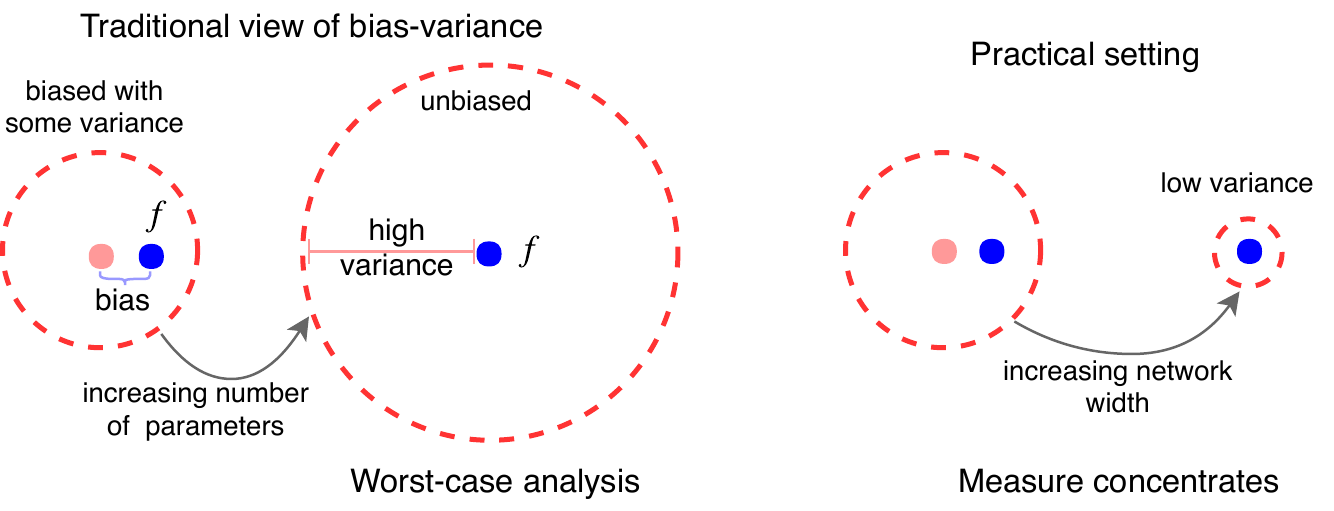}
 \caption{The dotted red circle depicts a cartoon version of the $\epsilon$-hypothesis class of the learner. The left side reflects common intuition, as informed by the bias-variance tradeoff and worst-case analysis from statistical learning theory. The right side reflects our view that variance can decrease with network width.}
 \label{fig:var_spectrum}
\end{figure}

\newpage
\section{Additional empirical results and discussion}
\label{app:empirical}





\subsection{CIFAR10}
\label{app:CIFAR10_width}

\begin{figure}[H]
    \centering
    \begin{subfigure}[t]{0.48\textwidth}
        \centering
        \includegraphics[width=\textwidth]{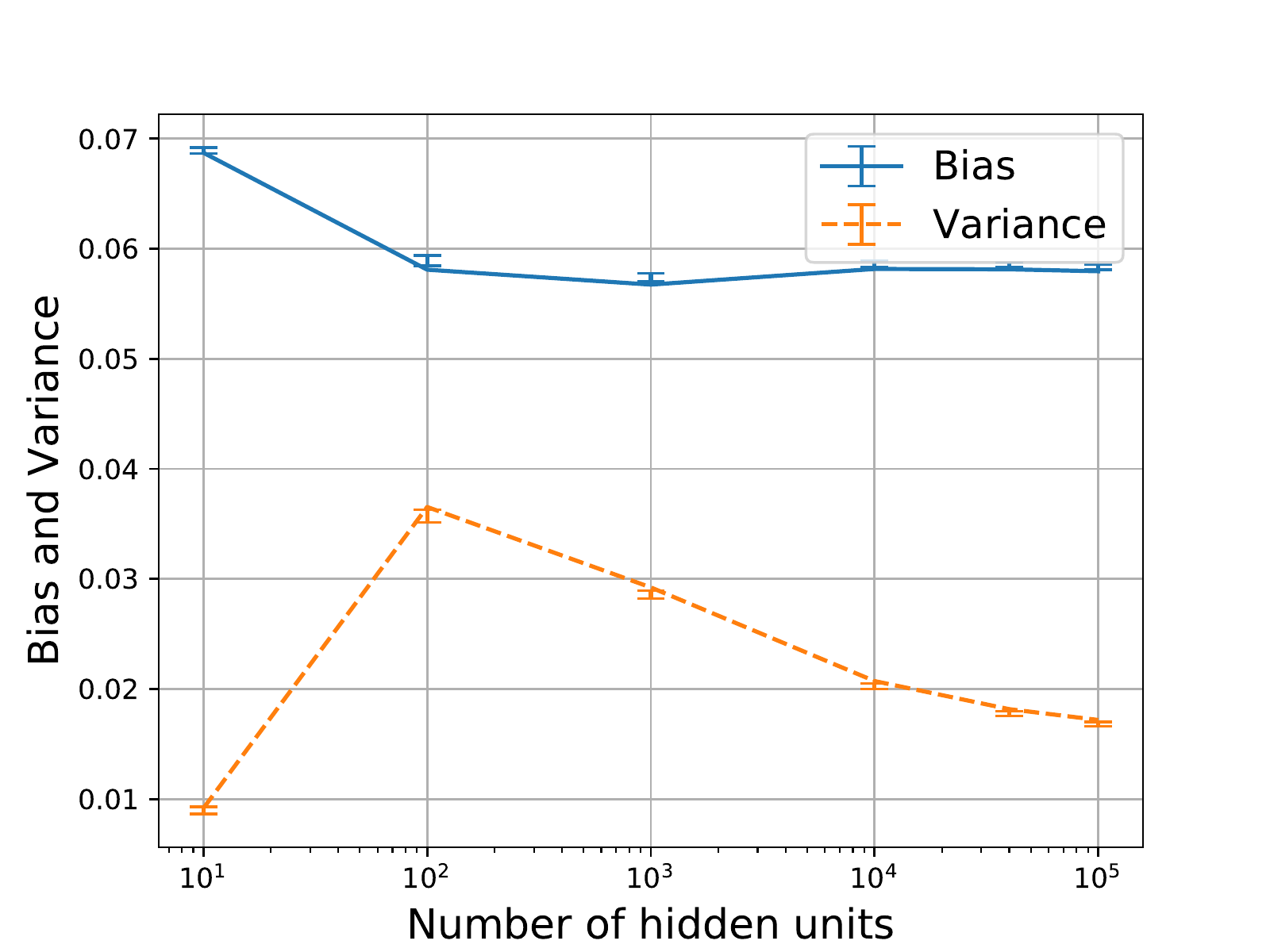}
    \end{subfigure}
    \hfill
    \begin{subfigure}[t]{0.48\textwidth}
        \centering
        \includegraphics[width=\textwidth]{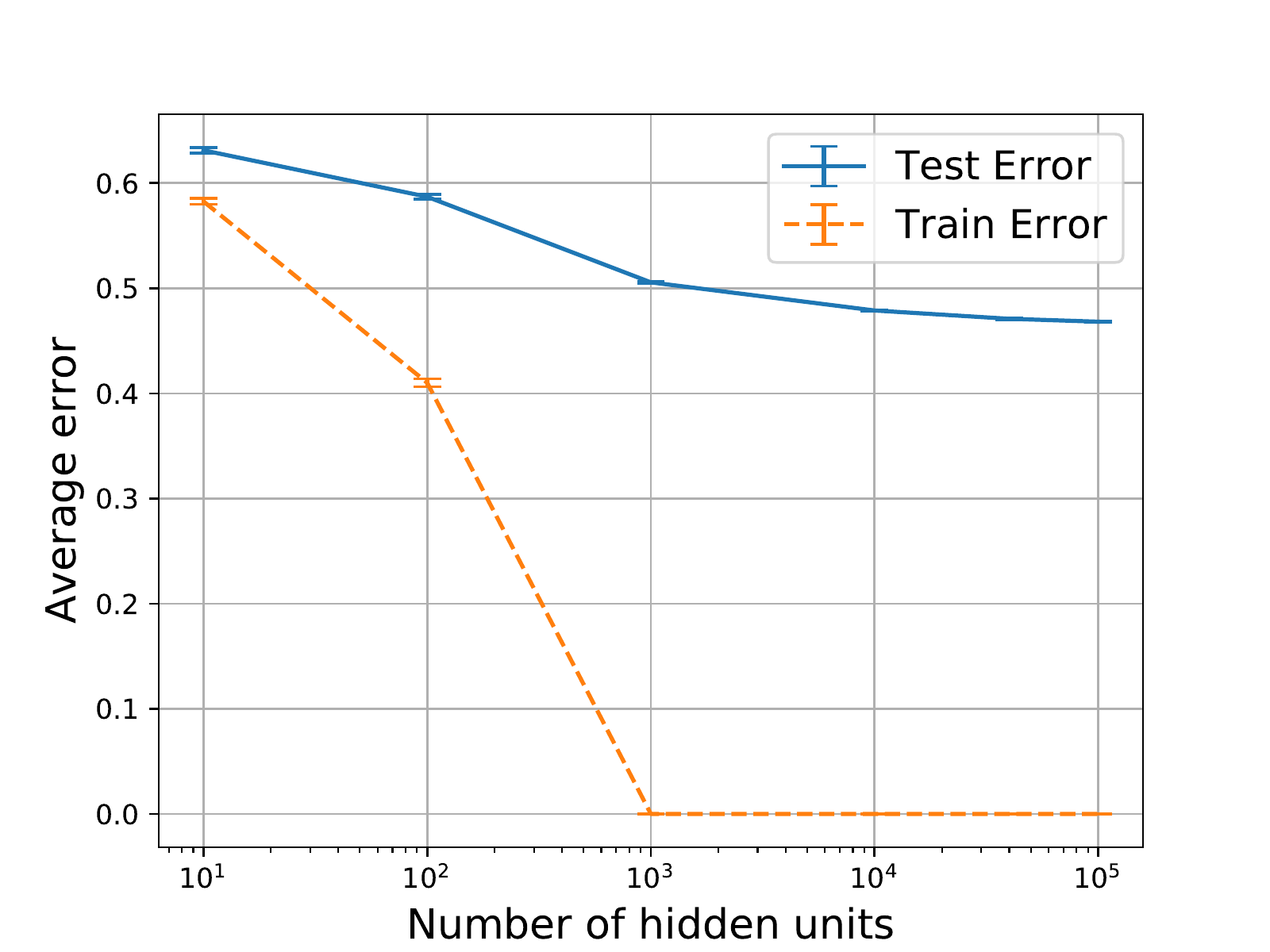}
    \end{subfigure}
    \caption{Bias-variance plot (left) and corresponding train and test error (right) for CIFAR10 after training for 150 epochs with step size 0.005 for all networks.}
\end{figure}

\begin{figure}[H]
    \centering
    \begin{subfigure}[t]{0.48\textwidth}
        \centering
        \includegraphics[width=\textwidth]{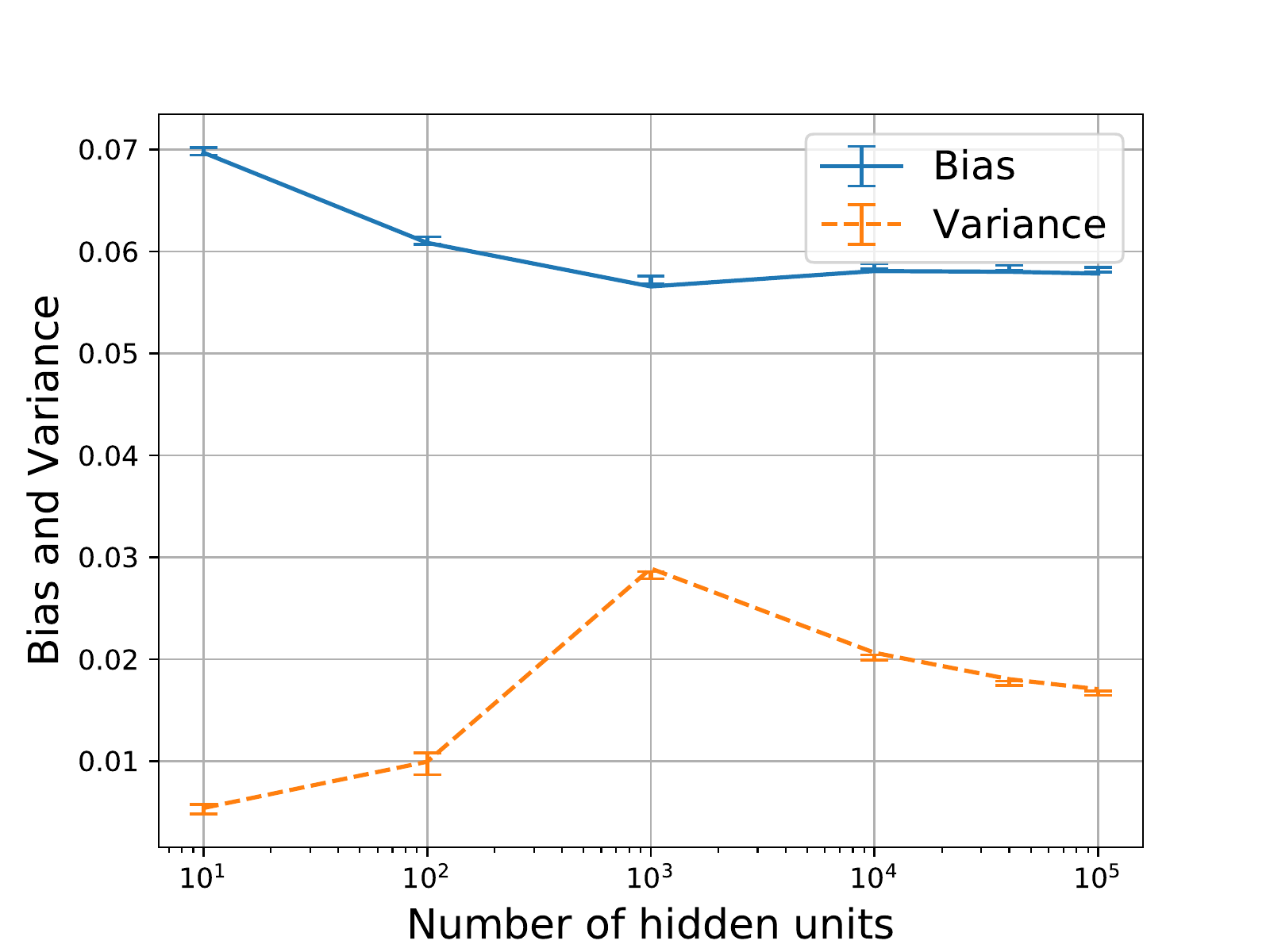}
    \end{subfigure}
    \hfill
    \begin{subfigure}[t]{0.48\textwidth}
        \centering
        \includegraphics[width=\textwidth]{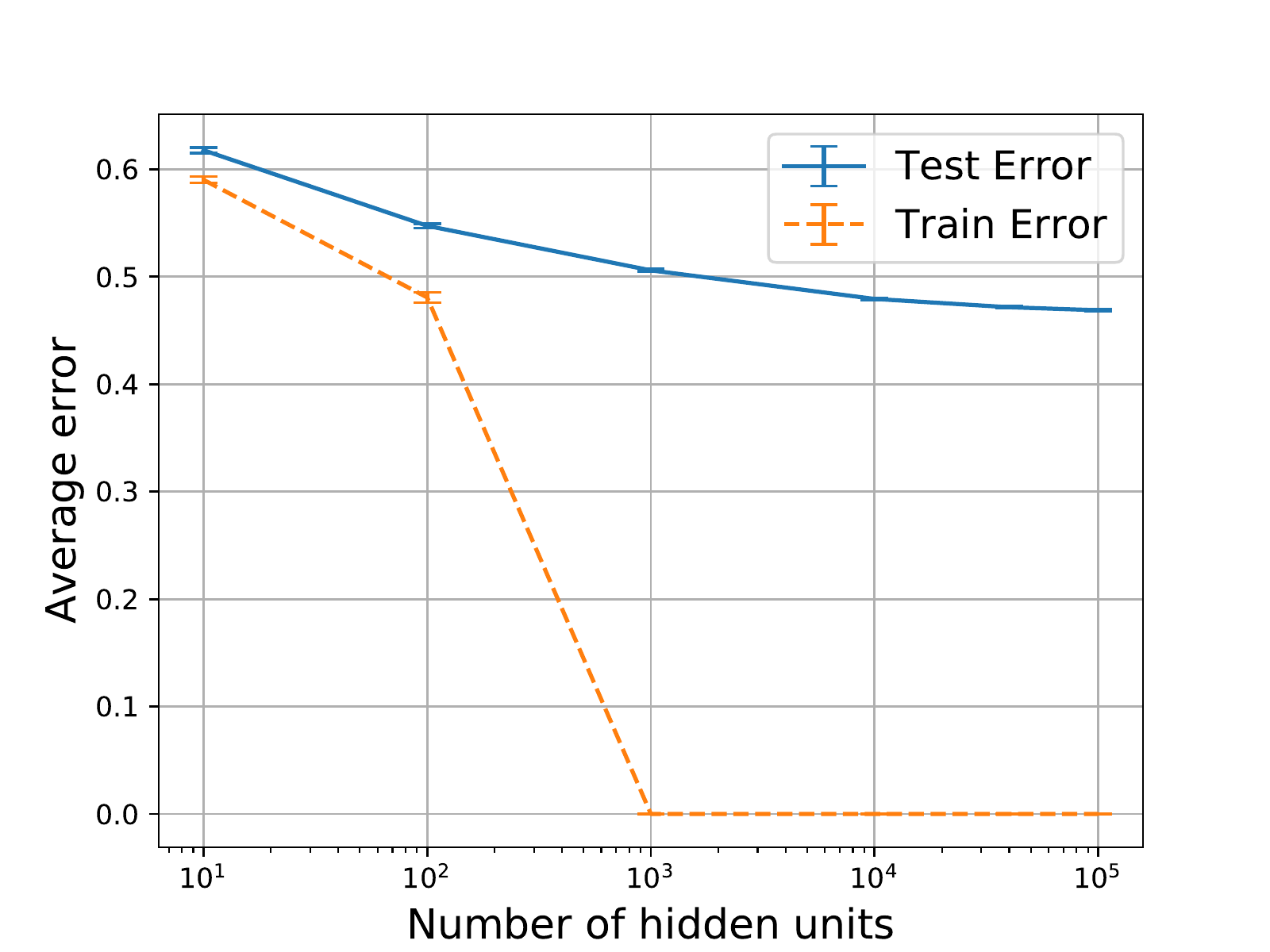}
    \end{subfigure}
    \caption{Bias-variance plot (left) and corresponding train and test error (right) for CIFAR10 after training for using \emph{early stopping} with step size 0.005 for all networks.}
\end{figure}

\subsection{SVHN}
\label{app:SVHN_width}

\begin{figure}[H]
    \centering
    \begin{subfigure}[t]{0.48\textwidth}
        \centering
        \includegraphics[width=\textwidth]{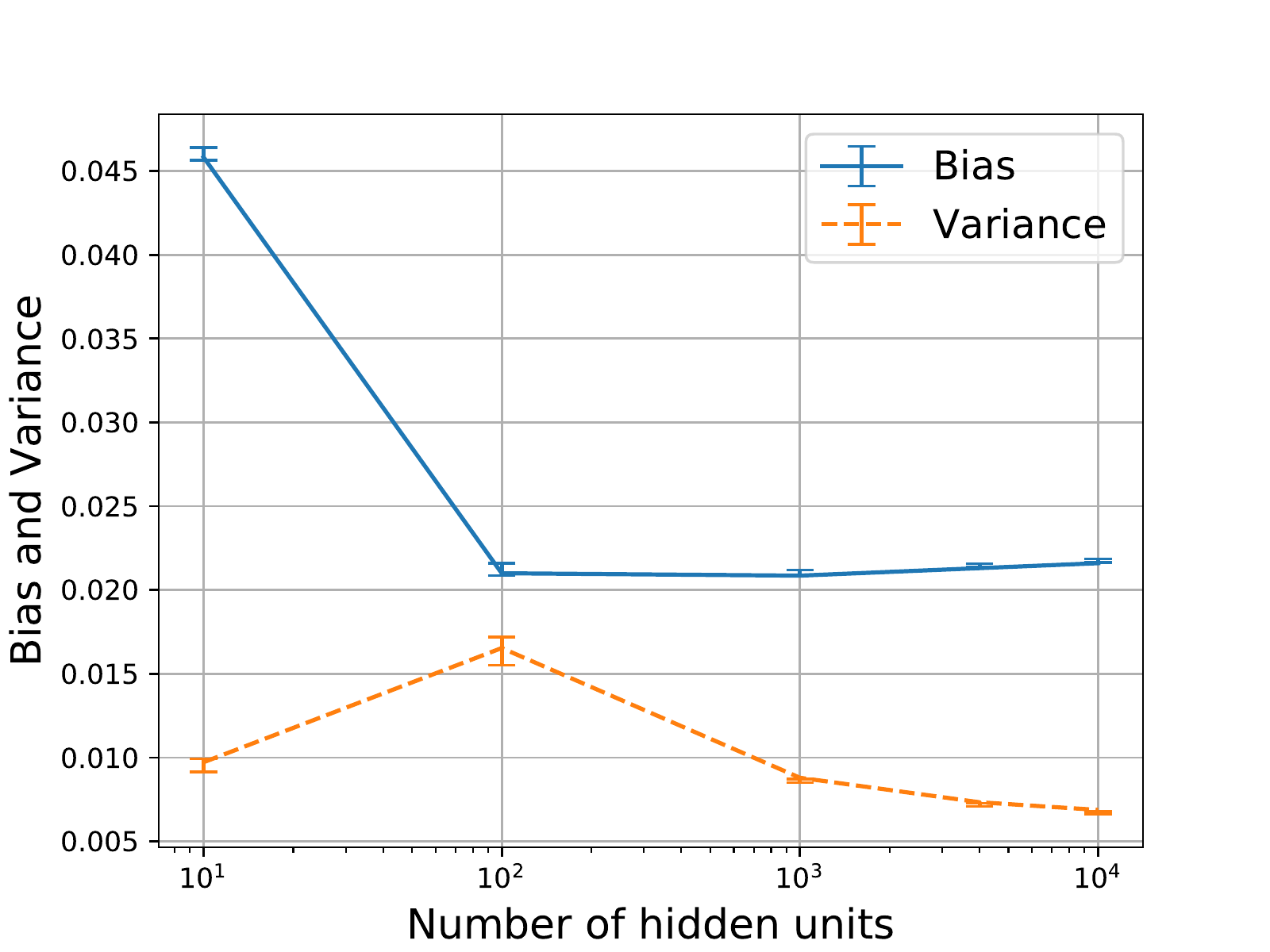}
    \end{subfigure}
    \hfill
    \begin{subfigure}[t]{0.48\textwidth}
        \centering
        \includegraphics[width=\textwidth]{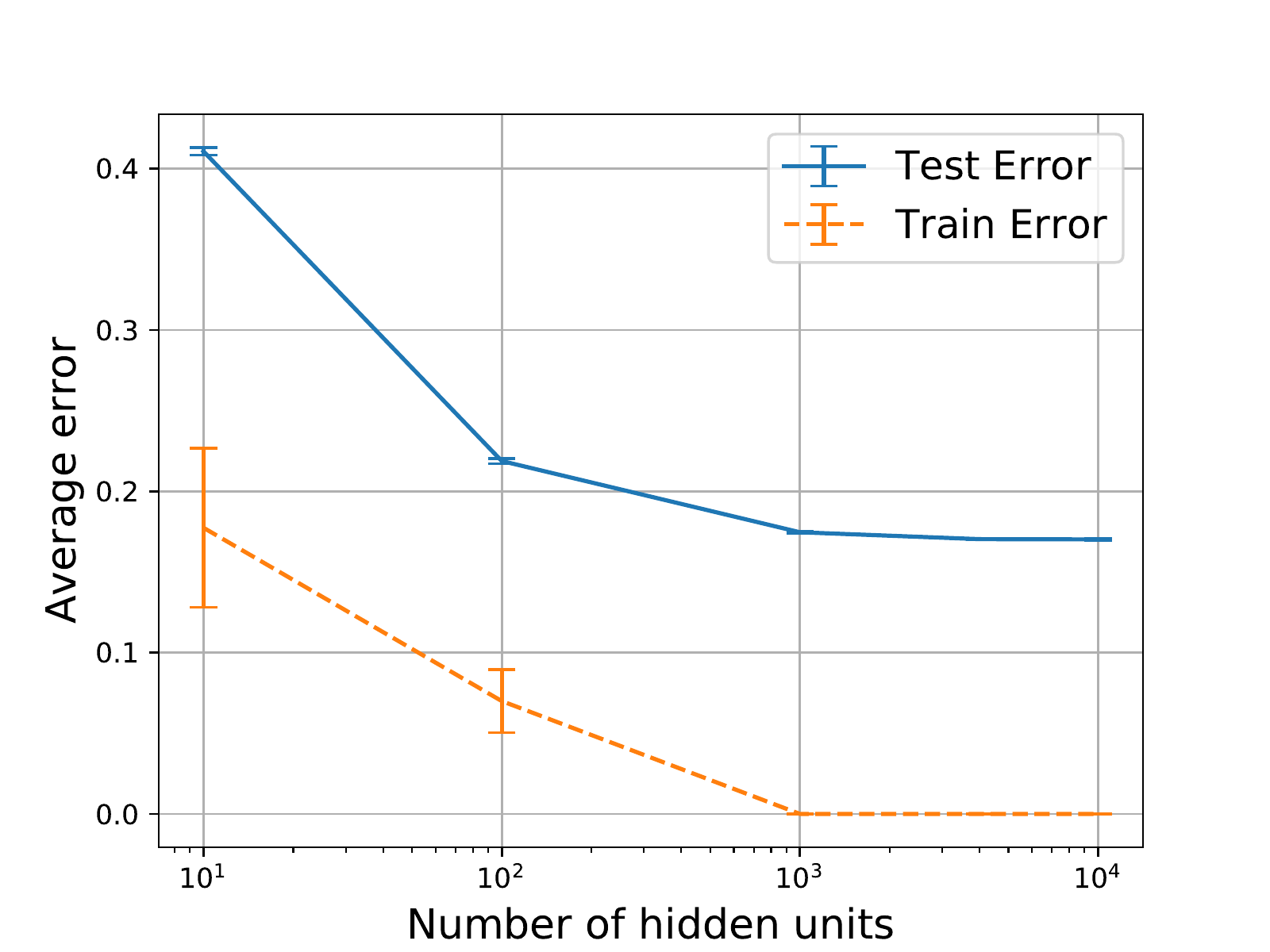}
    \end{subfigure}
    \caption{Bias-variance plot (left) and corresponding train and test error (right) for SVHN after training for 150 epochs with step size 0.005 for all networks.}
\end{figure}

\subsection{MNIST}
\label{app:MNIST_width}

\begin{figure}[H]
    \centering
    \begin{subfigure}[t]{0.48\textwidth}
        \centering
        \includegraphics[width=\textwidth]{figures/full_data_width/MNIST/bias-variance_long}
    \end{subfigure}
    \begin{subfigure}[t]{0.48\textwidth}
        \centering
         \includegraphics[width=\textwidth]{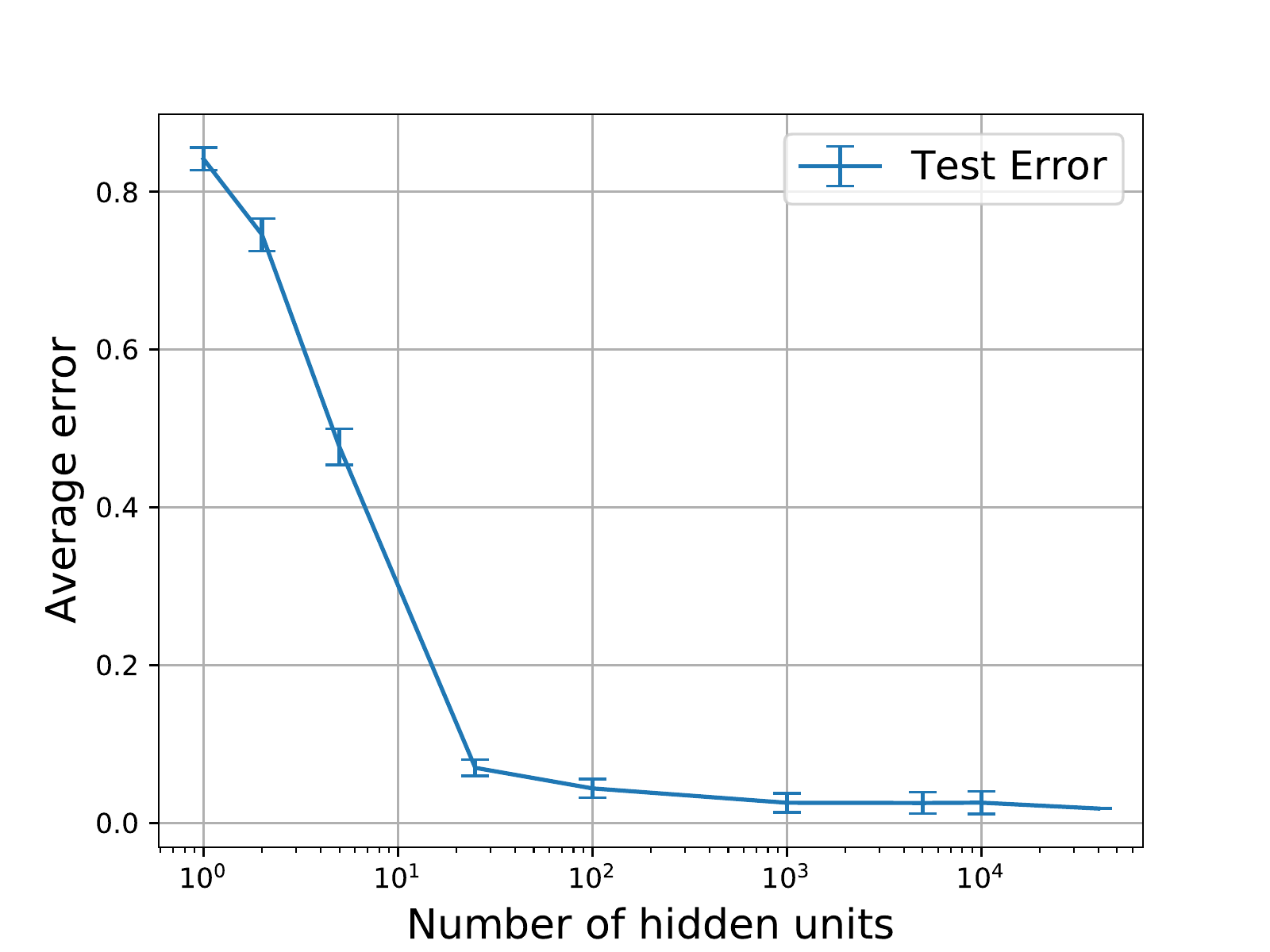}
    \end{subfigure}
    \caption{MNIST bias-variance plot from main paper (left) next to the corresponding test error (right)}
\end{figure}

\begin{figure}[H]
    \centering
    \includegraphics[width=.48\textwidth]{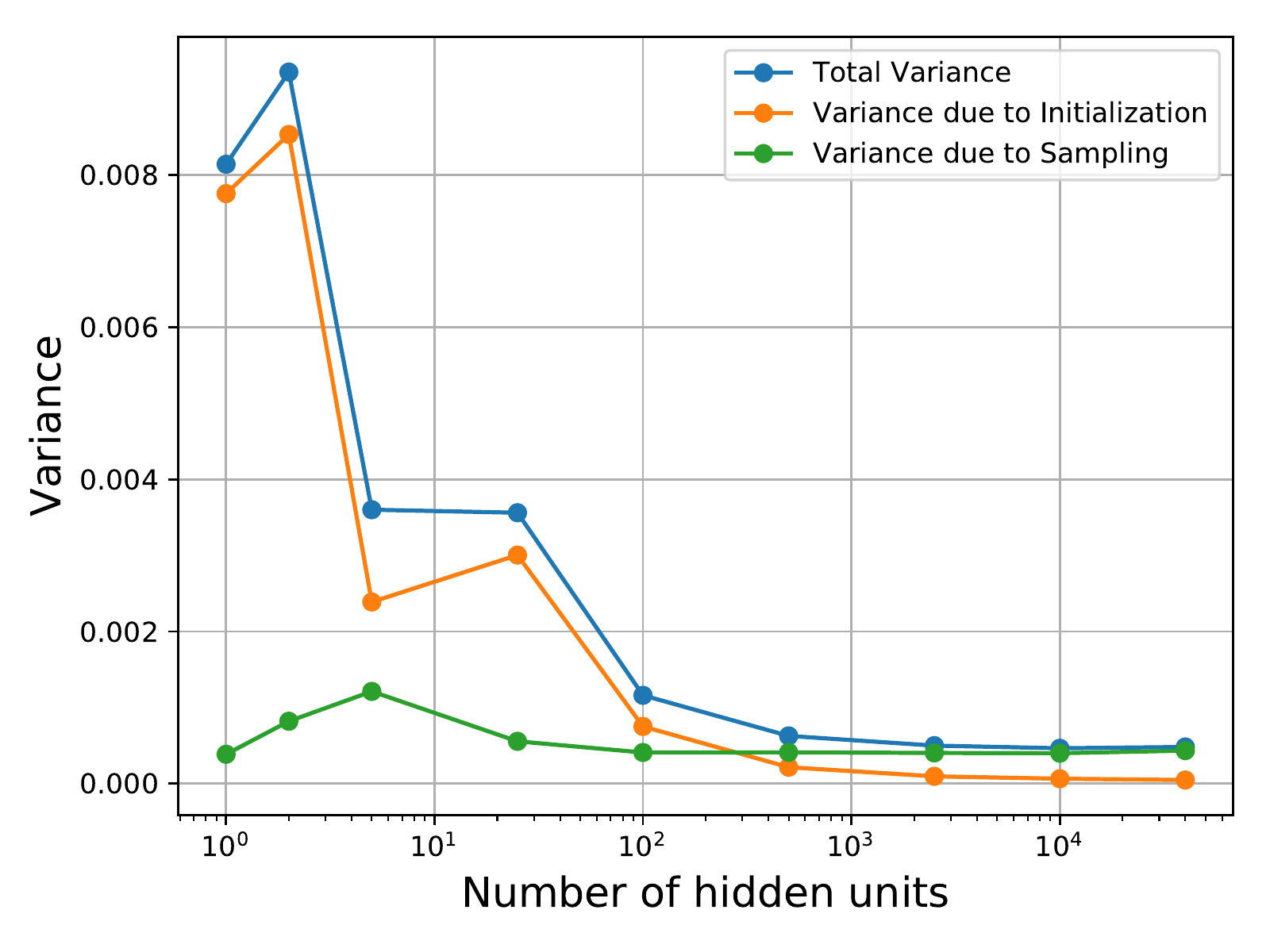}
    \caption{Decomposed variance on MNIST}
\end{figure}

\subsection{Tuned learning rates for SGD}
\label{app:tuned_lr}

\begin{figure}[H]
    \centering
    \begin{subfigure}[t]{0.48\textwidth}
        \centering
        \includegraphics[width=\textwidth]{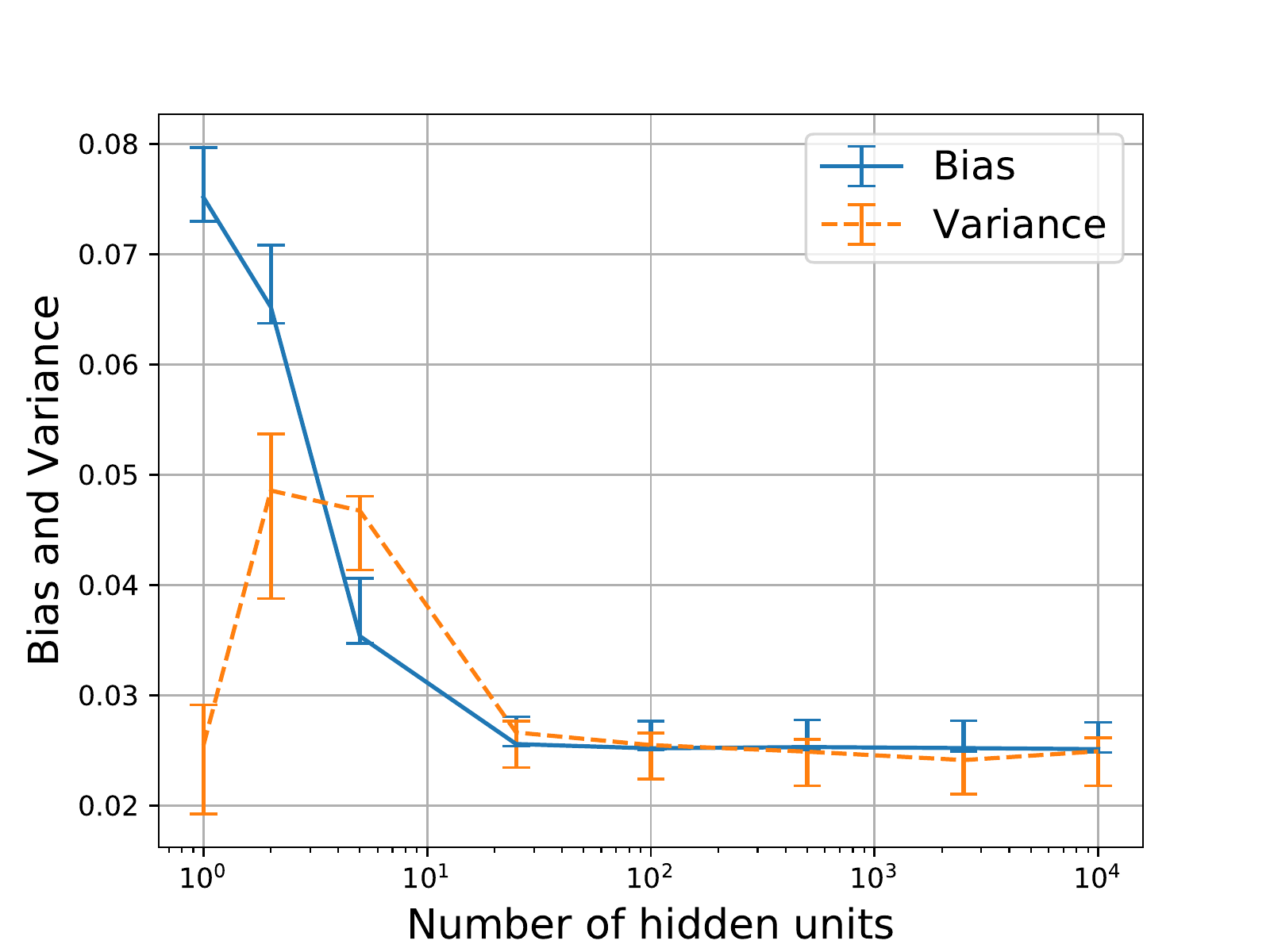}
        \caption{Variance decreases with width, even in the small data setting (SGD). This figure is in the main paper, but we include it here to compare with the corresponding step sizes used.}
    \end{subfigure}
    \hfill
    \begin{subfigure}[t]{0.48\textwidth}
        \centering
        \includegraphics[width=\textwidth]{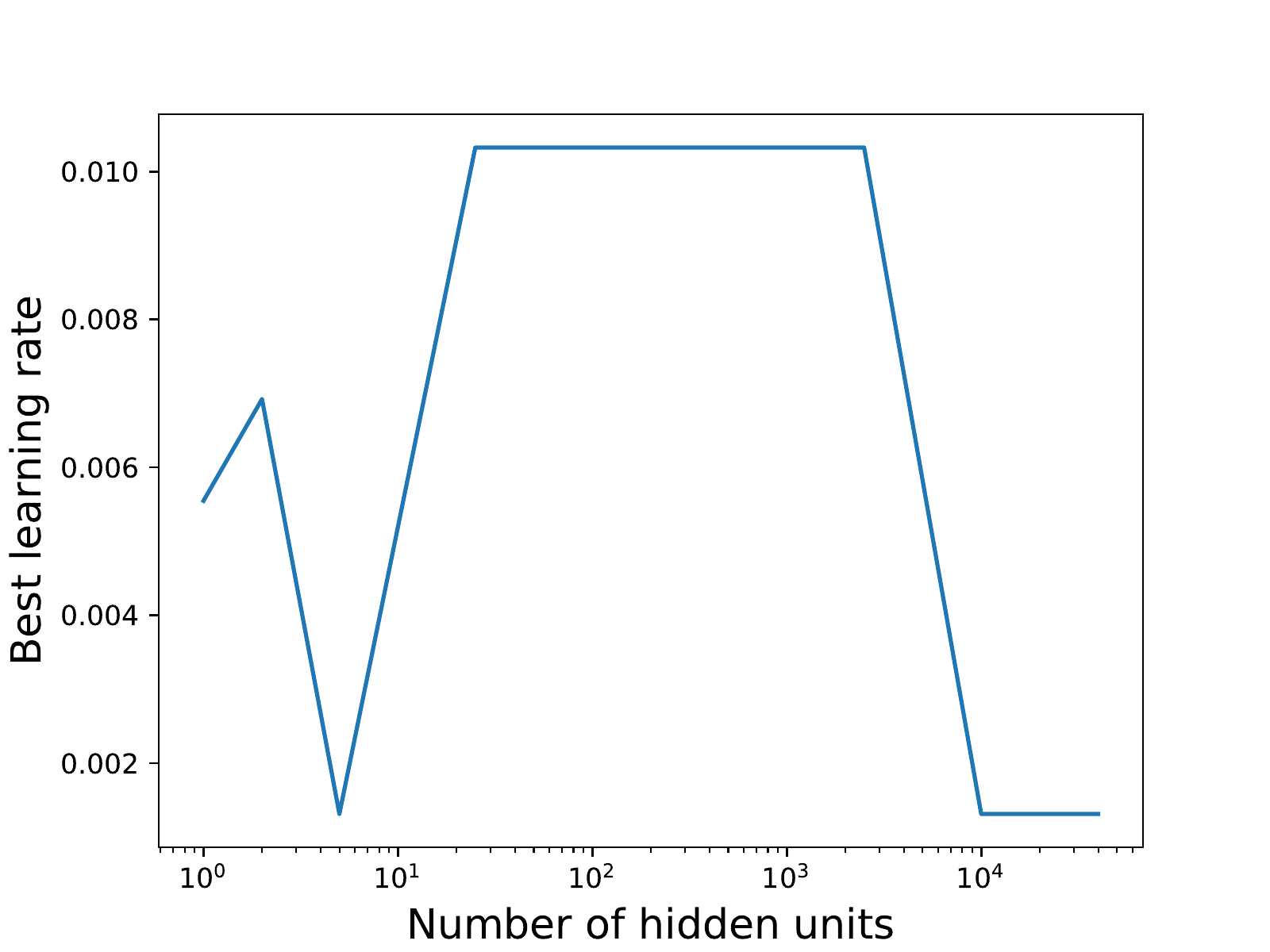}
        \caption{Corresponding optimal learning rates found, by random search, and used.}
    \end{subfigure}
\end{figure}

\subsection{Fixed learning rate results for small data MNIST}
\label{app:fixed_lr}

\begin{figure}[H]
    \centering
    \begin{subfigure}[t]{0.48\textwidth}
        \centering
        \includegraphics[width=\textwidth]{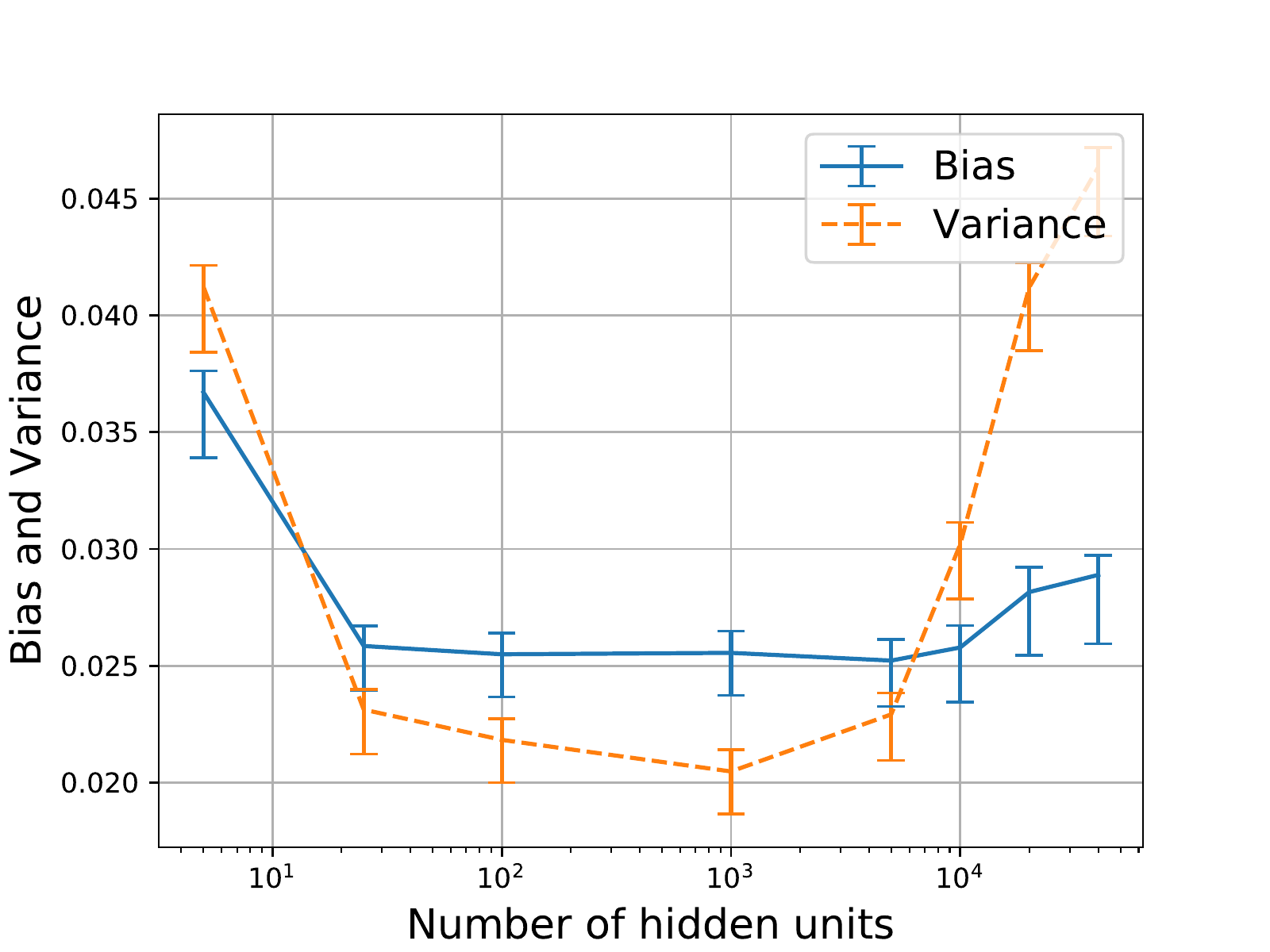}
    \end{subfigure}
    \hfill
    \begin{subfigure}[t]{0.48\textwidth}
        \centering
        \includegraphics[width=\textwidth]{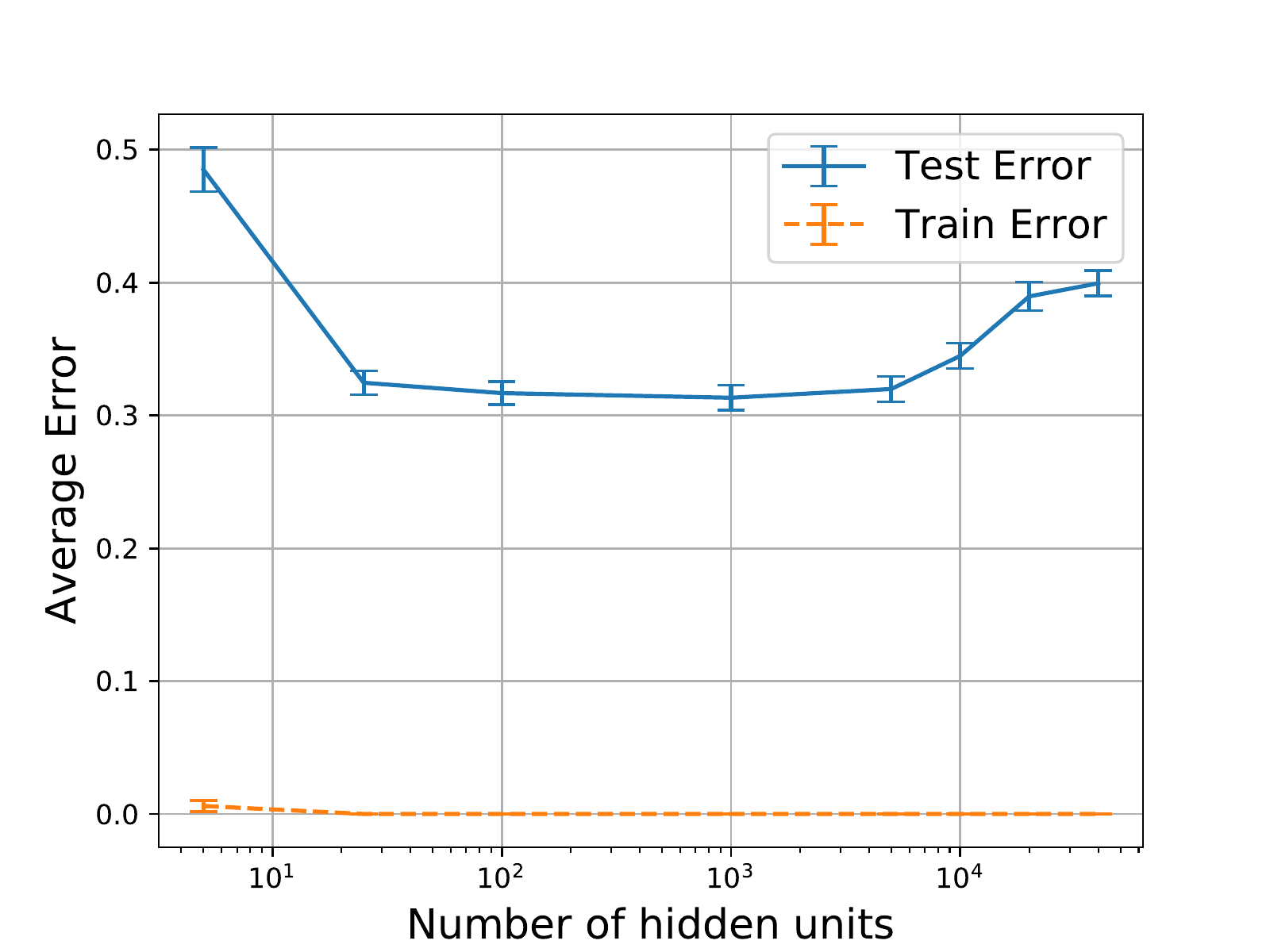}
    \end{subfigure}
    \caption{Variance on small data with a fixed learning rate of 0.01 for all networks.}
    \label{fig:fixed_step_size}
\end{figure}

Note that the U curve shown in \cref{fig:fixed_step_size} when we do not tune the step size is explained by the fact that the constant step chosen is a ``good'' step size for some networks and ``bad'' for others. Results from \citet{Keskar2017}, \citet{smith2018}, and \citet{DBLP:journals/corr/abs-1711-04623} show that a step size that corresponds well to the noise structure in SGD is important for achieving good test set accuracy. Because our networks are different sizes, their stochastic optimization process will have a different landscape and noise structure. By tuning the step size, we are making the experimental design choice to keep \emph{optimality of step size} constant across networks, rather than keeping step size constant across networks. To us, choosing this control makes much more sense than choosing to control for step size. Note that \citet{park2019} show that, as long as the network size is not too big, controlling for ``optimality of step size'' and ``keeping step size constant'' with increasing network width actually correspond to the same thing, as long as the networks are not too big, which fits well with the fact that we used the same step size across network widths for almost all of our experiments.

\subsection{Other optimizers for width experiment on small data MNIST}
\label{app:other_optimizers}

\begin{figure}[H]
    \centering
    \begin{subfigure}[t]{0.48\textwidth}
        \centering
        \includegraphics[width=\textwidth]{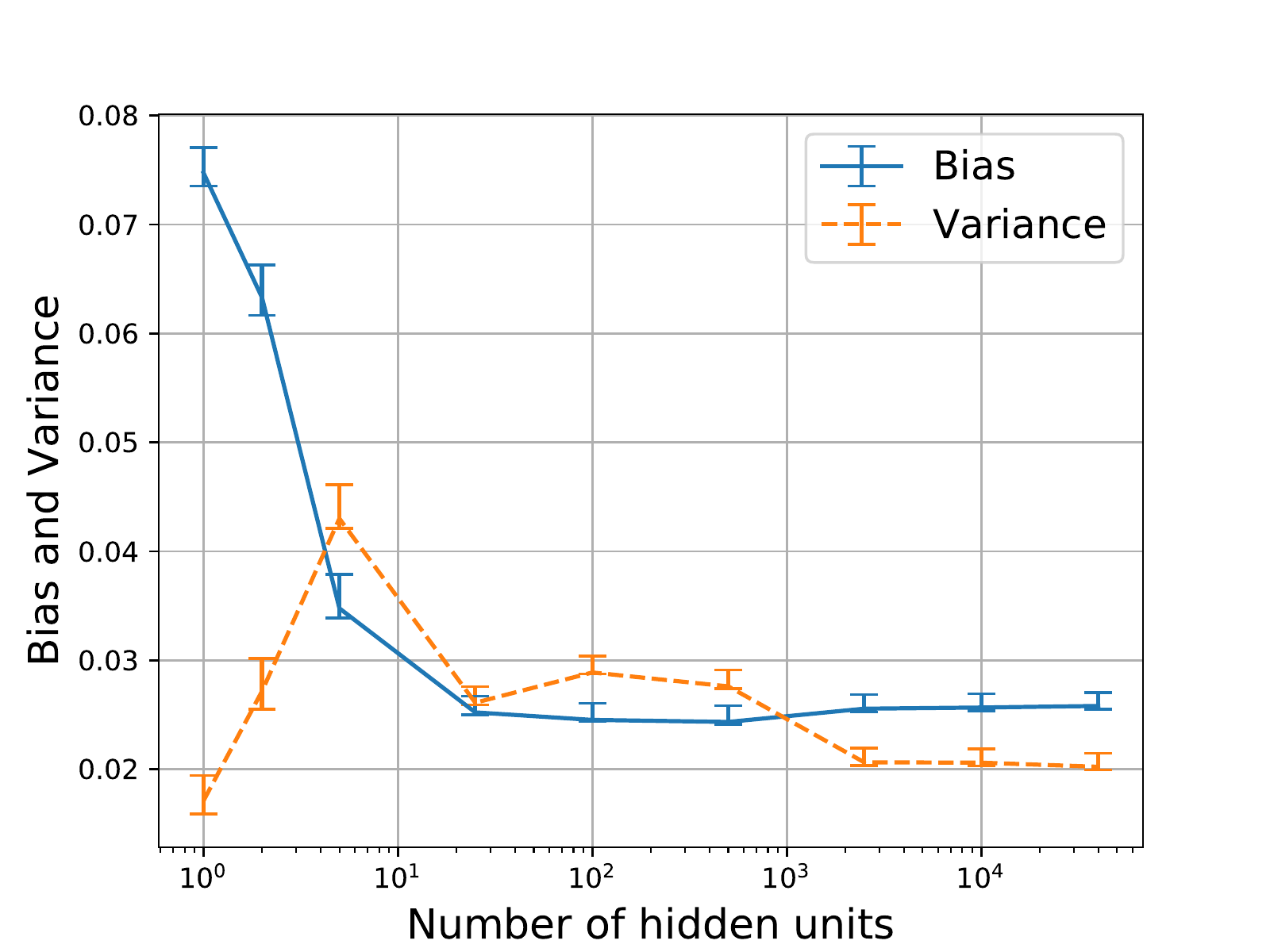}
    \end{subfigure}
    \hfill
    \begin{subfigure}[t]{0.48\textwidth}
        \centering
        \includegraphics[width=\textwidth]{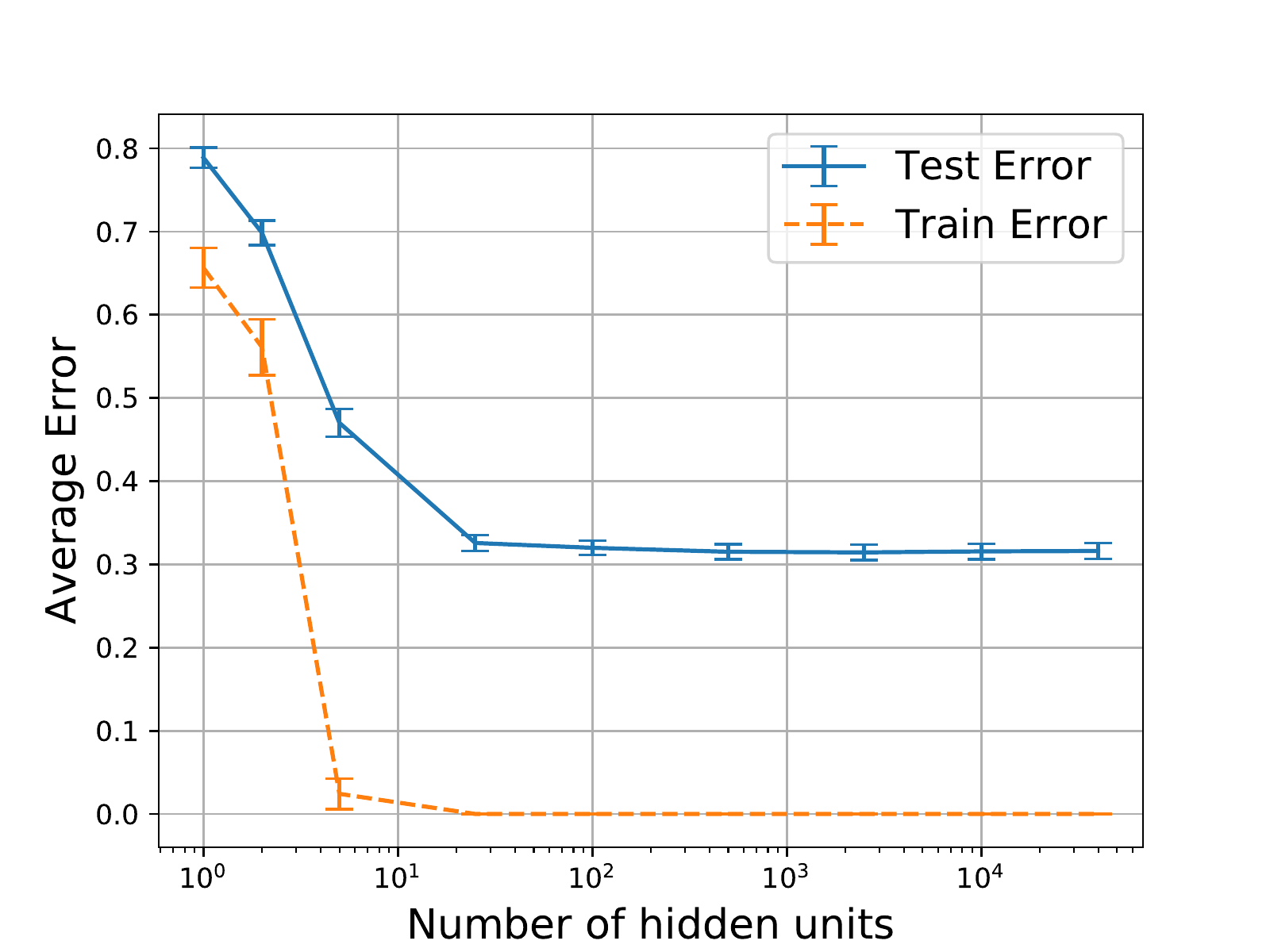}
    \end{subfigure}
    \caption{Variance decreases with width in the small data setting, even when using batch gradient descent.}
\end{figure}

\begin{figure}[H]
    \centering
    \begin{subfigure}[t]{0.48\textwidth}
        \centering
        \includegraphics[width=\textwidth]{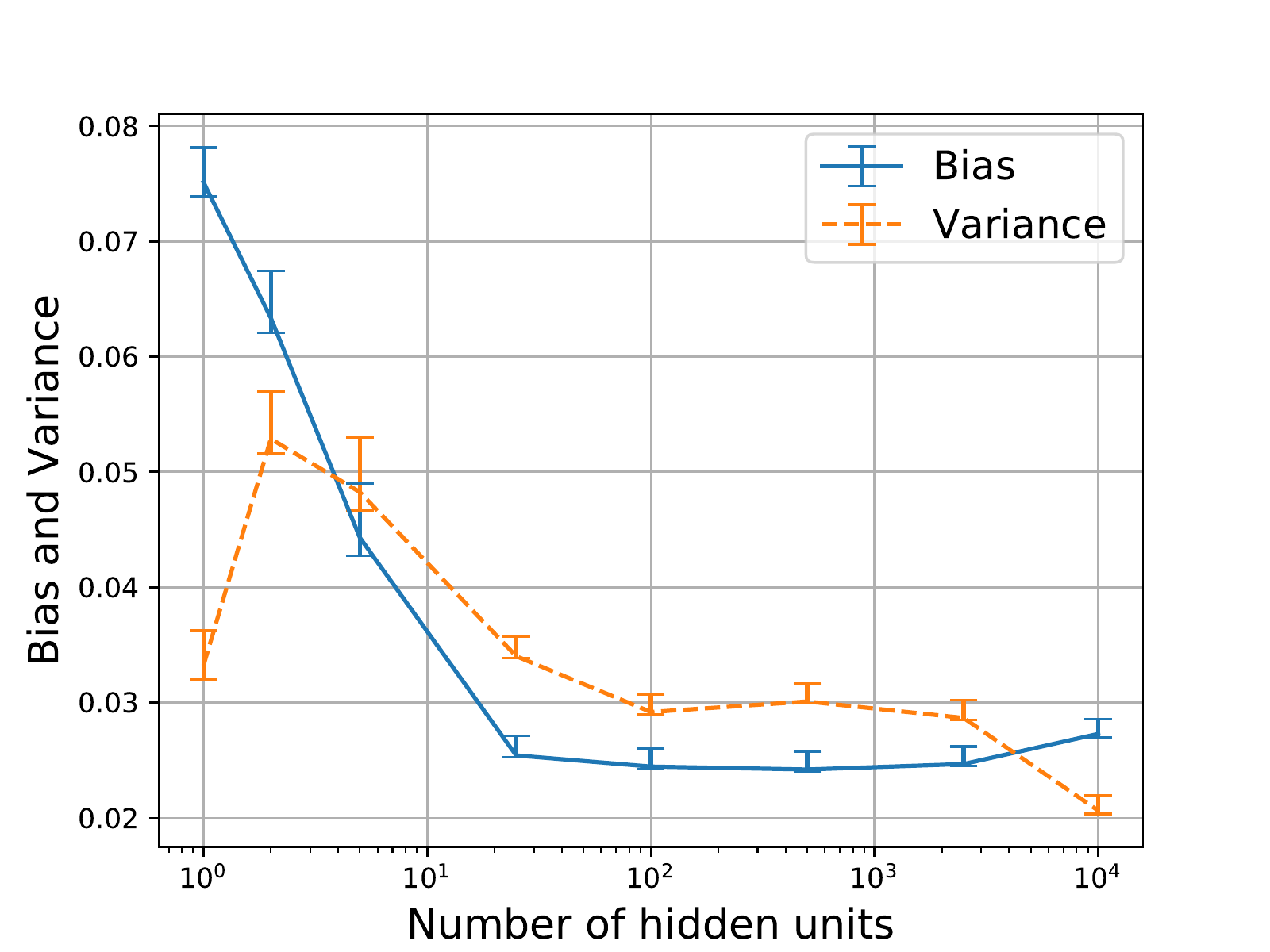}
    \end{subfigure}
    \hfill
    \begin{subfigure}[t]{0.48\textwidth}
        \centering
        \includegraphics[width=\textwidth]{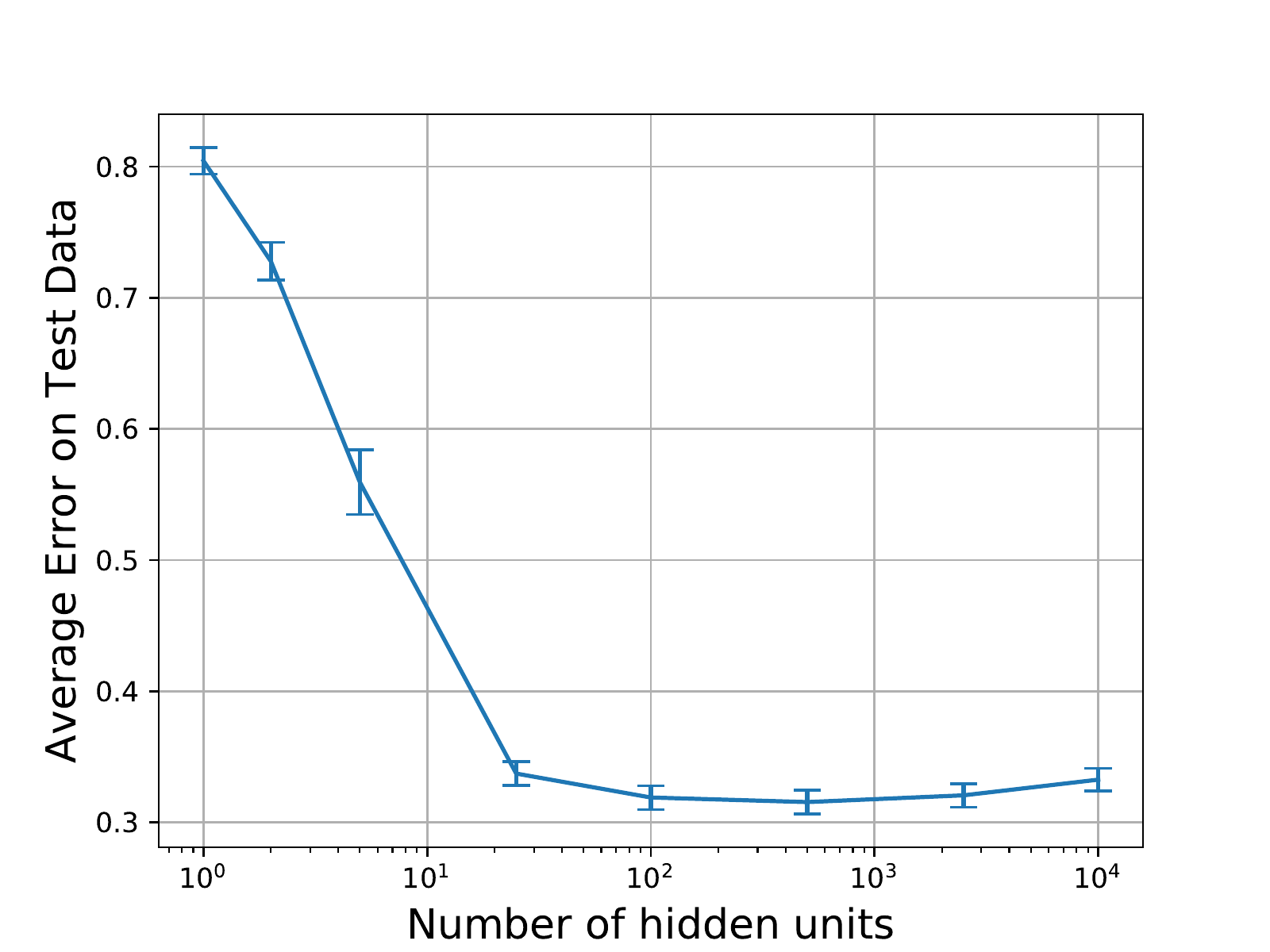}
    \end{subfigure}
    \caption{Variance decreases with width in the small data setting, even when using a strong optimizer, such as PyTorch's LBFGS, as the optimizer.}
\end{figure}

\newpage
\subsection{Sinusoid regression experiments}
\label{app:sinusoid_regression}

\begin{figure}[H]
    \centering
    \begin{subfigure}[t]{0.48\textwidth}
        \centering
        \includegraphics[width=\textwidth]{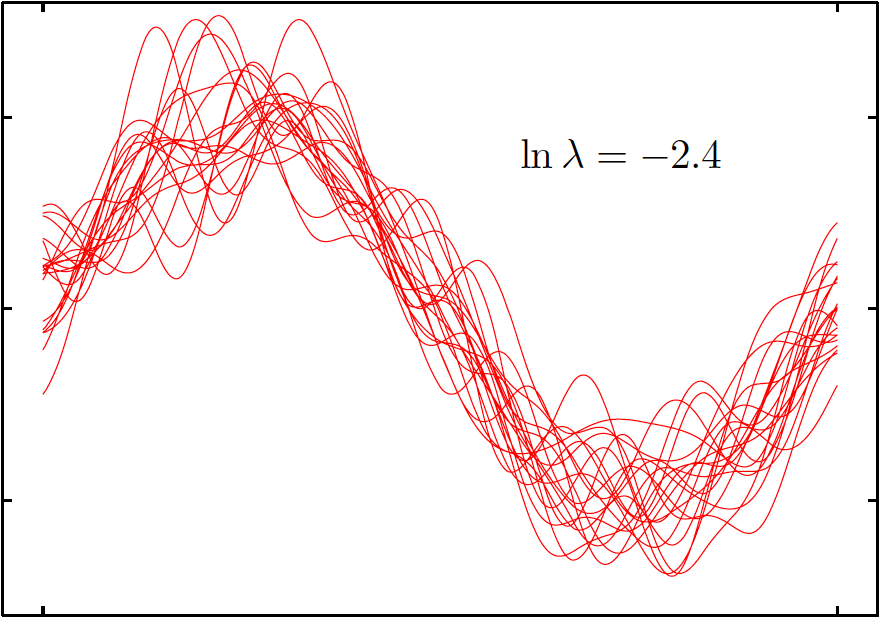}
        \caption{Example of the many different functions learned by a high variance learner \citep[Section 3.2]{Bishop:2006}}
    \end{subfigure}
    \hfill
    \begin{subfigure}[t]{0.48\textwidth}
        \centering
        \includegraphics[width=\textwidth]{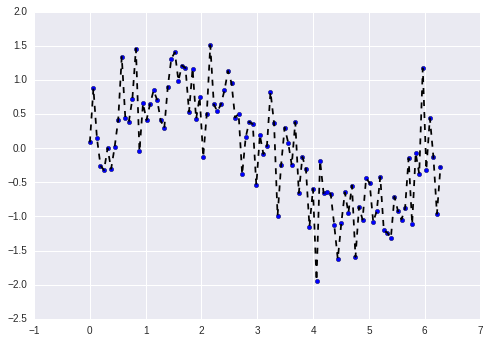}
        \caption{Caricature of a single function learned by a high variance learner \citep{wtf_is_bv}}
    \end{subfigure}
    \caption{Caricature examples of high variance learners on sinusoid task. Below, we find that this does not happen with increasingly wide neural networks (\cref{fig:app_all_learned_sinusoids} and \cref{fig:app_mean_and_var_vis}).}
    \label{fig:high_var_caricature}
\end{figure}

\begin{figure}[H]
    \centering
    \includegraphics[width=.55\textwidth]{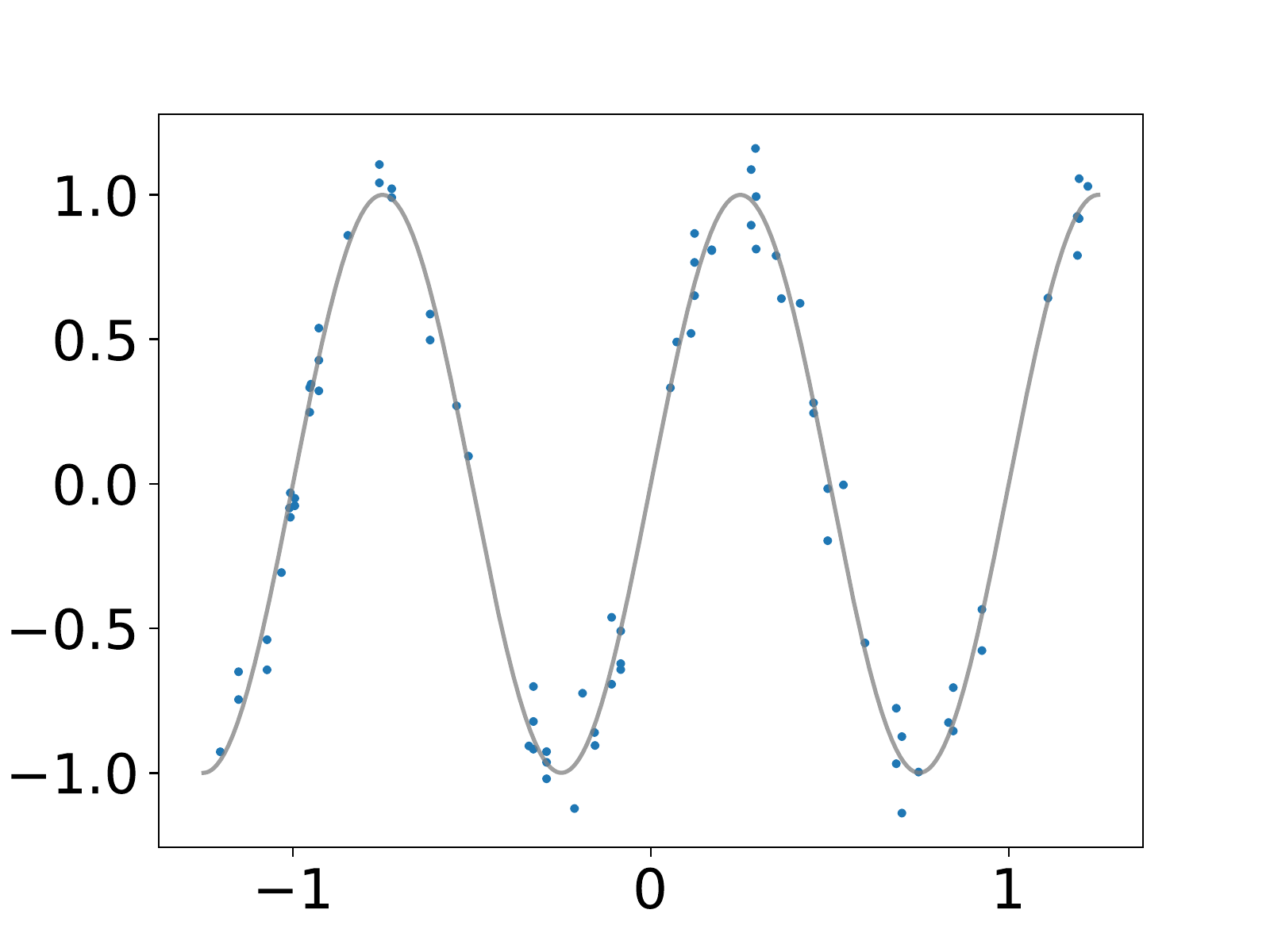}
    \caption{Target function of the noisy sinusoid regression task (in gray) and an example of a training set (80 data points) sampled from the noisy distribution.}
\end{figure}

\begin{figure}[H]
    \centering
    \begin{subfigure}[t]{0.32\textwidth}
        \centering
        \includegraphics[width=\textwidth]{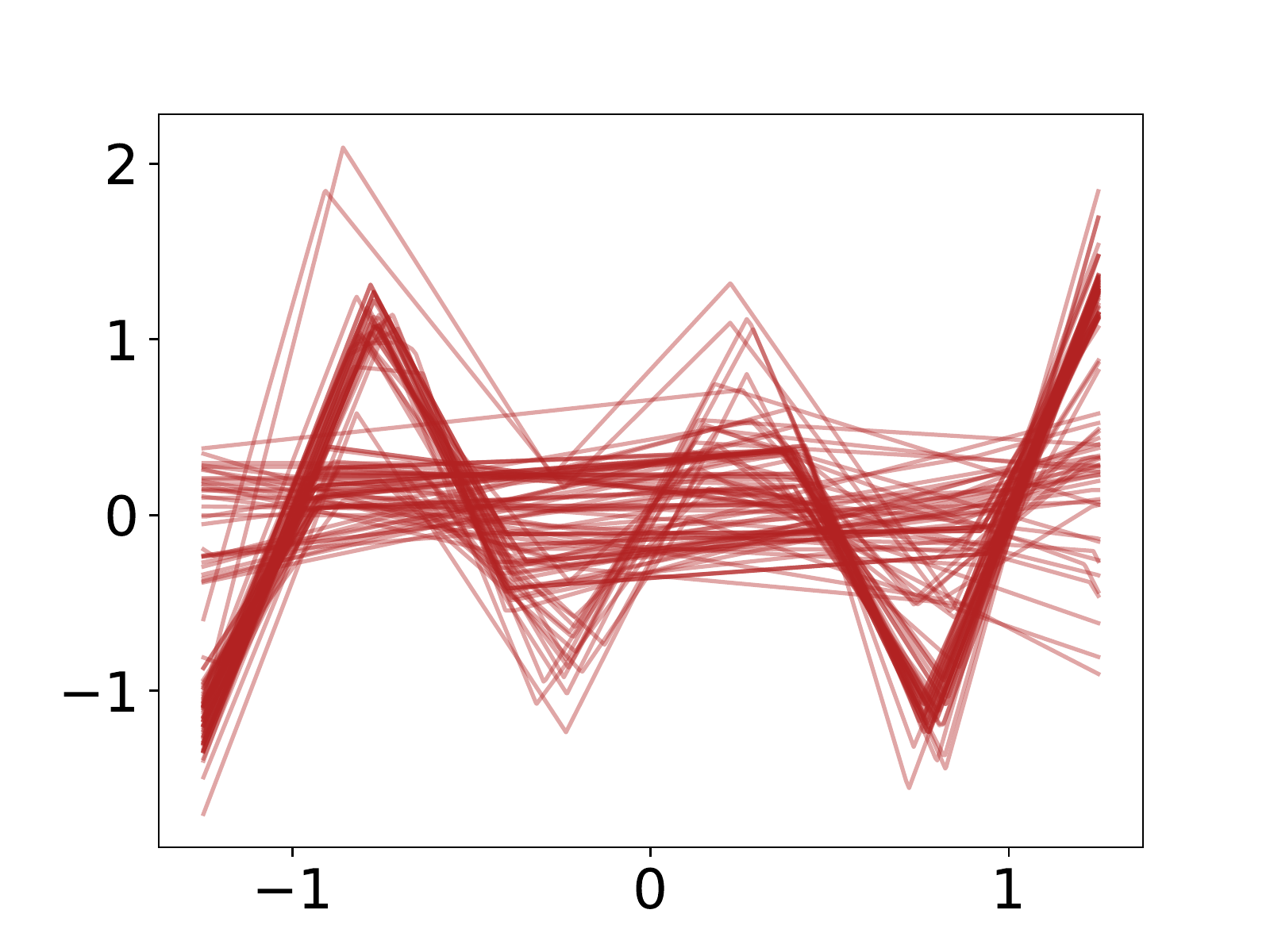}
    \end{subfigure}
    \hfill
    \begin{subfigure}[t]{0.32\textwidth}
        \centering
        \includegraphics[width=\textwidth]{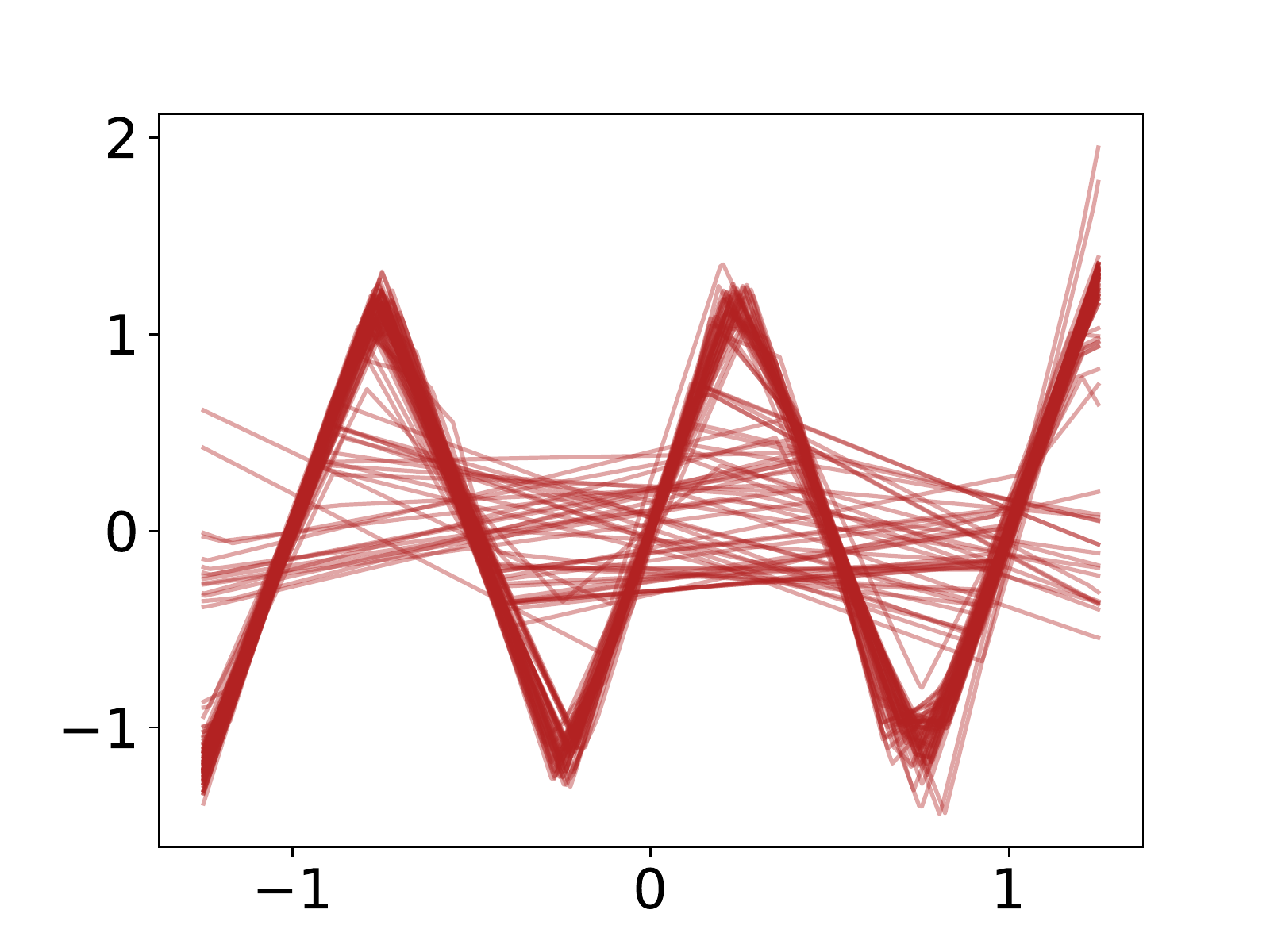}
    \end{subfigure}
    \hfill
    \begin{subfigure}[t]{0.32\textwidth}
        \centering
        \includegraphics[width=\textwidth]{figures/sinusoids/all_functions_vis/width15}
    \end{subfigure}
    
    \begin{subfigure}[t]{0.32\textwidth}
        \centering
        \includegraphics[width=\textwidth]{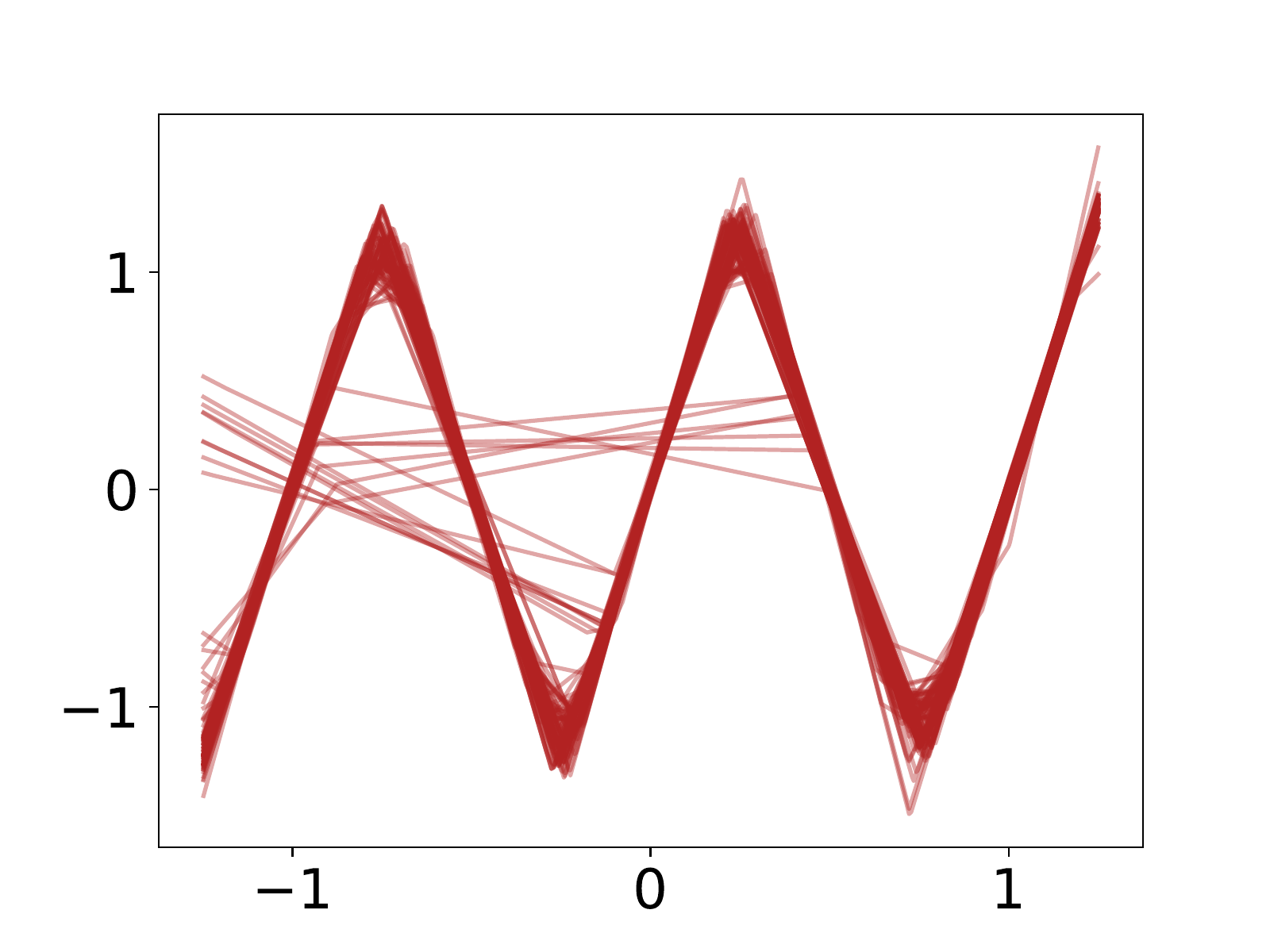}
    \end{subfigure}
    \hfill
    \begin{subfigure}[t]{0.32\textwidth}
        \centering
        \includegraphics[width=\textwidth]{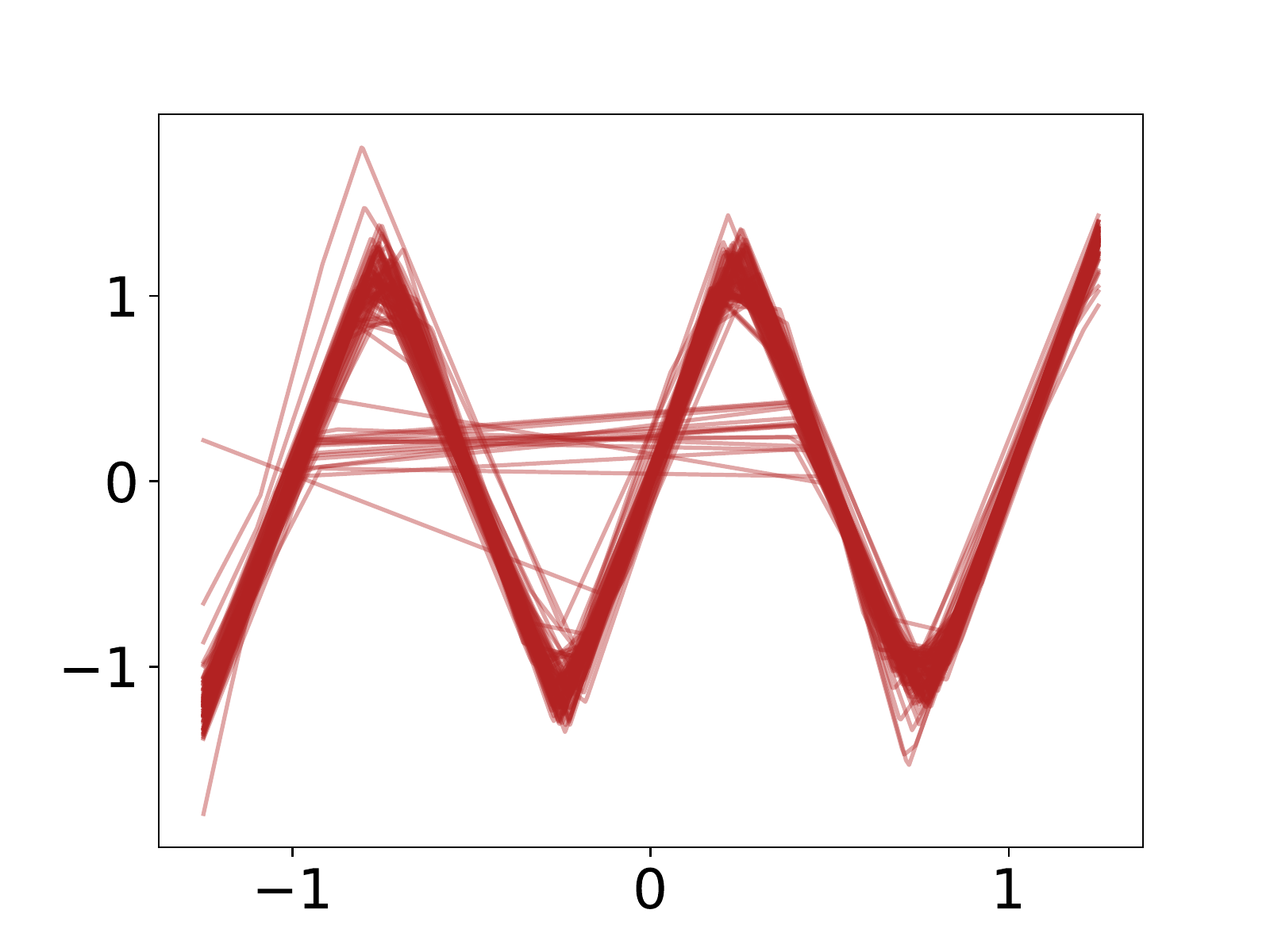}
    \end{subfigure}
    \hfill
    \begin{subfigure}[t]{0.32\textwidth}
        \centering
        \includegraphics[width=\textwidth]{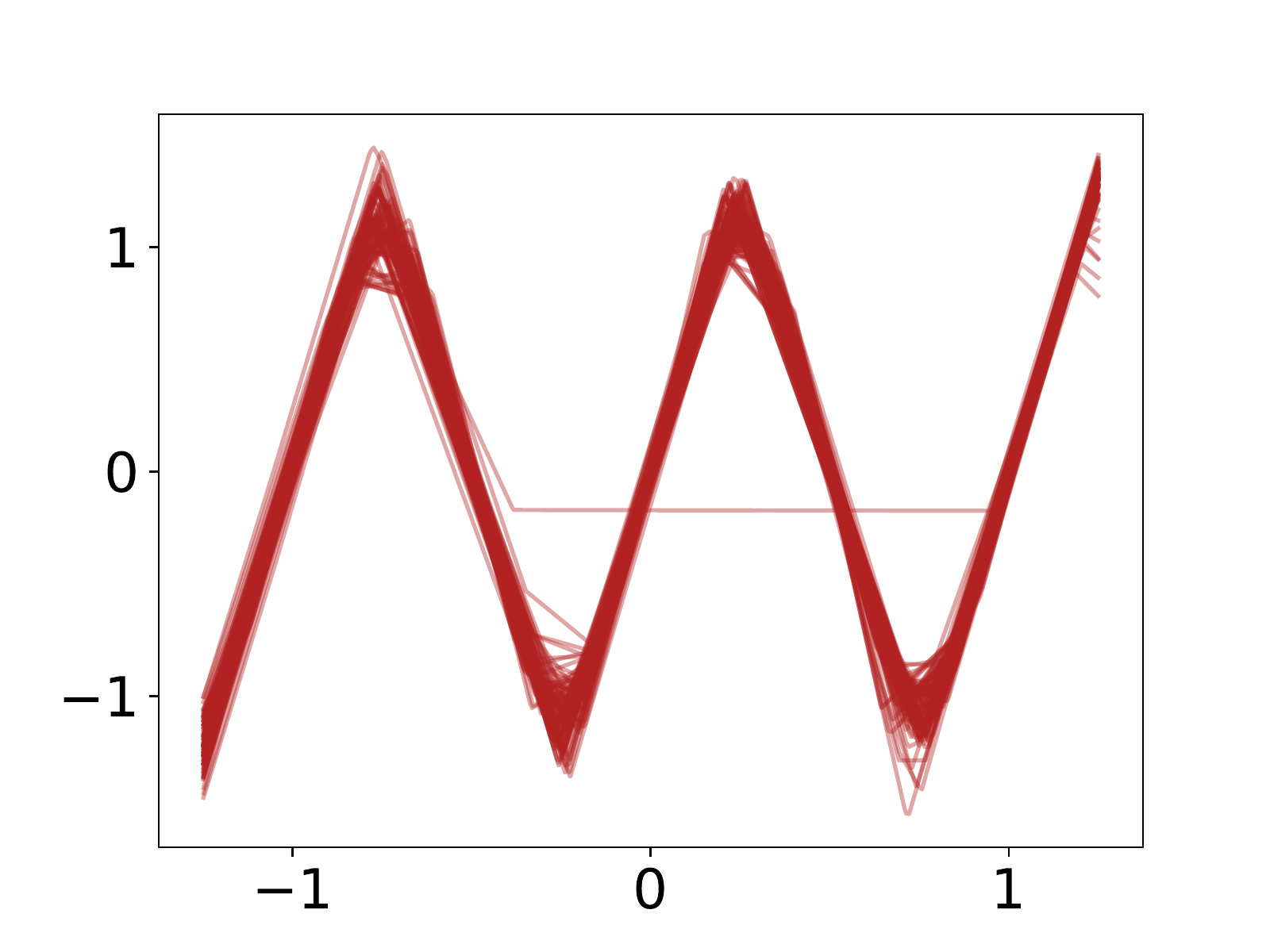}
    \end{subfigure}
    
    \begin{subfigure}[t]{0.32\textwidth}
        \centering
        \includegraphics[width=\textwidth]{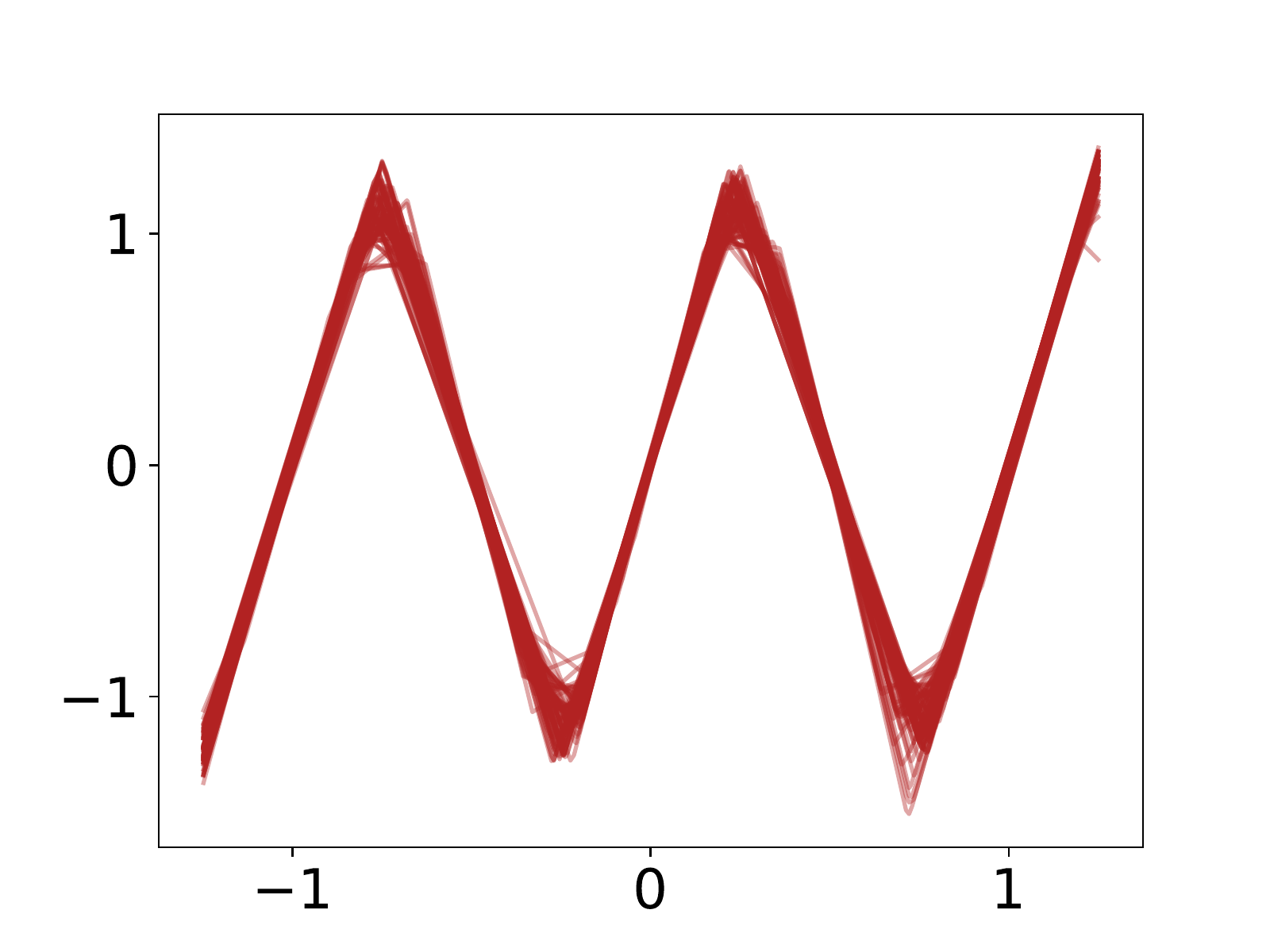}
    \end{subfigure}
    \hfill
    \begin{subfigure}[t]{0.32\textwidth}
        \centering
        \includegraphics[width=\textwidth]{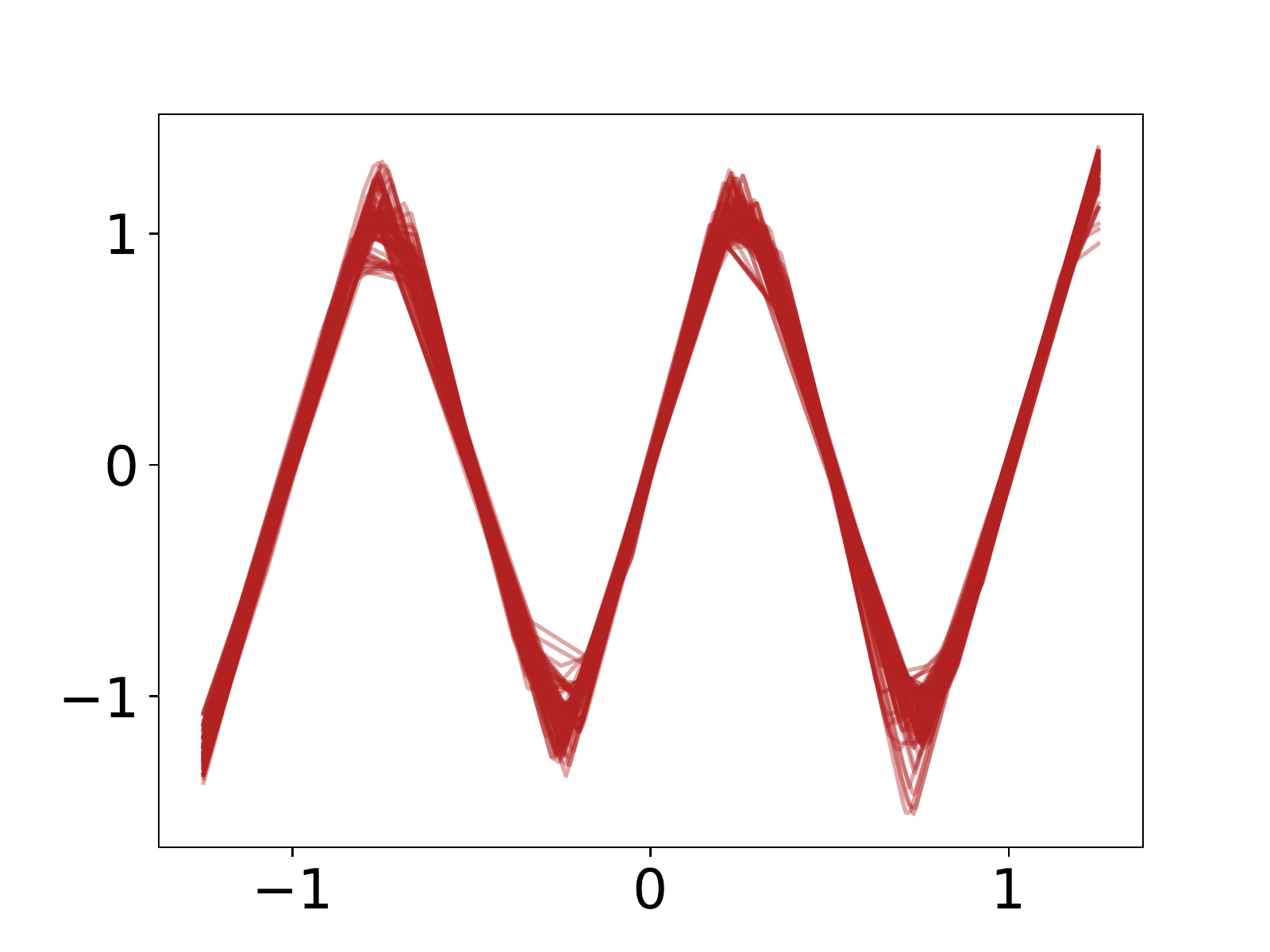}
    \end{subfigure}
    \hfill
    \begin{subfigure}[t]{0.32\textwidth}
        \centering
        \includegraphics[width=\textwidth]{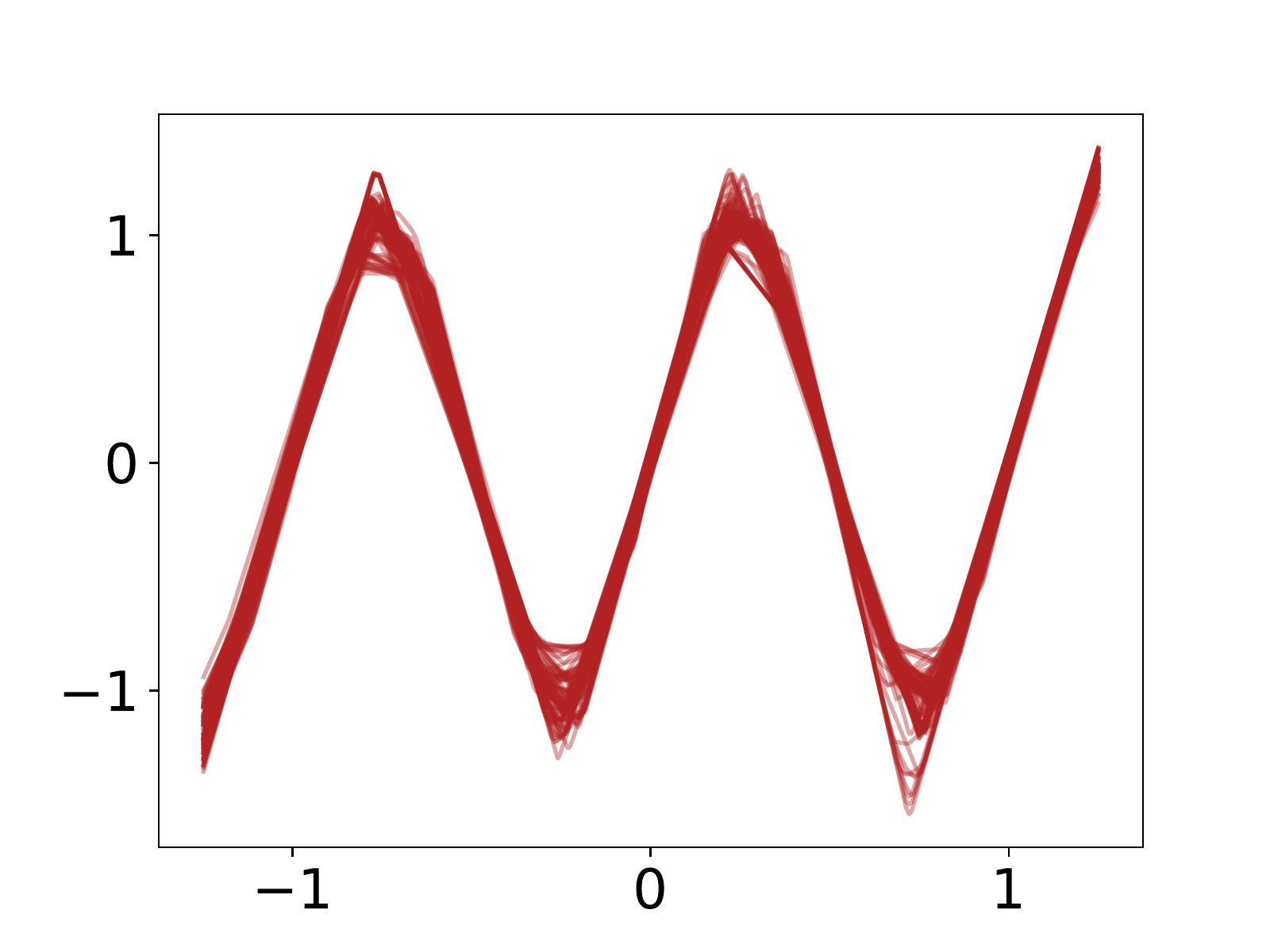}
    \end{subfigure}
    
    \begin{subfigure}[t]{0.32\textwidth}
        \centering
        \includegraphics[width=\textwidth]{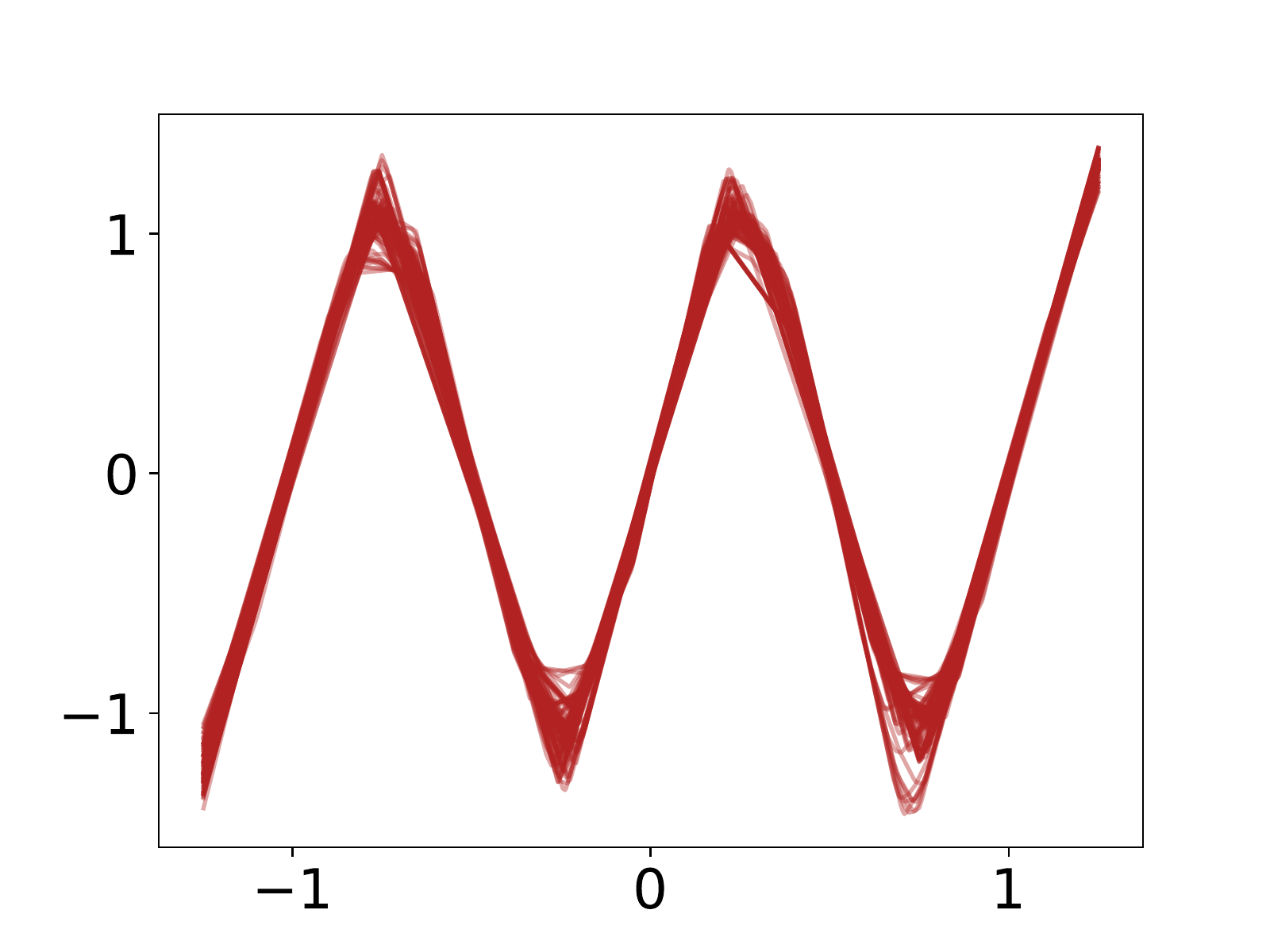}
    \end{subfigure}
    \hfill
    \begin{subfigure}[t]{0.32\textwidth}
        \centering
        \includegraphics[width=\textwidth]{figures/sinusoids/all_functions_vis/width1000}
    \end{subfigure}
    \hfill
    \begin{subfigure}[t]{0.32\textwidth}
        \centering
        \includegraphics[width=\textwidth]{figures/sinusoids/all_functions_vis/width10000}
    \end{subfigure}
    
    \caption{Visualization of 100 different functions learned by the different width neural networks. Darker color indicates higher density of different functions. Widths in increasing order from left to right and top to bottom: 5, 10, 15, 17, 20, 22, 25, 35, 75, 100, 1000, 10000. We do \textit{not} observe the caricature from \cref{fig:high_var_caricature} as width is increased.}
    \label{fig:app_all_learned_sinusoids}
\end{figure}

\begin{figure}[H]
    \centering
    \begin{subfigure}[t]{0.32\textwidth}
        \centering
        \includegraphics[width=\textwidth]{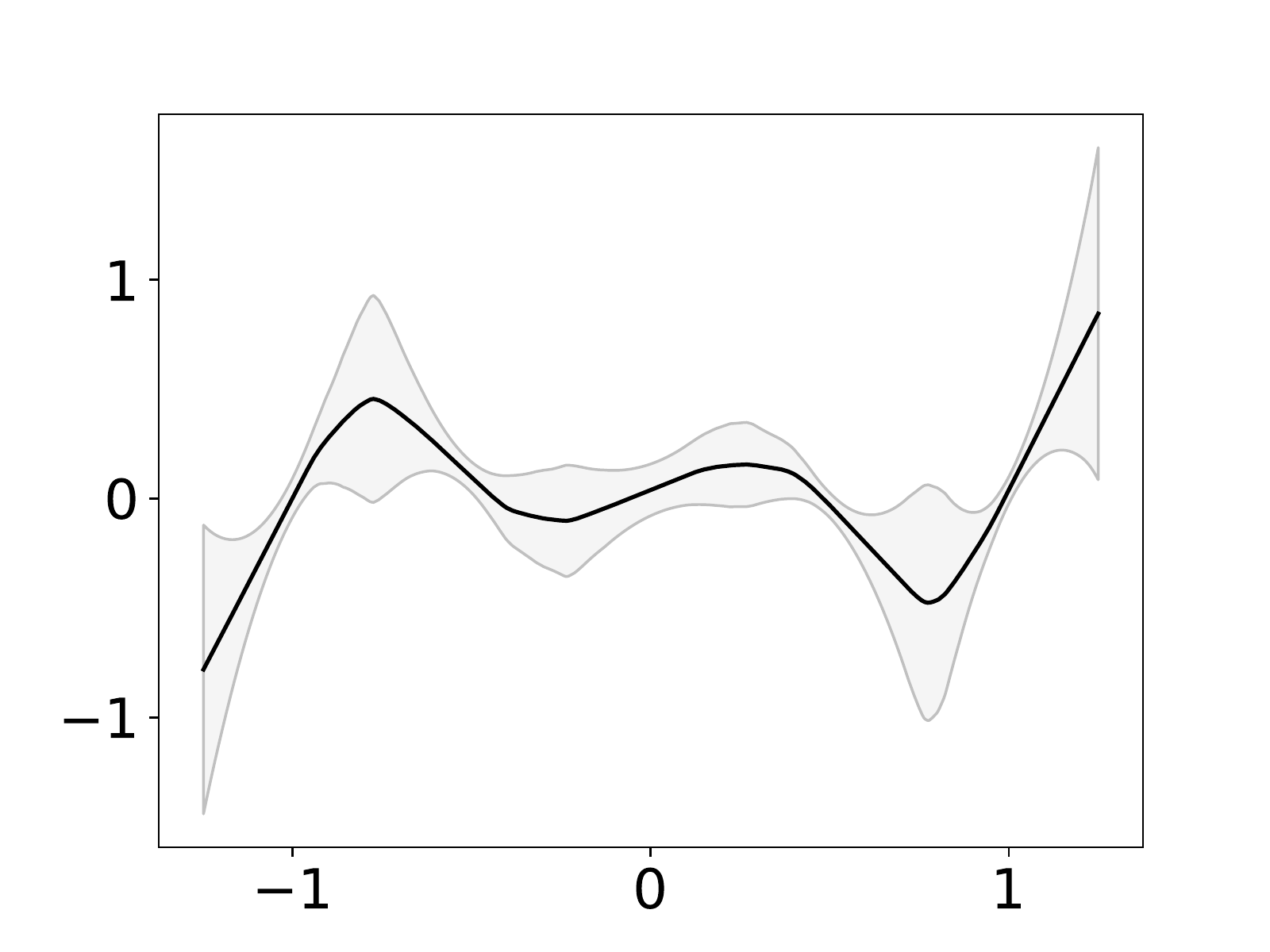}
    \end{subfigure}
    \hfill
    \begin{subfigure}[t]{0.32\textwidth}
        \centering
        \includegraphics[width=\textwidth]{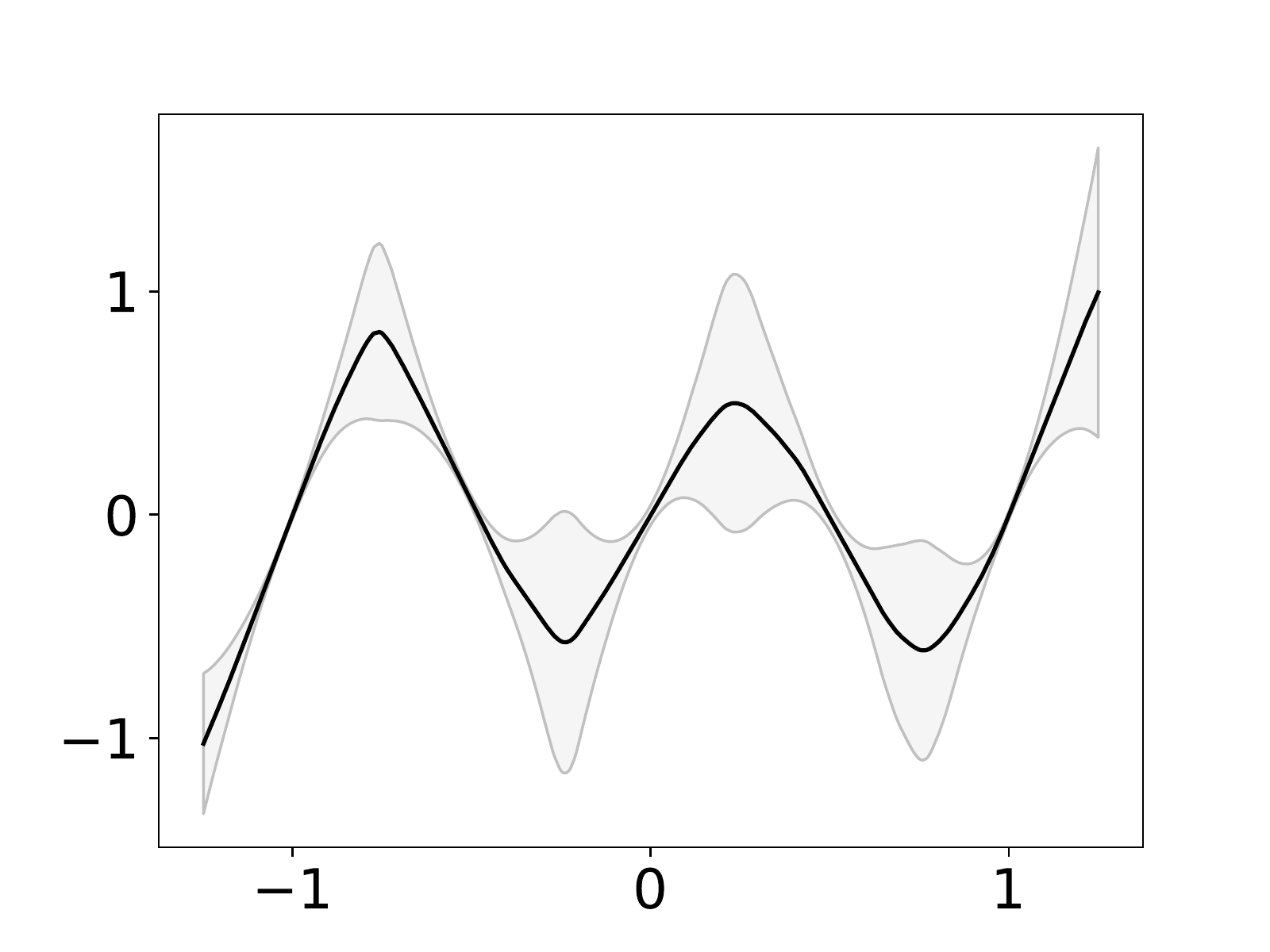}
    \end{subfigure}
    \hfill
    \begin{subfigure}[t]{0.32\textwidth}
        \centering
        \includegraphics[width=\textwidth]{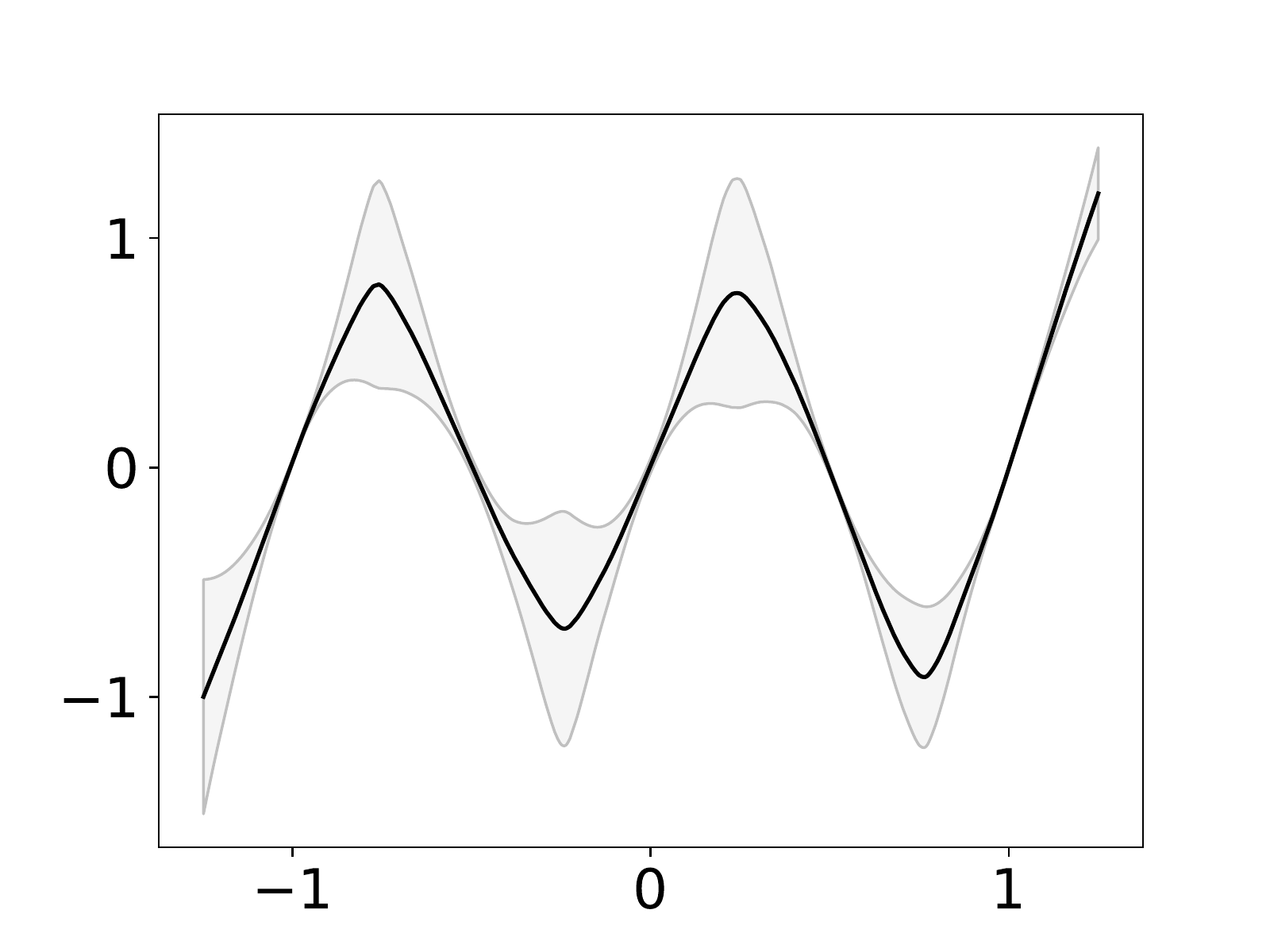}
    \end{subfigure}
    
    \begin{subfigure}[t]{0.32\textwidth}
        \centering
        \includegraphics[width=\textwidth]{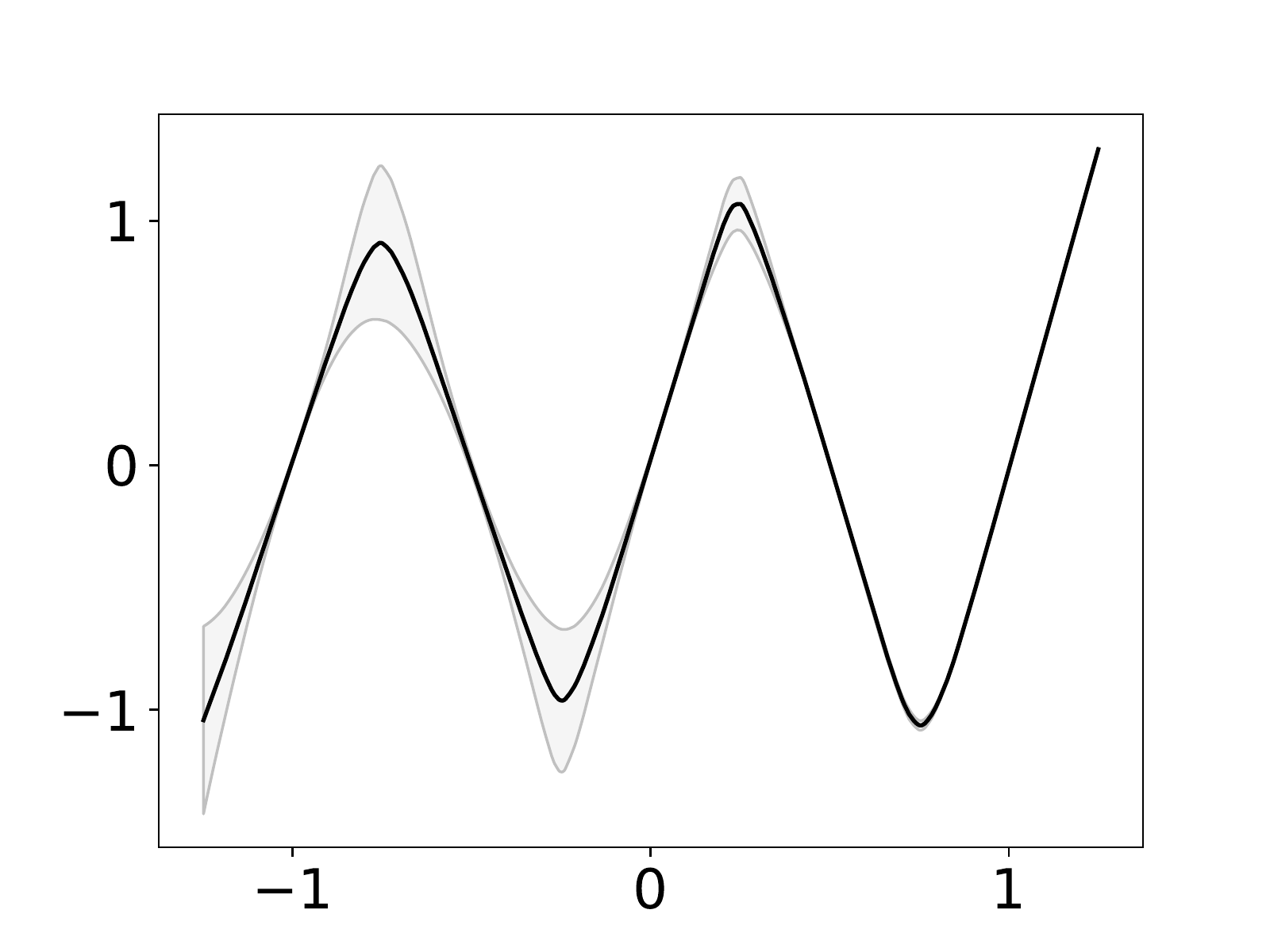}
    \end{subfigure}
    \hfill
    \begin{subfigure}[t]{0.32\textwidth}
        \centering
        \includegraphics[width=\textwidth]{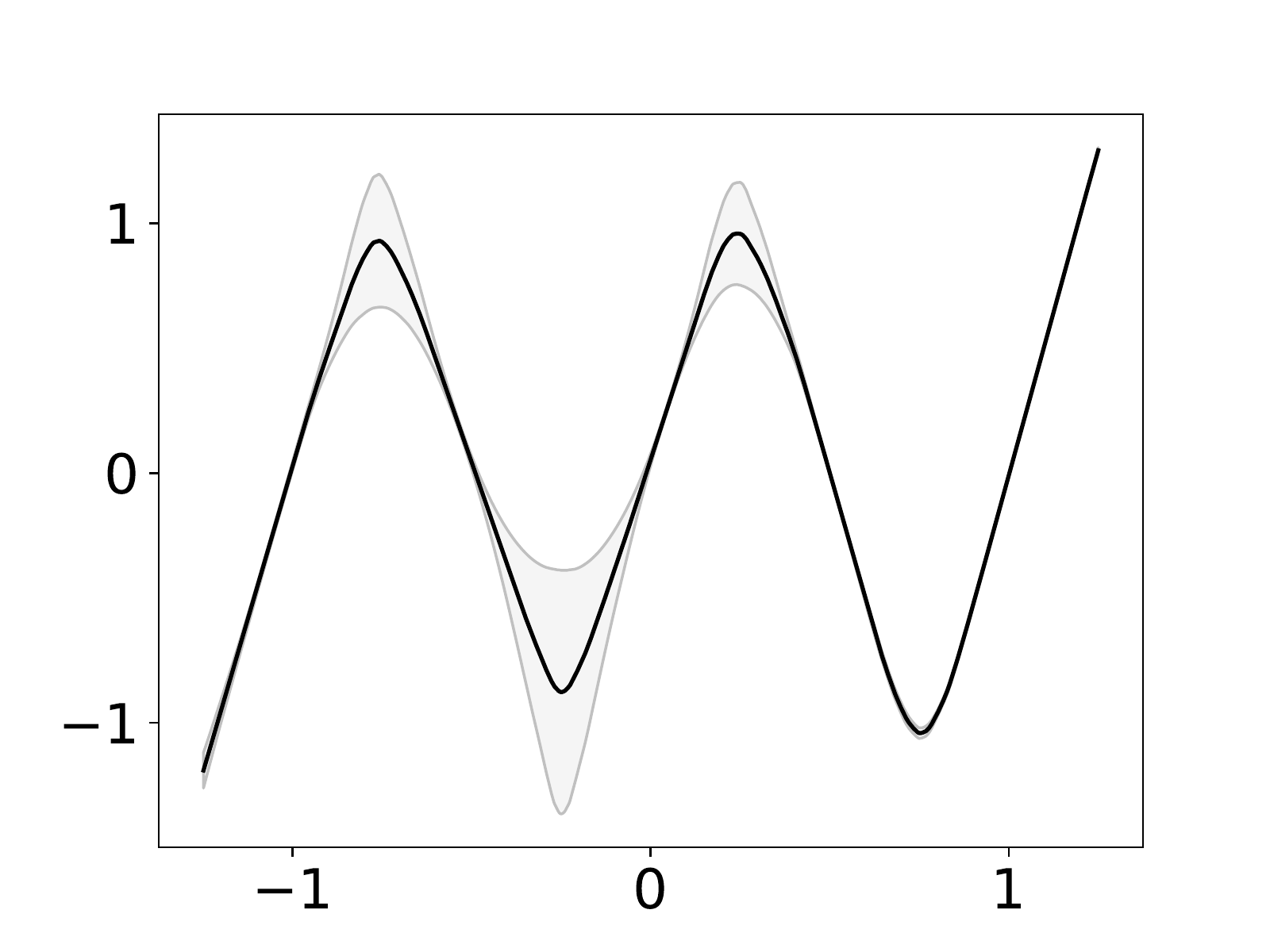}
    \end{subfigure}
    \hfill
    \begin{subfigure}[t]{0.32\textwidth}
        \centering
        \includegraphics[width=\textwidth]{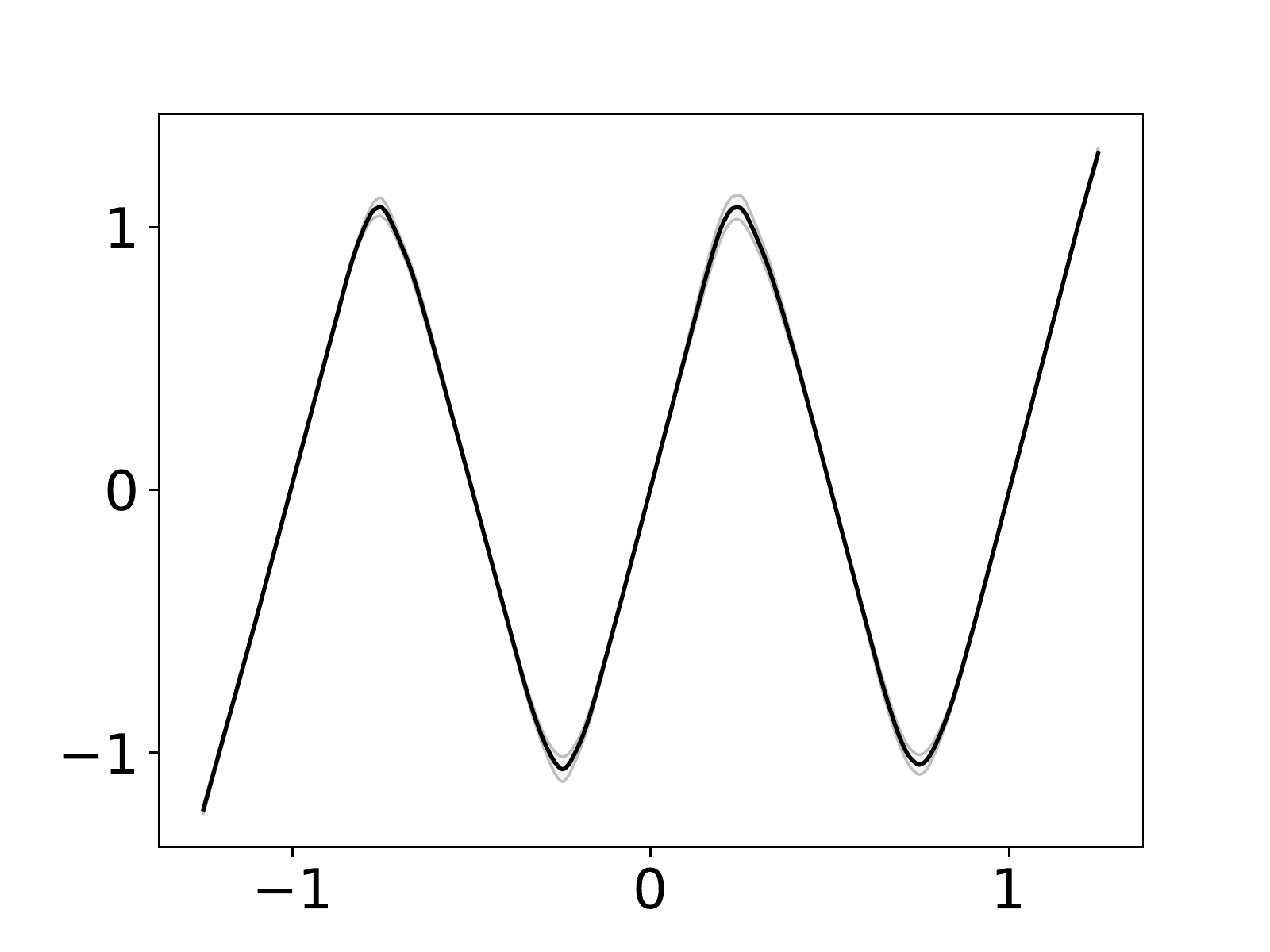}
    \end{subfigure}
    
    \begin{subfigure}[t]{0.32\textwidth}
        \centering
        \includegraphics[width=\textwidth]{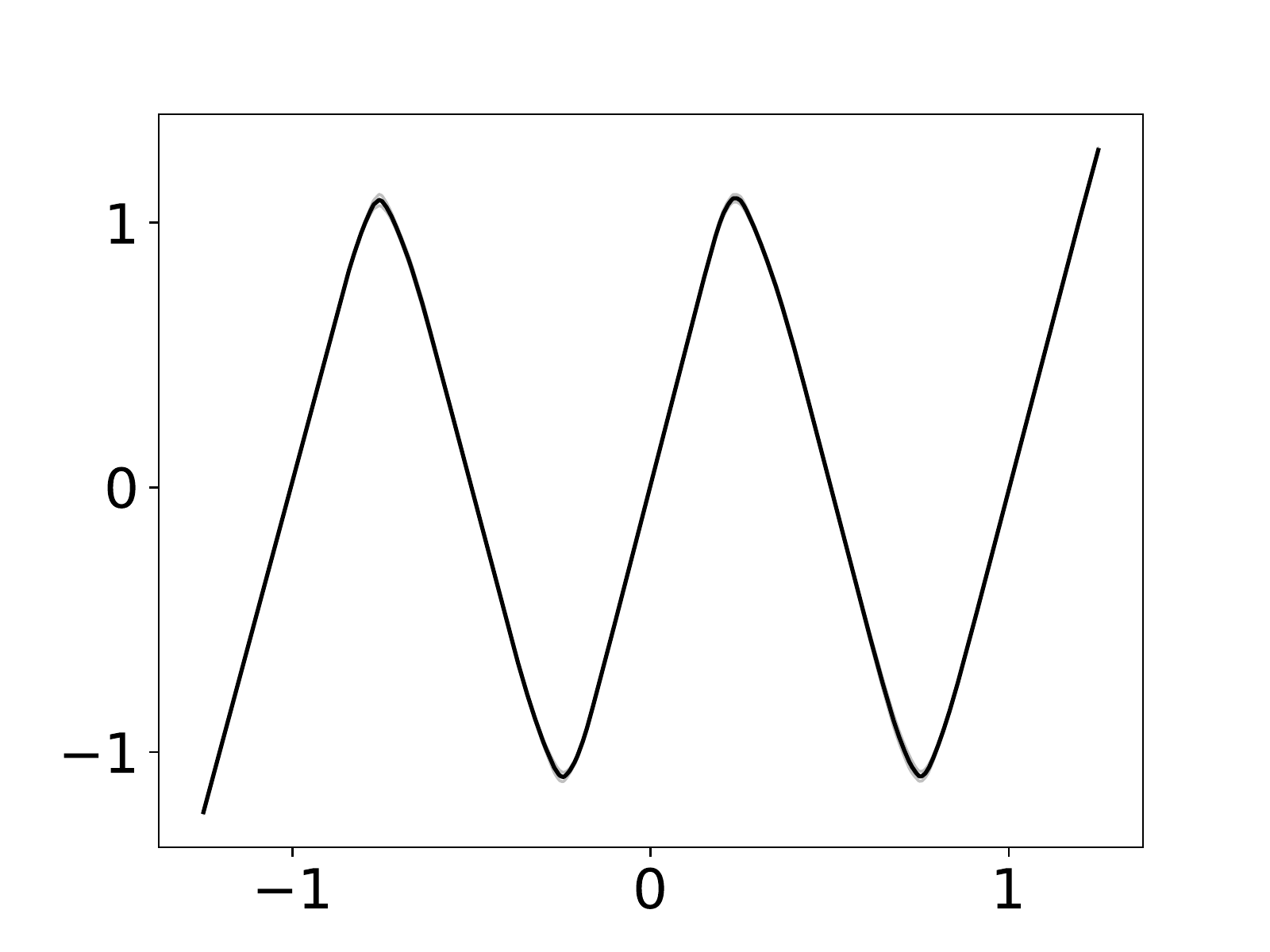}
    \end{subfigure}
    \hfill
    \begin{subfigure}[t]{0.32\textwidth}
        \centering
        \includegraphics[width=\textwidth]{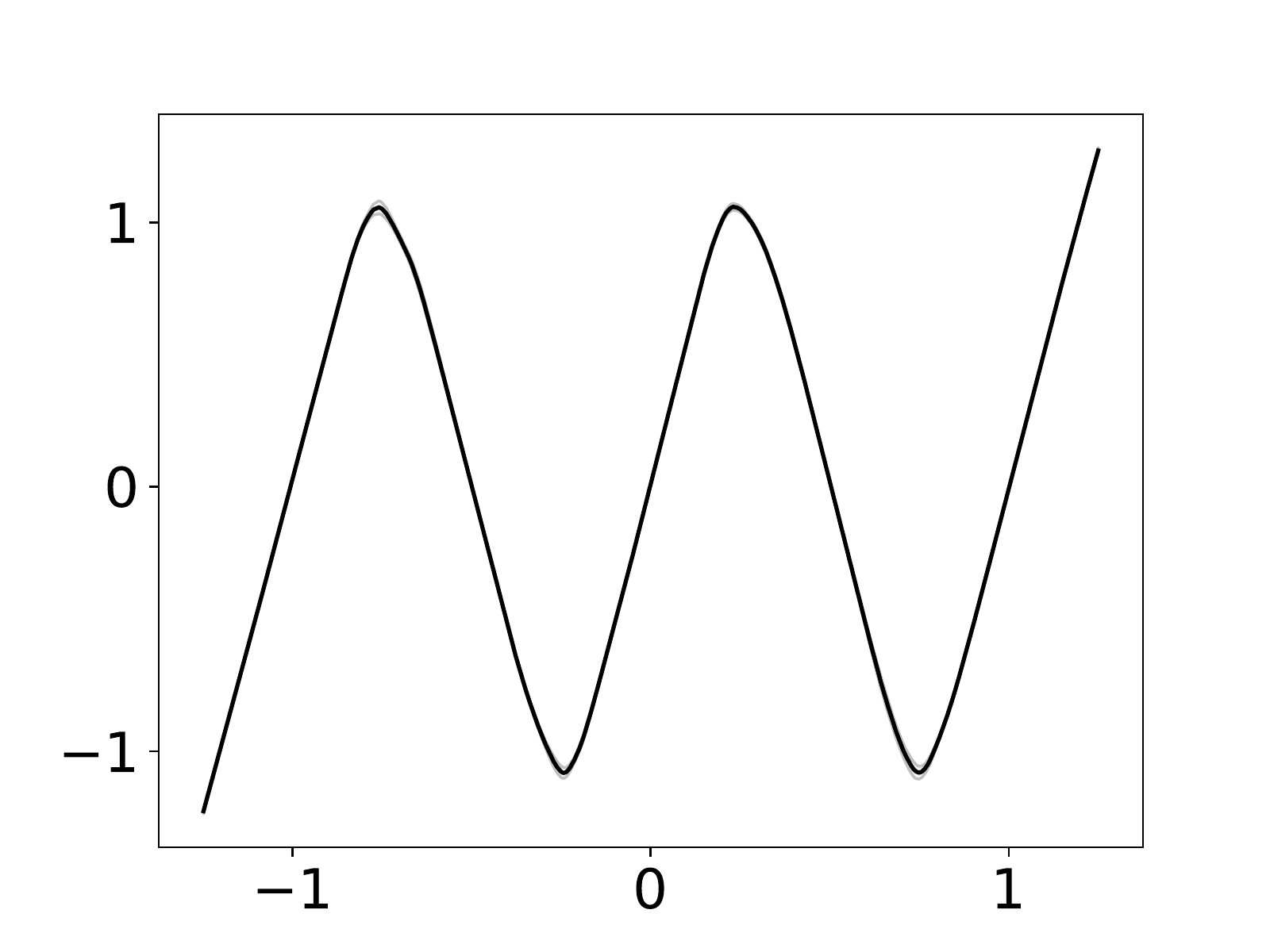}
    \end{subfigure}
    \hfill
    \begin{subfigure}[t]{0.32\textwidth}
        \centering
        \includegraphics[width=\textwidth]{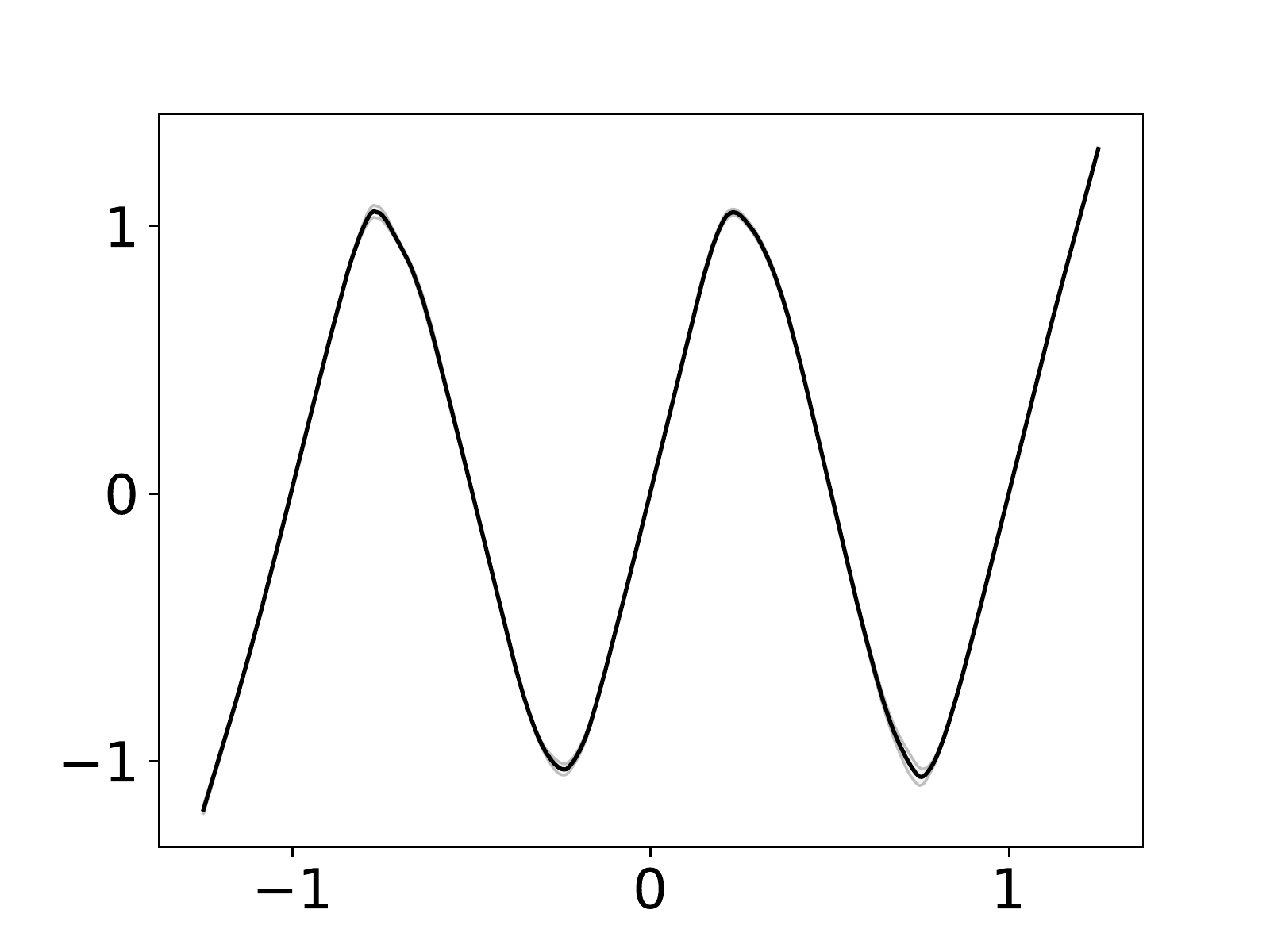}
    \end{subfigure}
    
    \begin{subfigure}[t]{0.32\textwidth}
        \centering
        \includegraphics[width=\textwidth]{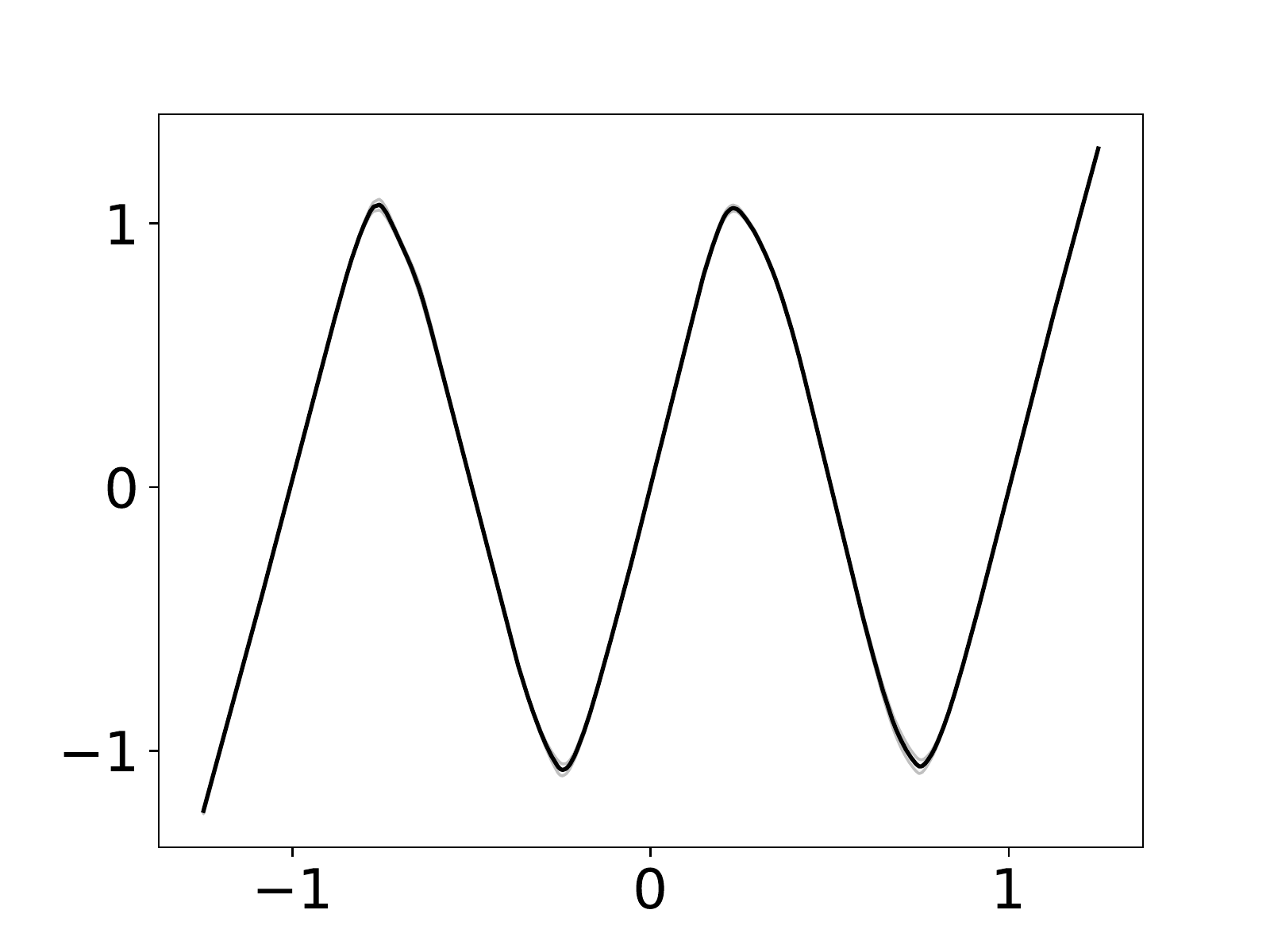}
    \end{subfigure}
    \hfill
    \begin{subfigure}[t]{0.32\textwidth}
        \centering
        \includegraphics[width=\textwidth]{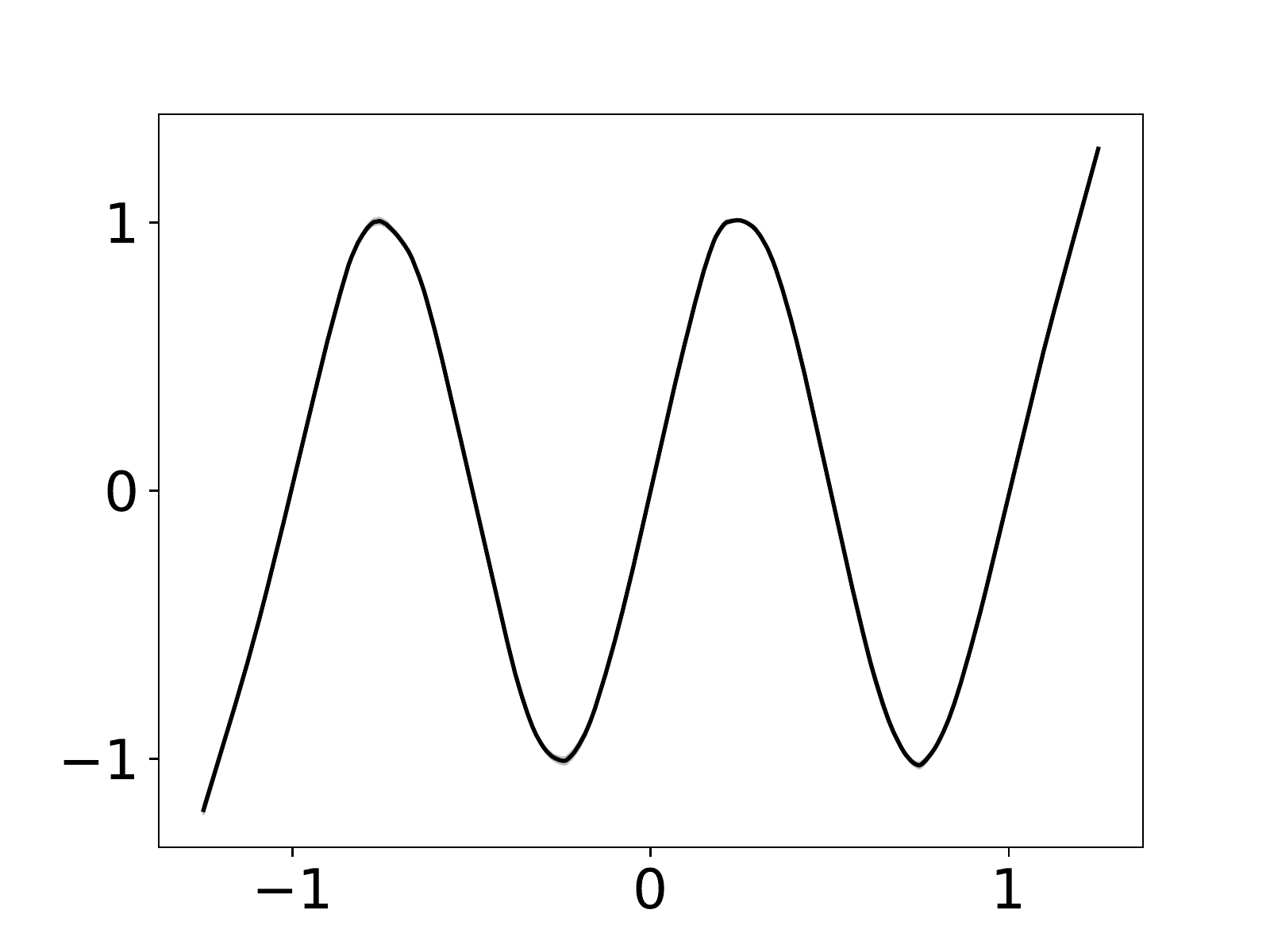}
    \end{subfigure}
    \hfill
    \begin{subfigure}[t]{0.32\textwidth}
        \centering
        \includegraphics[width=\textwidth]{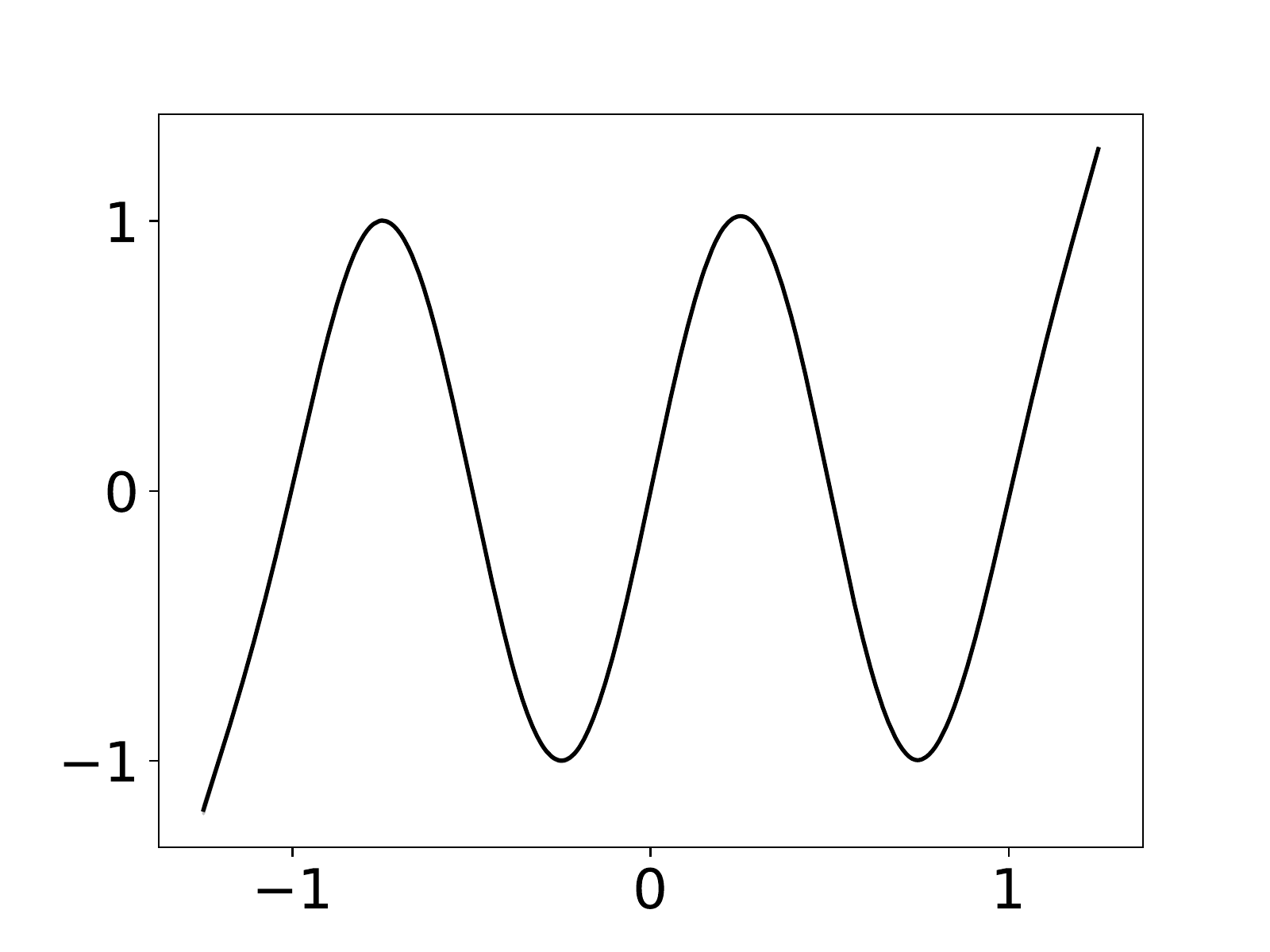}
    \end{subfigure}
    
    \caption{Visualization of the mean prediction and variance of the different width neural networks. Widths in increasing order from left to right and top to bottom: 5, 10, 15, 17, 20, 22, 25, 35, 75, 100, 1000, 10000.}
    \label{fig:app_mean_and_var_vis}
\end{figure}

\begin{figure}[H]
    \centering
    \begin{subfigure}[t]{0.48\textwidth}
        \centering
        \includegraphics[width=\textwidth]{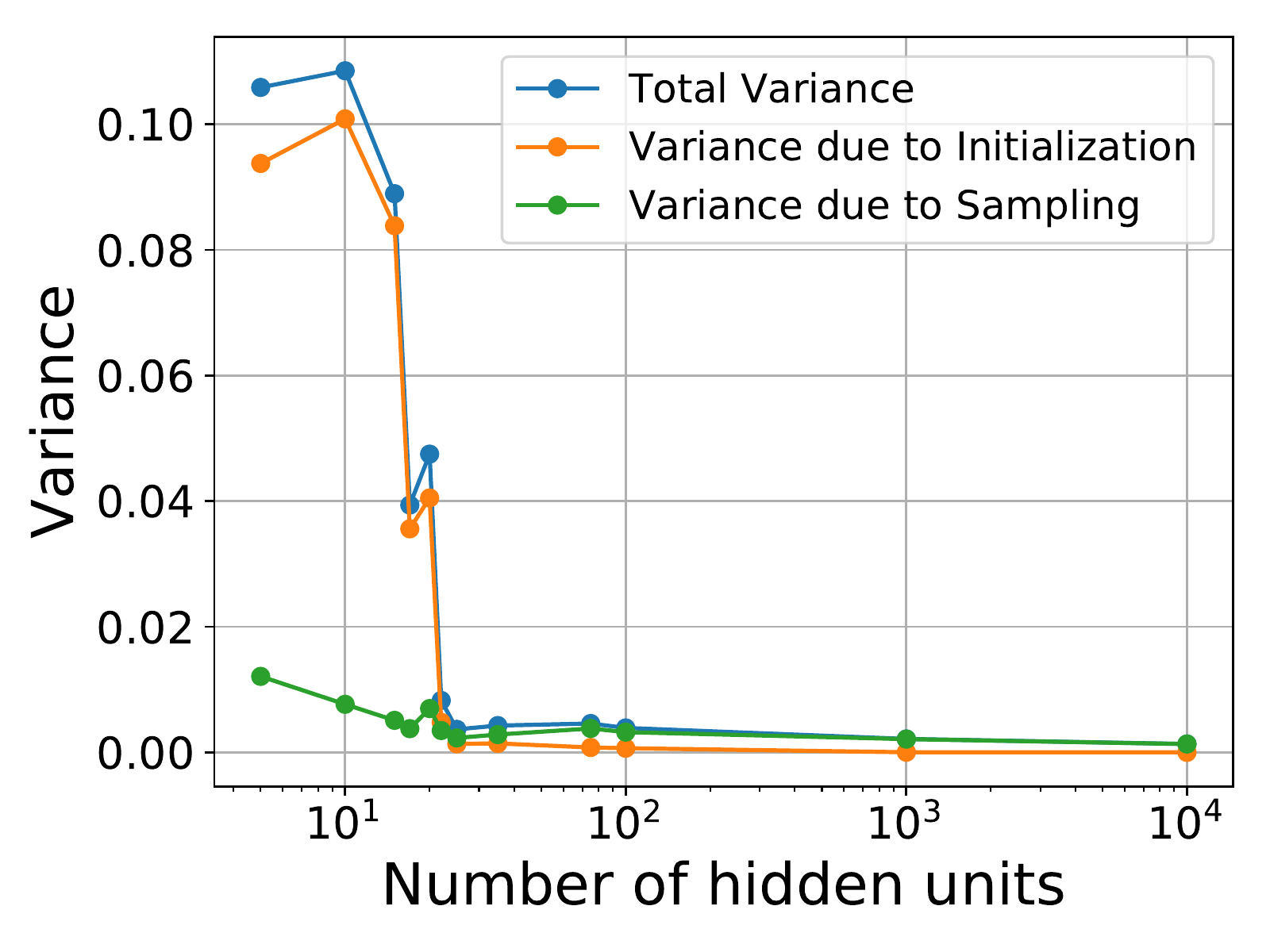}
    \end{subfigure}
    \hfill
    \begin{subfigure}[t]{0.48\textwidth}
        \centering
        \includegraphics[width=\textwidth]{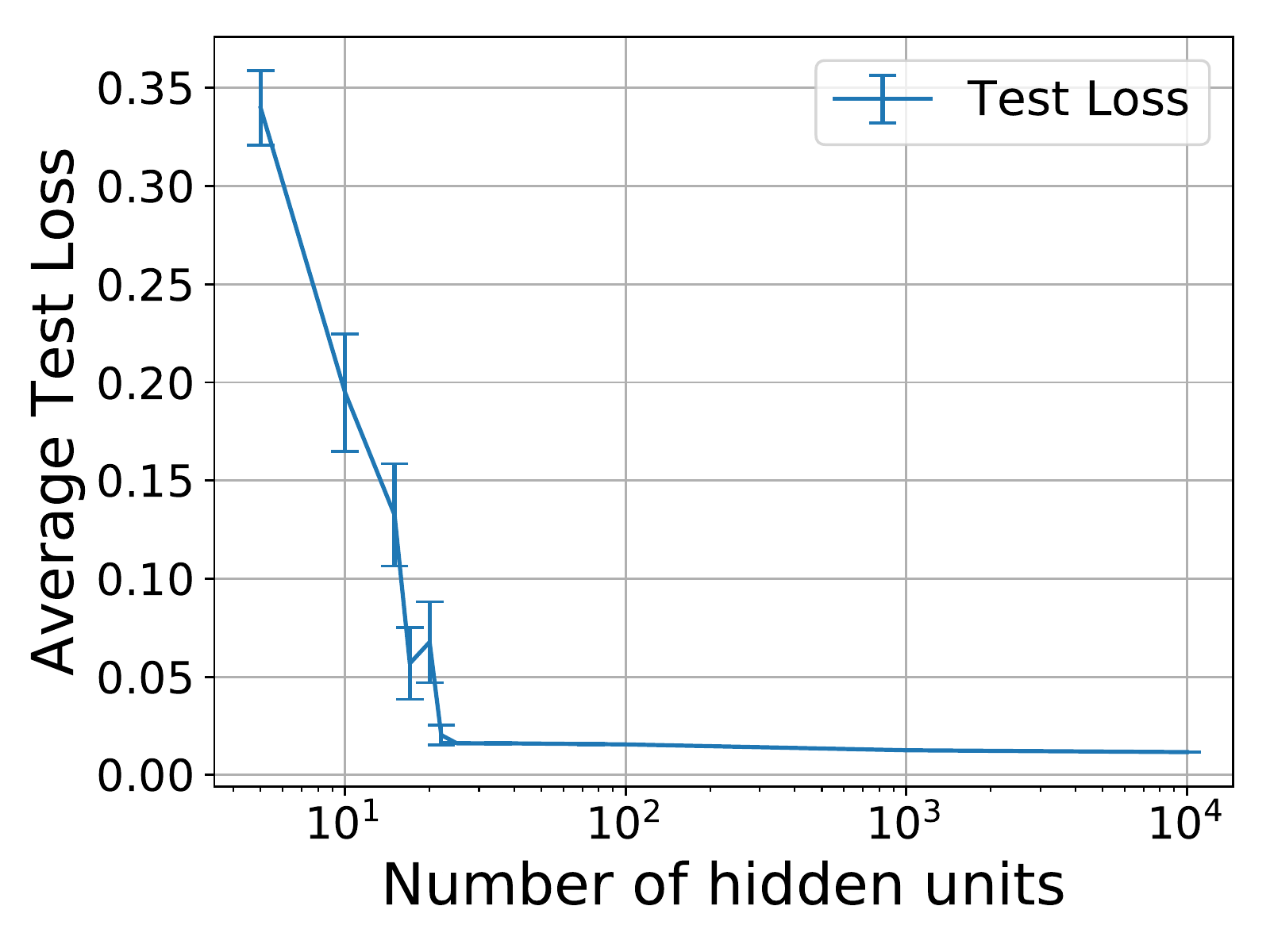}
    \end{subfigure}
    \caption{We observe the same trends of decomposed variance (left) and test error (right) in the sinusoid regression setting.}
    \label{fig:sinusoid_curves_app}
\end{figure}

\section{Depth and variance}
\label{app:depth}

\subsection{Main graphs}

\begin{figure}[H]
    \centering
    \begin{subfigure}[t]{0.48\textwidth}
        \centering
        \includegraphics[width=\textwidth]{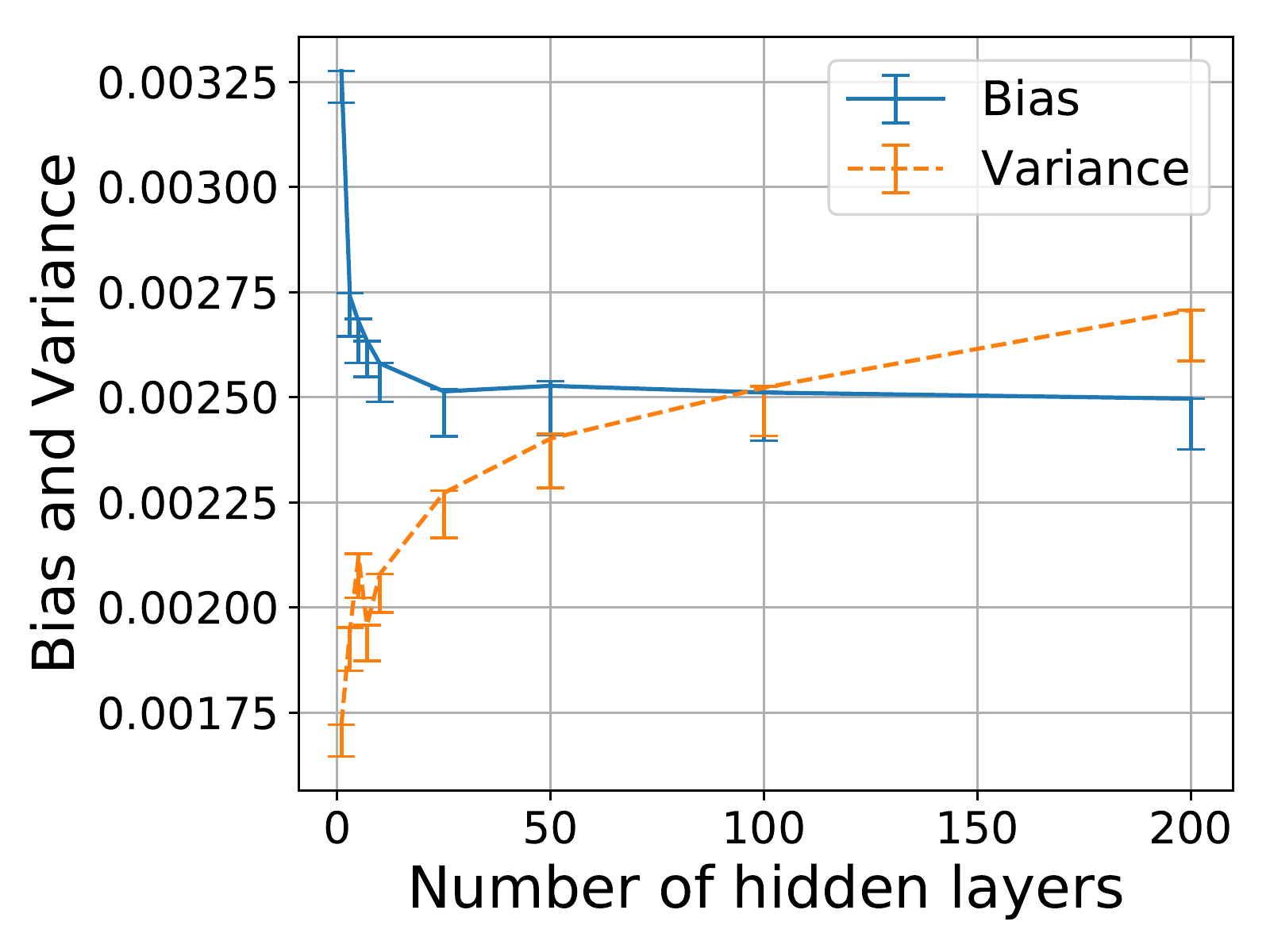}
         \caption{Bias and variance trends with depth, using dynamical isometry}
         \label{fig:dyn_iso_bv}
    \end{subfigure}
    \hfill
    \begin{subfigure}[t]{0.48\textwidth}
        \centering
         \includegraphics[width=\textwidth]{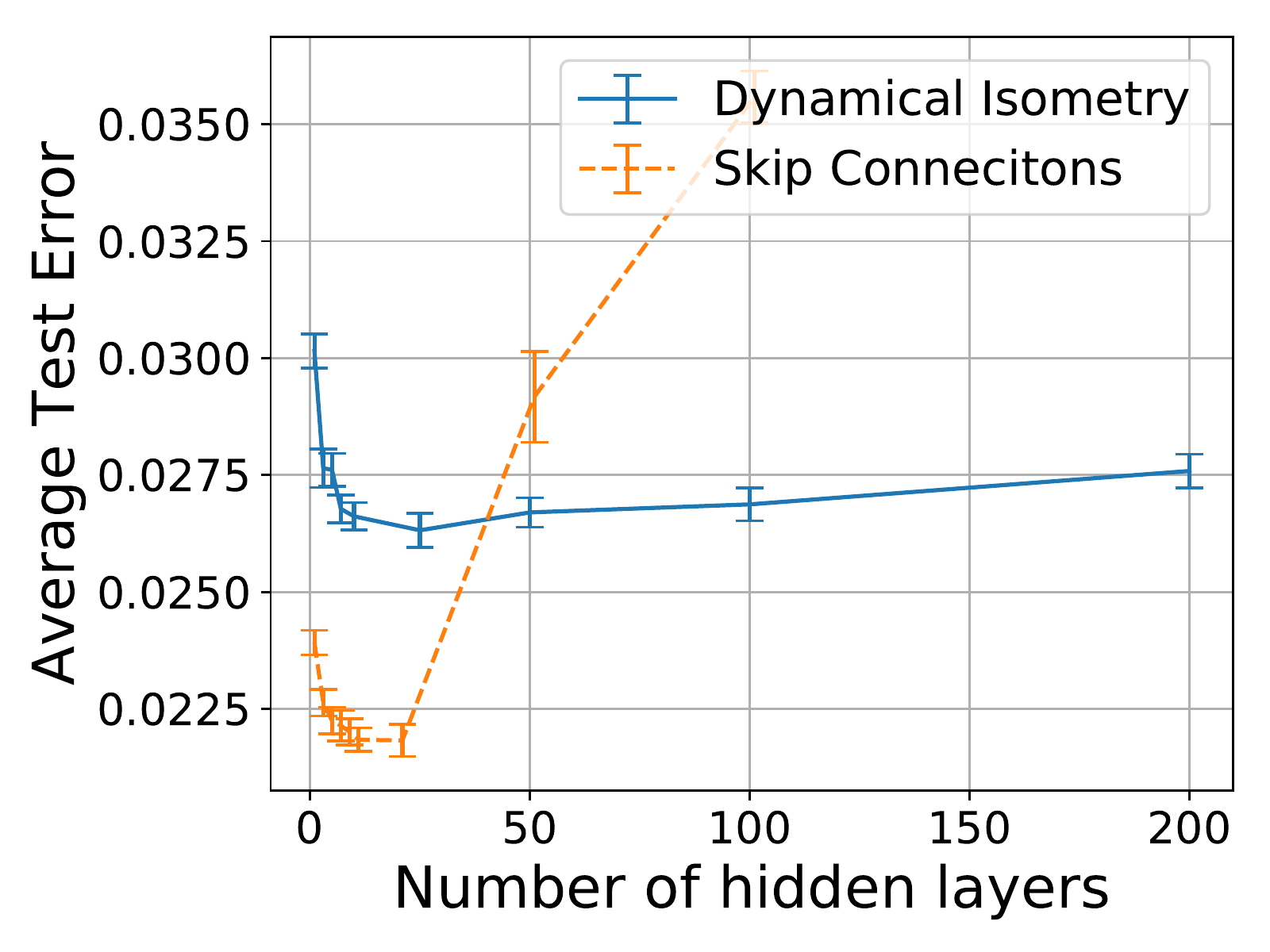}
         \caption{Test error trends, using dynamical isometry vs.\ skip connections}
        \label{fig:depth_test_error}
    \end{subfigure}
    \caption{We can see that, when using dynamical isometry, bias decreases with depth and variance slowly increases with depth (left). This increase in variance is so small that it only translates to a an increase in test error of about 0.1\% for depth 25 to depth 200 (right).}
\end{figure}

\subsection{Discussion on need for careful experimental design}
\label{app:depth_exp_protocol_discussion}

Depth is an important component of deep learning. We study its effect on bias and variance by fixing width and varying depth. However, there are pathological problems associated with training very deep networks such as vanishing/exploding gradient \citep{Hochreiter:91, long-term_dependencies, xavier2010}, signal not being able to propagate through the network \citep{deep_info_prop}, and gradients resembling white noise \citep{shattered_gradient}. \citet{resnet1} pointed out that very deep networks experience high test set error and argued it was due to high training set loss. However, while skip connections \citep{resnet1}, better initialization \citep{xavier2010}, and batch normalization \citep{batchnorm} have largely served to facilitate low training loss in very deep networks, the problem of high \textit{test set} error still remains.

The current best practices for achieving low test error in very deep networks arose out of trying to solve the above problems in training. An initial step was to ensure the mean squared singular value of the input-output Jacobian, at initialization, is close to 1 \citep{xavier2010}. More recently, there has been work on a stronger condition known as \textit{dynamical isometry}, where \textit{all} singular values remain close to 1 \citep{Saxe14exactsolutions, resurrecting_sigmoid}. \citet{resurrecting_sigmoid} also empirically found that dynamical isometry helped achieve low test set error. Furthermore, \citet[Figure 1]{xiao18} found evidence that test set performance did not degrade with depth when they lifted dynamical isometry to CNNs. This why we settled on dynamical isometry as the best known practice to control for as many confounding factors as possible.

We first ran experiments with vanilla full connected networks (\cref{fig:depth_vanilla}). These have clear training issues where networks of depth more than 20 take very long to train to the target training loss of \mbox{5e-5}. The bias curve is not even monotonically decreasing. Clearly, there are important confounding factors not controlled for in this simple setting. Still, note that variance increases roughly linearly with depth.

We then study fully connected networks with skip connections between every 2 layers (\cref{fig:depth_skip}). While this allows us to train deeper networks than without skip connections, many of the same issues persist (e.g.\ bias still not monotonically decreasing). The bias, variance, and test error curves are all checkmark-shaped. 

\subsection{Vanilla fully connected depth experiments}
\label{app:depth_vanilla}

\begin{figure}[H]
    \centering
    \begin{subfigure}[t]{0.45\textwidth}
        \centering
        \includegraphics[width=\textwidth]{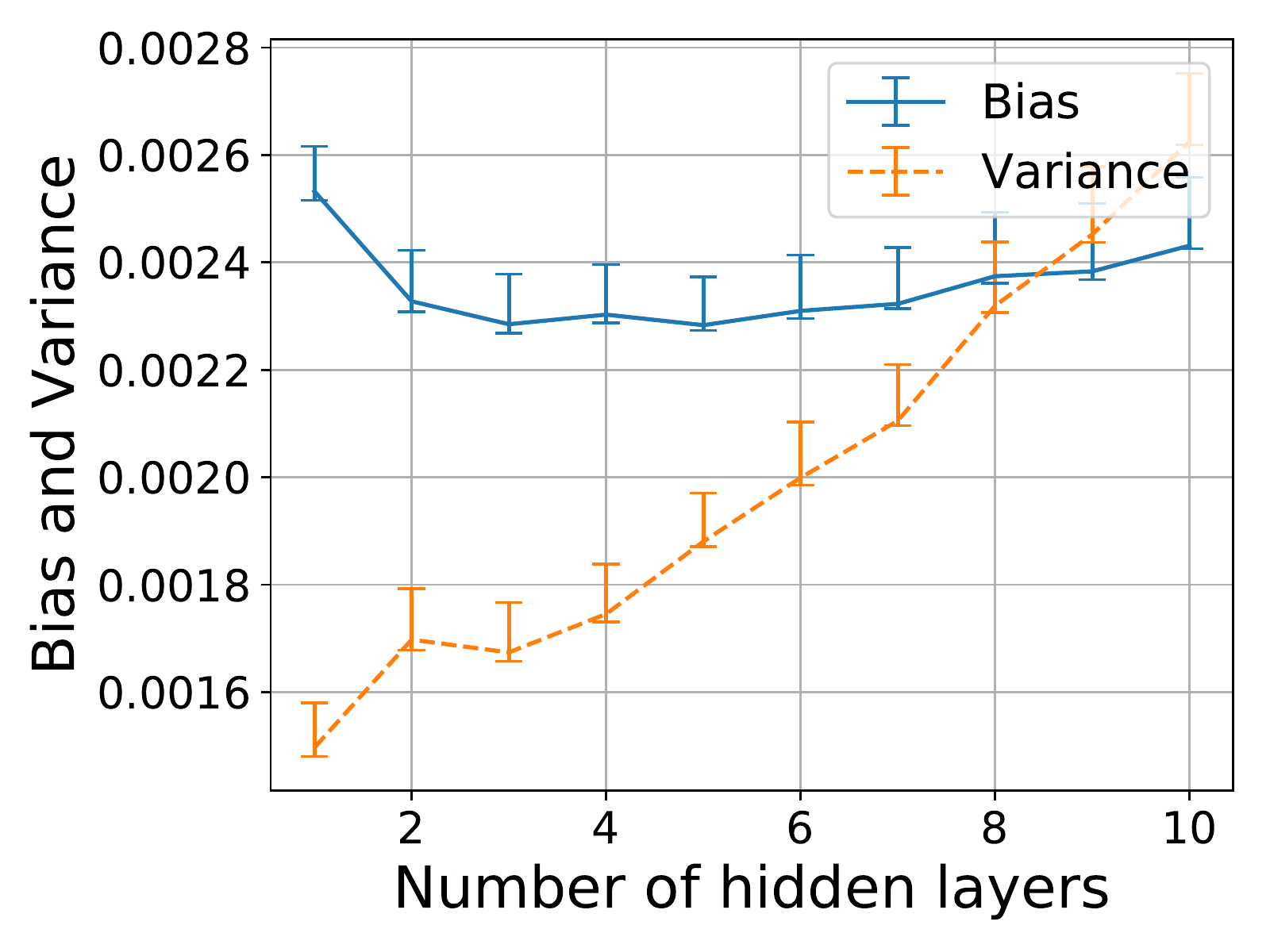}
    \end{subfigure}
    \hfill
    \begin{subfigure}[t]{0.45\textwidth}
        \centering
        \includegraphics[width=\textwidth]{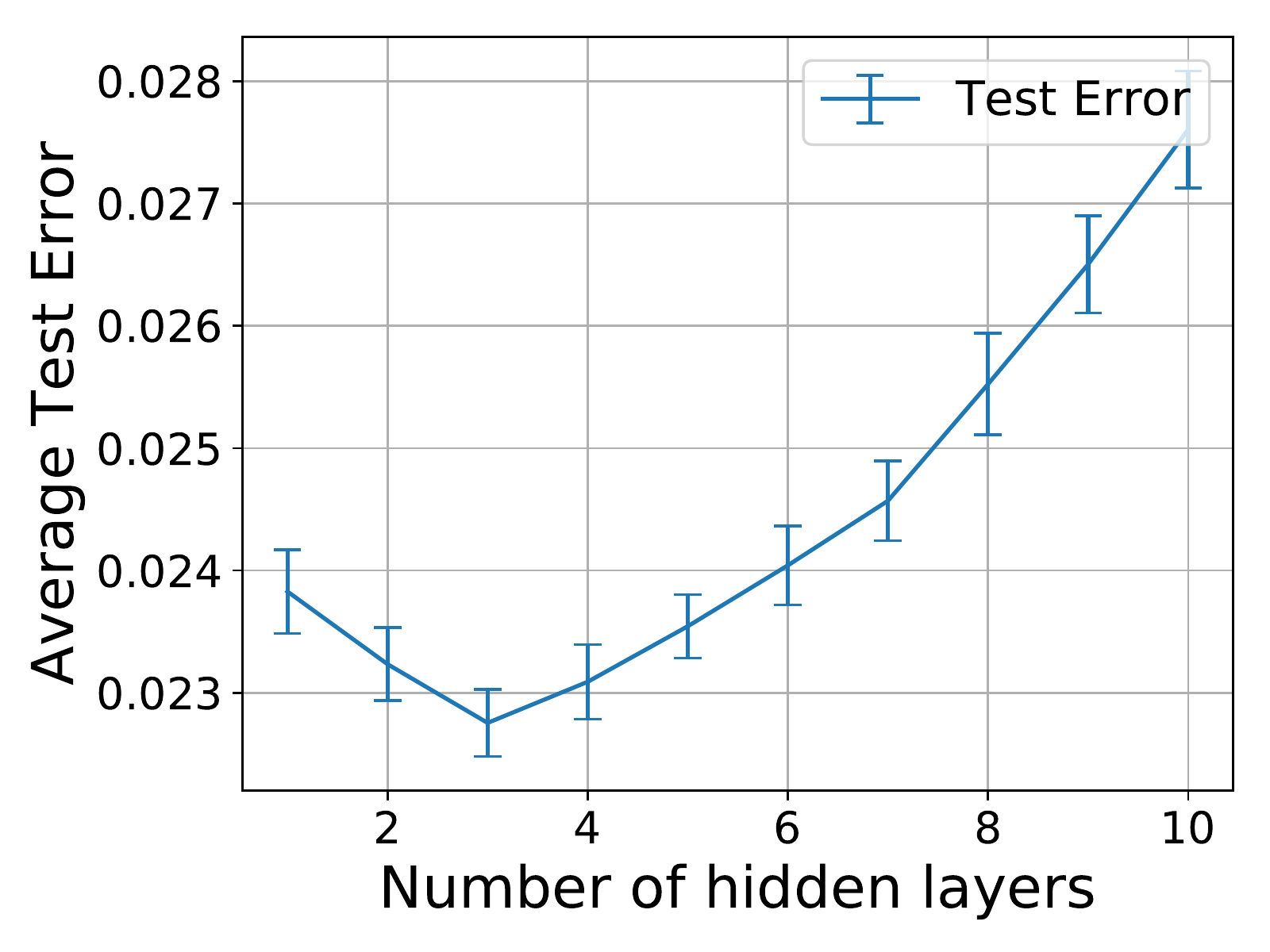}
    \end{subfigure}
    \caption{Test error quickly degrades in fairly shallow fully connected networks, and bias does not even monotonically decrease with depth. However, this is the first indication that variance might \textit{increase} with depth. All networks have training error 0 and are trained to the same training loss of 5e-5.}
    \label{fig:depth_vanilla}
\end{figure}

\subsection{Skip connections depth experiments}
\label{app:depth_skip}

\begin{figure}[H]
    \centering
    \begin{subfigure}[t]{0.45\textwidth}
        \centering
        \includegraphics[width=\textwidth]{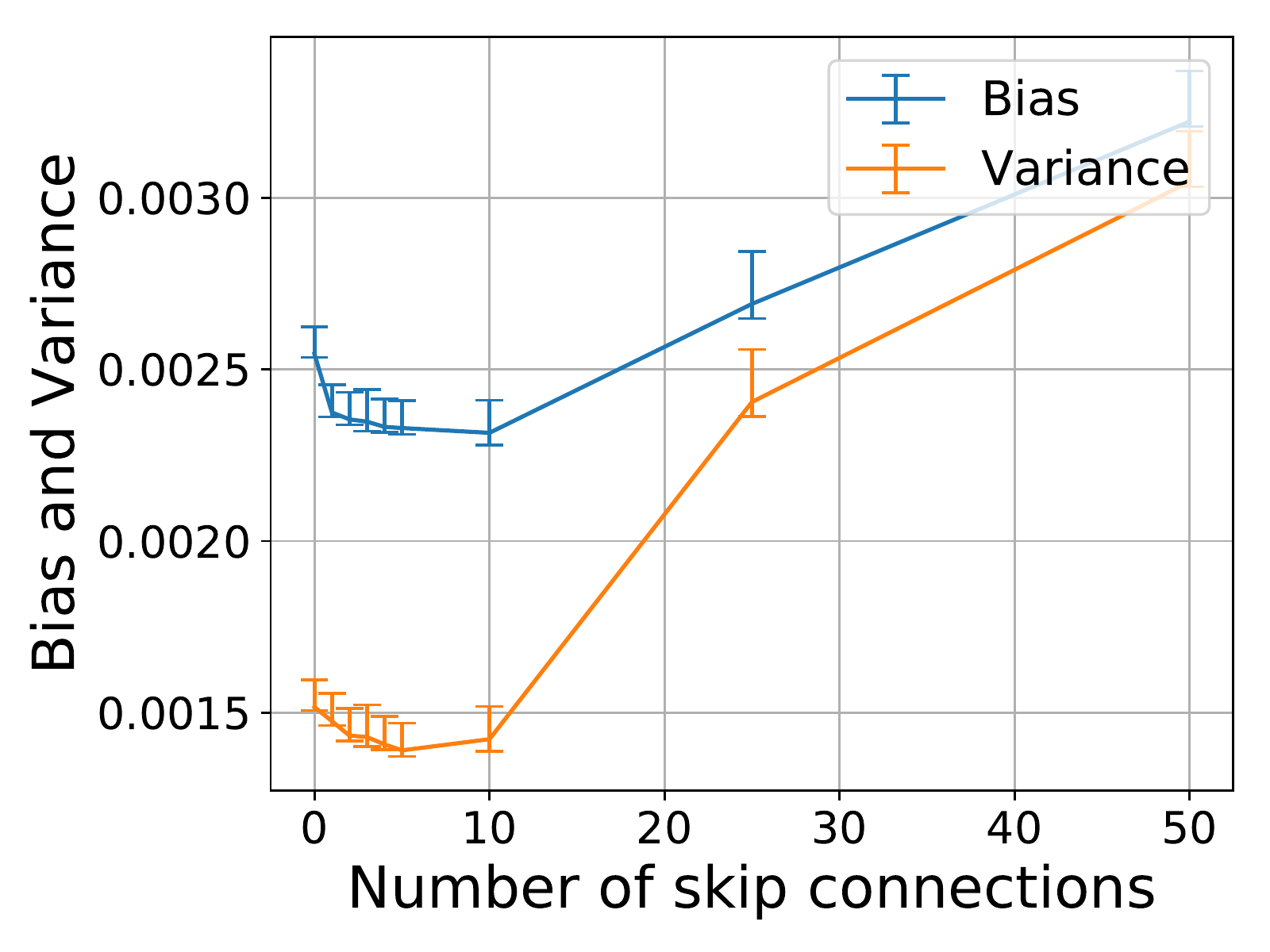}
    \end{subfigure}
    \hfill
    \begin{subfigure}[t]{0.45\textwidth}
        \centering
        \includegraphics[width=\textwidth]{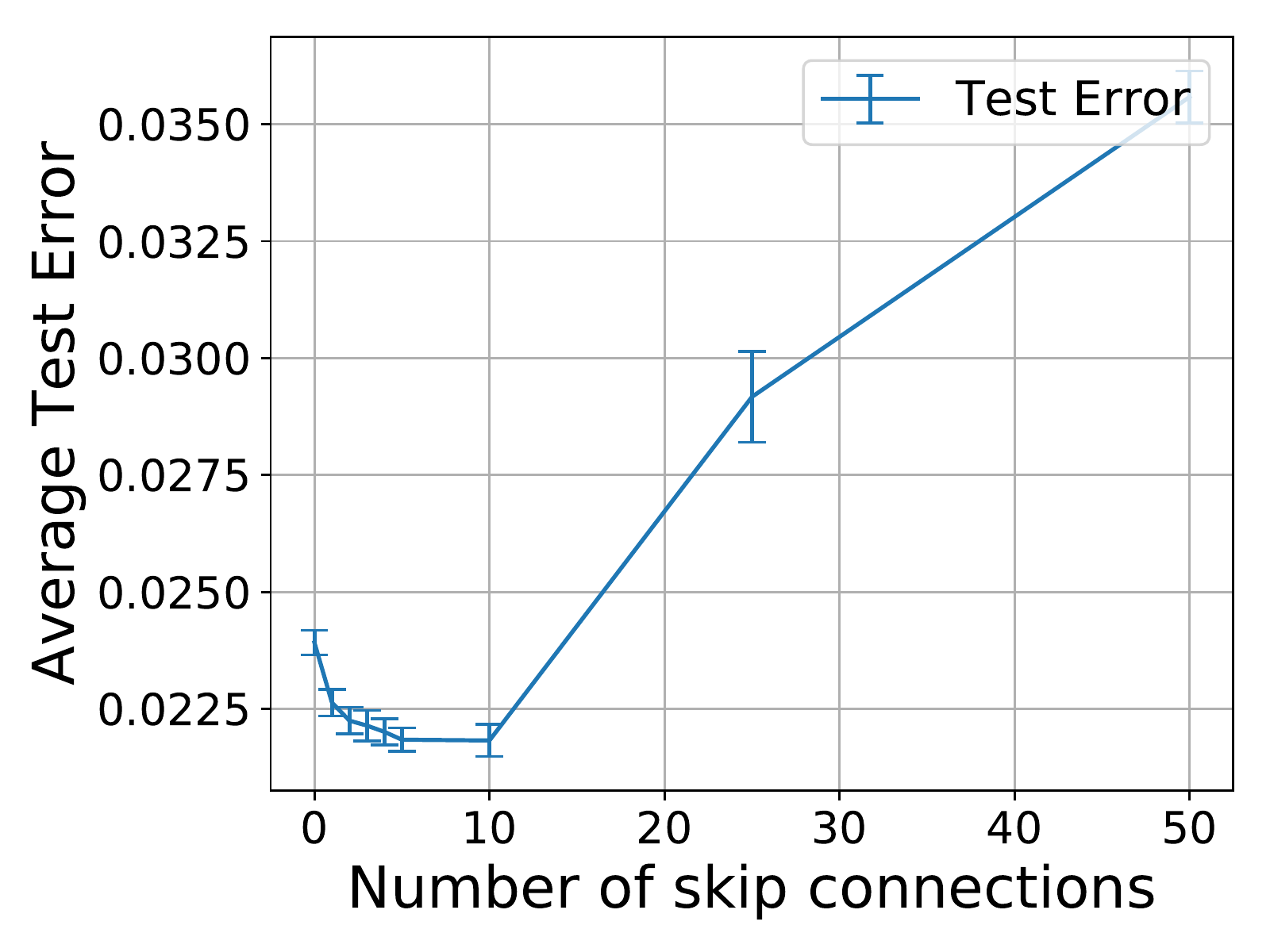}
    \end{subfigure}
    \caption{While the addition of skip connections (between every other layer) might push the bottom of the U curve in test error out to 10 skip connections (21 layers), which is further than where the bottom was observed without skip connections (3 layers), test error still degrades noticeably in greater depths. Additionally, bias still does not even monotonically decrease with depth. While skip connections appear to have helped control for the factors we want to control, they were not completely satisfying. All networks have training error 0 and are trained to the same training loss of 5e-5.}
    \label{fig:depth_skip}
\end{figure}

\subsection{Dynamical isometry depth experiments}
\label{app:depth_dyn_iso}

The figures in this section are included in the main paper, but they are included here for comparison to the above and for completeness.

\begin{figure}[H]
    \centering
    \begin{subfigure}[t]{0.45\textwidth}
        \centering
        \includegraphics[width=\textwidth]{figures/depth_dyn_iso/bias-variance}
    \end{subfigure}
    \hfill
    \begin{subfigure}[t]{0.45\textwidth}
        \centering
        \includegraphics[width=\textwidth]{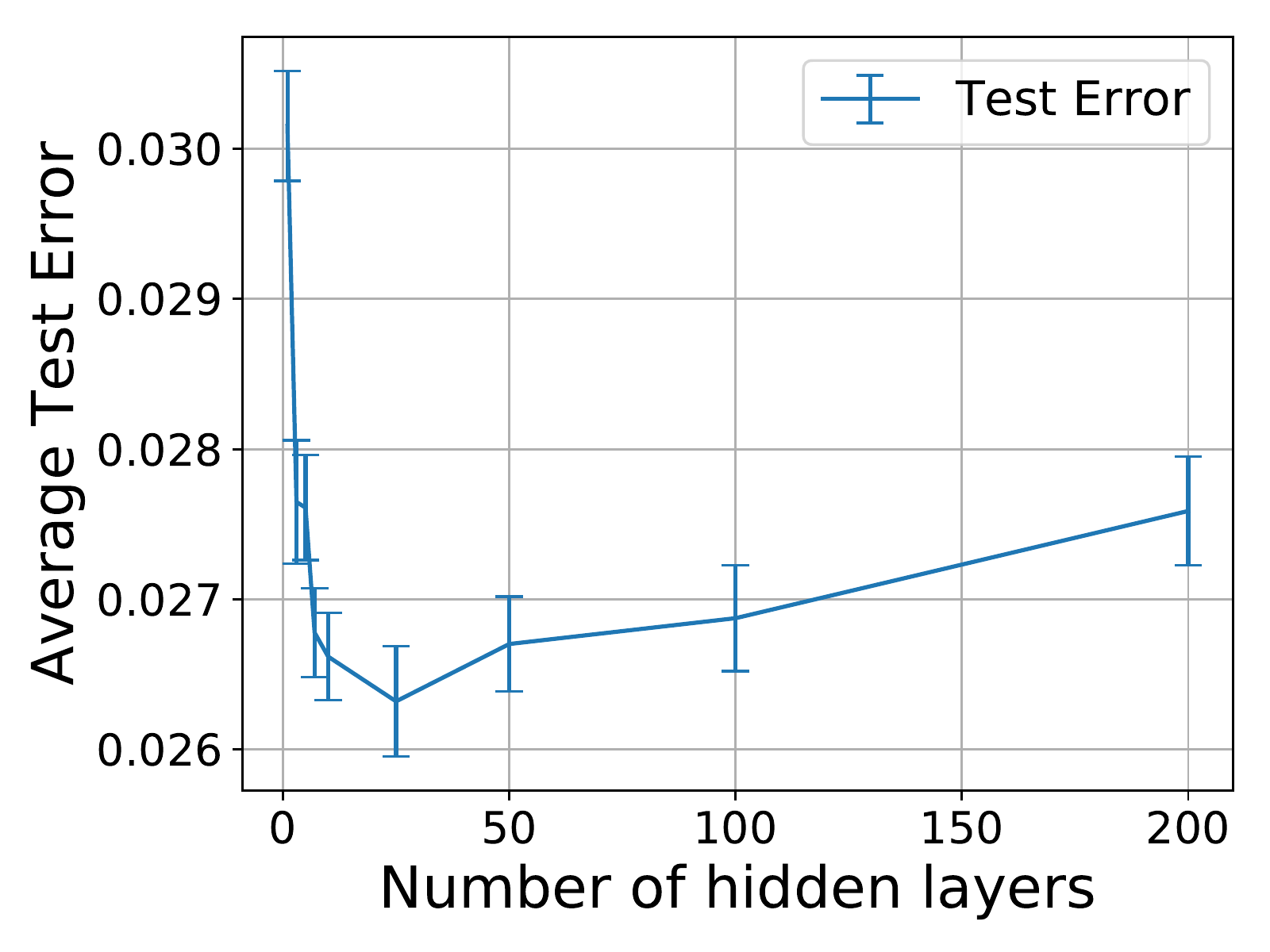}
    \end{subfigure}
    \caption{ Additionally, dynamical isometry seems to cause bias to decrease monotonically with depth. While skip connections appear to have helped control for the factors we want to control, they were not completely satisfying. All networks have training error 0 and are trained to the same training loss of 5e-5.}
\end{figure}

\section{Some Proofs} 
\label{app:proofs}

\subsection{Proof of Classic Result for Variance of Linear Model}
\label{app:linear_underparam}

Here, we reproduce the classic result that variance grows with the number of parameters in a linear model. This result can be found in \citet{hastie_09}'s book, and a similar proof can be found in \citet{linear_var_lecture}'s lecture slides.
\begin{proof} 
For a fixed $x$, we have $h(x) = x^T \hat{\theta}$. 
Taking $\hat{\theta} = \Sigma^{-1} X^T Y$ to be the gradient descent solution, and using $Y= X\theta + \epsilon$, we obtain:
$$
h(x) = x^T \Sigma^{-1} X^T (X\theta + \epsilon) = 
x^T \theta + x^T \Sigma^{-1} X^T \epsilon
$$
Hence $\E_\epsilon[h(x)] = x^T \theta$, and the variance is,  
\begin{align*}
    \Var_\epsilon(h(x)) &= \E_\epsilon[(h(x) - \E_\epsilon[h(x)])^2] \\
    &= \E_\epsilon[(x^T \theta + x^T \Sigma^{-1} X^T \epsilon - x^T \theta)^2] \\
    &= \E_\epsilon[(x^T \Sigma^{-1} X^T \epsilon)^2] \\
    &= \E_\epsilon[(x^T \Sigma^{-1} X^T \epsilon)(x^T \Sigma^{-1} X^T \epsilon)^T] \\
    &= \E_\epsilon[x^T \Sigma^{-1} X^T \epsilon \epsilon^T (x^T \Sigma^{-1} X^T)^T] \\
    &= \sigma_\epsilon^2 x^T \Sigma^{-1}  \Sigma \Sigma^{-1} x\\
    &= \sigma_\epsilon^2 x^T \Sigma^{-1}  \Sigma \Sigma^{-1} x\\
    &= \sigma_\epsilon^2 x^T \Sigma^{-1} x \\
    &= \sigma_\epsilon^2 \Tr(x^T \Sigma^{-1} x) \\
    &= \sigma_\epsilon^2 \Tr(x x^T \Sigma^{-1}) 
\end{align*}

Taking the expected value over the empirical distribution, $\hat{p}$, of the sample, we find an explicit increasing dependence on $N$:
\begin{align*}
    \E_{x \sim \hat{p}} [\Var_\epsilon(h(x))] &= \E_{x \sim \hat{p}} [\sigma_\epsilon^2 \Tr(x x^T \Sigma^{-1})] \\
    &= \sigma_\epsilon^2 \Tr(\E_{x \sim \hat{p}} [x x^T] \Sigma^{-1}) \\
    &= \sigma_\epsilon^2 \Tr \left(\frac{1}{m} \Sigma \Sigma^{-1} \right) \\
    &= \sigma_\epsilon^2 \frac{1}{m} \Tr(I_N) \\
    &= \sigma_\epsilon^2 \frac{N}{m}
\end{align*}

\end{proof}

\subsection{Proof of Result for Variance of Over-parameterized Linear Models}
\label{app:linear_overparam}

Here, we produce a variation on what was done in \cref{app:linear_underparam} to show that variance does not grow with the number of parameters in over-parameterized linear models. Recall that we are considering the setting where $N > m$, where $N$ is the number of parameters and $m$ is the number of training examples.

\begin{proof}
By the law of total variance, 
$$ 
\Var(h(x)) = \E_\epsilon\Var_{\theta_0}(h(x)) + \Var_{\epsilon}(\E_{\theta_0}[h(x)])
$$
Here have $h(x) = x^T \hat{\theta}$, where $\hat{\theta}$ the gradient descent solution   $\hat{\theta} = P_\perp(\theta_0) + \Sigma^+ X^T Y$, and $\theta_0 \sim \mathcal{N}(0, \frac{1}{N} I)$. Then, 
\begin{align*}
\Var_{\theta_0}(h(x)) &= \E_{\theta_0}[(h(x) - \E_{\theta_0}[h(x)])^2] \\
    &= \E_{\theta_0}[x^T(P_\perp(\theta_0) - \E_{\theta_0}[P_\perp(\theta_0)])^2]\\
    &=\Var_{\theta_0}(x^T P_\perp(\theta_0))\\
    &=\Var_{\theta_0}(P_\perp(x)^T P_\perp(\theta_0))\\
    &= \frac{1}{N} \|P_\perp(x)\|^2
\end{align*}
Since $\E_{\theta_0}(h(x)) = x^T \Sigma^+ X^T Y$, the calculation of $\Var_\epsilon(\E_{\theta_0}) h(x))$ is similar as in  \ref{app:linear_underparam}, where $\Sigma^{-1}$ is replaced by $\Sigma^+$.
Thus, 
$$
    \Var_\epsilon(\E_{\theta_0} h(x))   
    = \sigma_\epsilon^2 \Tr(x x^T \Sigma^{+}) 
$$

Taking the expected value over the empirical distribution, $\hat{p}$, of the sample, we find an explicit dependence on $r = \rank(X)$, not $N$:
\begin{align*}
    \E_{x \sim \hat{p}} [\Var(h(x))] &= 0 +  \E_{x \sim \hat{p}} [\sigma_\epsilon^2 \Tr(x x^T \Sigma^{+})] \\
    &= \sigma_\epsilon^2 \Tr(\E_{x \sim \hat{p}} [x x^T] \Sigma^{+}) \\
    &= \sigma_\epsilon^2 \Tr \left(\frac{1}{m} \Sigma \Sigma^{+} \right) \\
    &= \sigma_\epsilon^2 \frac{1}{m} \Tr(I_r^+) \\
    &= \sigma_\epsilon^2 \frac{r}{m}
\end{align*}
where $I_r^+$ denotes the diagonal matrix with 1 for the first $r$ diagonal elements and $0$ for the remaining $N - r$ elements.
\end{proof}

\subsection{Proof of \cref{thm:init-var-decay}}
\label{app:more_general_setting}

First we state some known concentration results \citep{ledoux2001concentration} that we will use in the proof.

\begin{lemma}[Levy] \label{Levy} Let $h: S^{n}_R \to \reals$ be a function on the $n$-dimensional Euclidean sphere of radius $R$, with Lipschitz constant $L$; and $\theta \in S^n_R$ chosen uniformly at random for the normalized measure.  Then 
\beq
\Prob(|h(\theta)-\E[h]| >\epsilon) \leq 2 \exp\left( - C \frac{n \epsilon^2}{ L^2 R^2 }\right)
\eeq
for some universal constant $C >0$.
\end{lemma} 

Uniform measures on high dimensional spheres approximate Gaussian distributions \citep{ledoux2001concentration}. Using this, Levy's lemma yields an analogous concentration inequality for functions of Gaussian variables:

\begin{lemma}[Gaussian concentration] \label{Levy-Gauss} Let $h: \R^n \to \reals$ be a function on the Euclidean space $\R^n$, with Lipschitz constant $L$;  and $\theta \sim \mathcal{N}(0, \sigma \mathbb{I}_n)$  sampled from an isotropic $n$-dimensional Gaussian. Then: 
\beq
\Prob(|h(\theta)-\E[h]| >\epsilon) \leq 2 \exp\left( - C \frac{\epsilon^2}{ L^2 \sigma^2 }\right)
\eeq
for some universal constant $C >0$.
\end{lemma}
Note that in the Gaussian case, the bound is dimension free. 

In turn, concentration inequalities  give variance bounds for functions of random variables.

\begin{coro} \label{Levy-variance}
 Let $h$ be a function satisfying the conditions of Theorem \ref{Levy-Gauss},   and $\mbox{Var}(h) = \E[(h - \E[h])^2]$. Then 
 \beq 
  \mbox{Var}(h) \leq \frac{2 L^2 \sigma^2 }{C}
 \eeq
\end{coro}

\begin{proof} Let $g = h - \E[h]$. Then  $\mbox{Var}(h) = \mbox{Var}(g)$ and 
\beq
\mbox{Var}(g) = \E[|g|^2] 
= 2 \E\int_0^{|g|} t dt 
 =2 \E\int_0^\infty t \mathbbm{1}_{|g|>t} \, dt
\eeq
Now swapping expectation and integral (by Fubini theorem), and by using  the identity $\E \mathbbm{1}_{|g|>t} = \Prob(|g| > t)$, we obtain
\begin{align*}
\mbox{Var}(g) & = 2 \int_0^\infty t \, \Prob_R(|g| > t) \, d t \\
& \leq 2 \int_0^\infty 2 t  \exp\left( - C \frac{t^2}{ L^2\sigma^2}\right) d t \\
 & = 2 \left[-\frac{L^2 \sigma^2 }{C} \exp\left( - C \frac{t^2}{L^2\sigma^2}\right)\right]_0^\infty  = \frac{2 L^2\sigma^2}{C}
 \end{align*} 
\end{proof}

We are now ready to prove Theorem $\ref{thm:init-var-decay}$.
We first recall our assumptions:

\begin{assumption}
\label{assum:invariant-space}
The optimization of the loss function is invariant with respect to $\theta_{\mathcal{M}\perp}$.
\end{assumption}

\begin{assumption}
\label{assum:deterministic-solution}  
Along $\mathcal{M}$, optimization yields solutions independently of the initialization $\theta_0$.
\end{assumption}

We add the following assumptions.

\begin{assumption}
\label{assum:lipschitz}
The prediction $h_{\theta}(x)$ is 
$L$-Lipschitz with respect to $\theta_{\mathcal{M}\perp}$.
\end{assumption}

\begin{assumption}
\label{assum:init}
The network parameters are initialized as 
\beq
    \theta_0 \sim 
    \mathcal{N}(0, \frac{1}{N}\cdot I_{N\times N}).
\eeq
\end{assumption}

We first prove that the Gaussian concentration theorem  translates into concentration of predictions in the setting of \cref{sec:variance-from-optimization}.
\begin{theorem}[Concentration of predictions]
\label{thm:concentration-predictions}
Consider the setting of \cref{sec:back-to-nn} and Assumptions \ref{assum:invariant-space} and \ref{assum:init}.  
Let $\theta$ denote the parameters at the end of the learning process.
Then, for a fixed data set, $S$ we get concentration of the prediction, under initialization randomness, 
\begin{equation}
    \Prob(|h_{\theta}(x)-\E[h_{\theta}(x)]| >\epsilon) \leq 2 \exp\left( - C \frac{N \epsilon^2}{ L^2}\right)
\end{equation}
for some universal constant $C >0$.
\end{theorem}

\begin{proof}
In our setting,  
the parameters at the end of learning can be expressed as
\begin{equation}
    \theta = \theta_\mathcal{M}
^* + \theta_{\mathcal{M}^\perp}
\end{equation}
where $\theta_\mathcal{M}^*$ is independent of the initialization $\theta_0$.   To simplify notation, we will assume that, at least locally around $\theta_\mathcal{M}^*$, $\mathcal{M}$ is spanned by the first $d(N)$ standard basis vectors, and $\mathcal{M}^\perp$ by the remaining $N-d(N)$.
This will allow us, from now on, to use the same variable names for $\theta_\mathcal{M}$ and $\thmp$ to denote their lower-dimensional representations of dimension $d(N)$ and $N-d(N)$ respectively.
More generally, we can assume that there is a mapping from $\theta_\mathcal{M}$ and $\thmp$ to those lower-dimensional representations. 

From Assumptions~\ref{assum:invariant-space} and \ref{assum:init} we get
\begin{equation}
    \theta_{\mathcal{M}^\perp}
    \sim \mathcal{N}\left(0, \frac{1}{N} I_{(N-d(N))\times (N-d(N))}\right).
\end{equation}

Let $g(\thmp) 
\triangleq h_{\theta_\mathcal{M}
^* + \thmp}(x)$.
By Assumption~\ref{assum:lipschitz}, 
$g(\cdot)$ is $L$-Lipschitz.
Then, by the Gaussian concentration theorem we get,
\begin{equation}
    \Prob(|g(\thmp)-\E[g(\thmp)]| >\epsilon) \leq 2 \exp\left( - C \frac{N \epsilon^2}{ L^2}\right).
\end{equation}
\end{proof}
The result of Theorem~\ref{thm:init-var-decay} immediately follows from Theorem~\ref{thm:concentration-predictions} and Corollary~\ref{Levy-variance}, with $\sigma^2 = 1/N$:
\beq 
\Var_{\theta_0}(h_\theta(x)) \leq C \frac{2L^2}{N}
\eeq
Provided the Lipschitz constant $L$ of the prediction  grows more slowly than the square of dimension, $L=o(\sqrt{N})$, we conclude that the variance vanishes to zero as $N$ grows.

\subsection{Bound on classification error in terms of regression error}
\label{app:classification_regression_relation}

\newcommand{\Rcl}{\cR_{\mbox{\tiny classif}}}
\newcommand{\Rreg}{\cR_{\mbox{\tiny reg}}}
In this section we give a bound on  classification risk $\Rcl$ in terms of the regression risk $\Rreg$. 

{\bf Notation.} Our classifier defines a  map $h: \mathcal{X} \to  \mathbb{R}^k$, which outputs probability vectors $h(x) \in \mathbb{R}^k$, with $\sum_{y=1}^k h(x)_y = 1$. The classification loss is defined by 
\begin{align} 
L(h) &= \mbox{Prob}_{x,y} \{h(x)_y < \max_{y'} h(x)_{y'}\} \nonumber \\
&= \mathbb{E}_{(x,y)} I(h(x)_y < \max_{y'} h(x)_{y'})
\end{align}
where $I(a) = 1$ if predicate $a$ is true and 0 otherwise.  Given trained predictors $h_S$ indexed by  training dataset $S$, the classification and regression risks are given by,
\beq 
\Rcl = \mathbb{E}_S L(h_S), \qquad \Rreg = \mathbb{E}_S \mathbb{E}_{(x,y)} ||h_S(x) - Y||^2_2
\eeq
where $Y$ denotes the one-hot vector representation of the class $y$. 

\begin{prop}
The classification risk is bounded by four times the regression risk,  
$\Rcl \leq 4 \Rreg$.  
\end{prop}
\begin{proof}  First note that, if $h(x) \in \R^k$ is a probability vector, then 
\[
h(x)_y < \max_{y'} h(x)_{y'}\,  \Longrightarrow \, h(x)_y < \frac12
\]
By taking the expectation over $x, y$,  we obtain the inequality $L(h) \leq {\widetilde L}(h)$ where
\beq {\widetilde L} (h) = \mbox{Prob}_{x,y} \{ h(x)_y < \frac12\}\eeq  

We then have,
\begin{align*} \Rcl :=  \mathbb{E}_S L(h_S) 
&\leq \mathbb{E}_S \tilde{L}(h_S) \\
& = \mbox{Prob}_{S; \, x,y} \{ {h_S(x)}_y < \frac12 \} \\
&= \mbox{Prob}_{S; \, x,y} \{|h_S(x)_y - Y_y| >\frac12\} \\
&\leq \mbox{Prob}_{S; \, x,y} \{||h_S(x) - Y||_2 > \frac12 \}  \\
& = \mbox{Prob}_{S; \, x,y} \{||h_S(x) - Y||^2_2 > \frac14 \}
 \leq 4 \Rreg
\end{align*}

where the last inequality follows from  Markov's inequality.  

\end{proof}
\section{Common intuitions from impactful works} \label{app:intuitions}

\subsection{Quotes from influential papers}

``Neural Networks and the Bias/Variance Dilemma'' from \citep{geman}: ``How big a network should we employ? A small network, with say one hidden unit, is likely to be biased, since the repertoire of available functions spanned by $f(x; w)$ over allowable weights will in this case be quite limited. If the true regression is poorly approximated within this class, there will necessarily be a substantial bias. On the other hand, if we overparameterize, via a large number of hidden units and associated weights, then the bias will be reduced (indeed, with enough weights and hidden units, the network will interpolate the data), but there is then the danger of a significant variance contribution to the mean-squared error. (This may actually be mitigated by incomplete convergence of the minimization algorithm, as we shall see in Section 3.5.5.)''

``An Overview of Statistical Learning Theory'' from \citep{Vapnik:1999}: ``To avoid over fitting (to get a small confidence interval) one has to construct networks with small VC-dimension.''

``Stability and Generalization'' from \citet{Bousquet2002}: ``It has long been known that when trying to estimate an unknown function from data, one needs to find a tradeoff between bias and variance. Indeed, on one hand, it is natural to use the largest model in order to be able to approximate any function, while on the other hand, if the model is too large, then the estimation of the best function in the model will be harder given a restricted amount of data." Footnote: ``We deliberately do not provide a precise definition of bias and variance and resort to common intuition about these notions."

``Understanding the Bias-Variance Tradeoff'' from \citet{fortmann-roe_2012}: ``At its root, dealing with bias and variance is really about dealing with over- and under-fitting. Bias is reduced and variance is increased in relation to model complexity. As more and more parameters are added to a model, the complexity of the model rises and variance becomes our primary concern while bias steadily falls. For example, as more polynomial terms are added to a linear regression, the greater the resulting model's complexity will be.''
\begin{figure}[h]
 \centering
 \includegraphics[width=.6\textwidth]{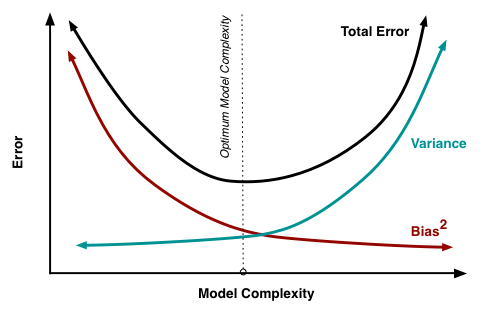}
 \caption{Illustration of common intuition for bias-variance tradeoff \citep{fortmann-roe_2012}}
\end{figure}

\subsection{Quotes from textbooks}

The concept of the bias-variance tradeoff is ubiquitious, appearing in many of the textbooks that are used in machine learning education:
\citet[Chapters 2.9 and 7.3]{hastie01statisticallearning},
\citet[Chapter 3.2]{Bishop:2006}, \citet[Chapter 5.4.4]{Goodfellow-et-al-2016}),
\citet[Chapter 2.3]{Abu-Mostafa:2012:LD:2207825},
\citet[Chapter 2.2.2]{James:2014:ISL:2517747}, \citet[Chapter 3.3]{hastie1990generalized}, \citet[Chapter 9.3]{DudaHart2001}. Here are two excerpts:
\begin{itemize}
    \item ``As a general rule, as we use more flexible methods, the variance will
    increase and the bias will decrease. The relative rate of change of these
    two quantities determines whether the test MSE increases or decreases. As
    we increase the flexibility of a class of methods, the bias tends to initially
    decrease faster than the variance increases. Consequently, the expected
    test MSE declines. However, at some point increasing flexibility has little
    impact on the bias but starts to significantly increase the variance. When
    this happens the test MSE increases'' \citep[Chapter 2.2.2]{James:2014:ISL:2517747}.
    
    \item ``Our goal is to minimize the expected loss, which we have decomposed into the
    sum of a (squared) bias, a variance, and a constant noise term. As we shall see, there
    is a trade-off between bias and variance, with very flexible models having low bias
    and high variance, and relatively rigid models having high bias and low variance'' \citep[Chapter 3.2]{Bishop:2006}.
    
    \item ``As the model complexity of our procedure is increased, the variance tends to increase and the squared bias tends to decrease'' \citep[Chapters 2.9]{hastie01statisticallearning}.
\end{itemize}

\end{appendices}

\end{document}